%% file: main.tex
\documentclass{article}

    \PassOptionsToPackage{numbers, compress}{natbib}

\usepackage[preprint]{neurips_2026}
\usepackage{amsmath}
\usepackage{graphicx}
\usepackage{booktabs}
\usepackage[table]{xcolor}
\usepackage{float}
\usepackage{amsthm}
\usepackage[ruled,vlined]{algorithm2e}
\SetKw{Return}{return}

\input{math_commands}

\usepackage{subcaption}
\usepackage{hyperref}
\usepackage{url}
\usepackage{siunitx}
\usepackage{etoc}

\usepackage{lastpage}
\definecolor{highlight}{RGB}{220,235,247} 
\definecolor{highlightDark}{RGB}{187,215,244} 
\usepackage{bbm}

\usepackage[textsize=tiny]{todonotes}
\newif\iffinal
\iffinal
    \newcommand{\yc}[1]{}
\else
    \newcommand{\yc}[1]{\todo[fancyline,color=blue!40]{YC: #1}\xspace}
\fi

\input{macros}
\definecolor{darkgreen}{HTML}{2E8B57}    

\usepackage[utf8]{inputenc} 
\usepackage[T1]{fontenc}    
\usepackage{hyperref}       
\usepackage{url}            
\usepackage{booktabs}       
\usepackage{amsfonts}       
\usepackage{nicefrac}       
\usepackage{microtype}      
\usepackage{pdflscape}      

\newif\iffinal
\iffinal
    \newcommand{\YC}[1]{}
\else
    \newcommand{\YC}[1]{\todo[fancyline,color=purple!40]{YC: #1}\xspace}
\fi

\title{Learning to Trigger:  Reinforcement Learning at the Large Hadron Collider}

\author{%
Zixin Ding$^{1}$~~~ Shaghayegh Emami$^{2}$~~ Giovanna Salvi$^{2}$~~
\textbf{Cecilia Tosciri}$^{1}$~~ \textbf{Abhijith~Gandrakota}$^{3}$~~ \\
\textbf{Jennifer~Ngadiuba}$^{3}$~~ \textbf{Nhan~Tran}$^{3}$~~ \textbf{Christian Herwig}$^{2}$~~\textbf{ David~W.~Miller}$^{1}$~~ \textbf{Yuxin Chen$^{1}$\thanks{Corresponding to \texttt{zixin@uchicago.edu}, \texttt{chenyuxin@uchicago.edu}}} \\ 
  \\
  University of Chicago$^{1}$\\
  University of Michigan, Ann Arbor$^{2}$ \\
  Fermi National Accelerator Laboratory$^{3}$
}

\begin{document}

\makeatletter
\let\oldaddcontentsline\addcontentsline
\renewcommand{\addcontentsline}[3]{}     
\makeatother

\maketitle

\begin{abstract}
High-throughput scientific facilities such as the Large Hadron Collider depend on real-time event filtering (\textit{triggering}) under tight constraints on bandwidth, latency, and storage. In practice, trigger menus are largely static and hand-tuned and can become suboptimal as detector conditions, pileup, and background composition drift over time. We cast online threshold tuning as a sequential decision-making problem: a reinforcement learning agent ingests streaming summaries of recent rates and signal-sensitive features and updates trigger thresholds to maximize signal efficiency while tracking a target background rate within a tolerance band. We adapt Group-Filtered Policy Optimization (GFPO) to streaming control and introduce two variants (GFPO-F, GFPO-FR) that enforce background rate feasibility during training. On a benchmark that emulates realistic collider operation, we study two representative triggers: a total transverse energy ($H_{T}$) trigger sensitive to pileup variation, and an anomaly-detection (AD) 
trigger based on reconstruction loss for rare or non-standard 
signatures. On Monte Carlo streams, our agent increases the fraction of in-tolerance time intervals by 48\% ($H_T$) and 28\% (AD), with a cumulative gain of up to 2\% in signal efficiency on those in-tolerance intervals. Transferring from simulation to \emph{real} collision data (CMS Run 283408), the same agent, without fine-tuning, achieves a 56\% ($H_T$) and 28\% (AD) in-tolerance improvement over baselines, with further signal-efficiency gain on both triggers. To our knowledge, this is the \emph{first} demonstration of RL-based trigger control on real Large Hadron Collider collision data. Code is available at 
\url{https://github.com/Zixind/GFPO\_LHC}.
\end{abstract}

\input{introduction}

\input{related_work}

\input{method}
\input{limitations}

\section{Acknowledgements}
This work is supported by the University of Chicago Joint Task Force Initiative (JTFI) – Partnerships for Emerging Technologies Seed Funding. A.G, J.N, and N.T are supported by FermiForward Discovery Group, LLC under Contract No. 89243024CSC000002 with the U.S. Department of Energy, Office of Science, Office of High Energy Physics. Z.D and Y.C are supported by the U.S. National Science Foundation (NSF) under grants 2037026, 2313131, 2332475 and 2543755.

\newpage
\bibliographystyle{unsrtnat}
\bibliography{neurips_2026}

\newpage



\appendix
\makeatletter
\let\addcontentsline\oldaddcontentsline  
\makeatother
\renewcommand{\contentsname}{Appendices}
{\small\tableofcontents}
\vspace{0.6em}





\newpage
\input{./Appendix/related_work_extended}

\input{./Appendix/pseudocode}
\input{./Appendix/Experimental_setup}
\input{./Appendix/Implementation}

\input{./Appendix/Appendix}
\newpage
\input{./Appendix/Autoencoder_ablation_study}

\input{./Appendix/anomaly_benchmarks}




\end{document}

%% file: math_commands.tex
\usepackage{amsmath,amsfonts,bm}

\def\eqref#1{equation~\ref{#1}}

\def\1{\bm{1}}

\DeclareMathAlphabet{\mathsfit}{\encodingdefault}{\sfdefault}{m}{sl}
\SetMathAlphabet{\mathsfit}{bold}{\encodingdefault}{\sfdefault}{bx}{n}



%% file: macros.tex
\usepackage{amsthm}
\usepackage{amssymb}
\usepackage{mathtools}

\usepackage[normalem]{ulem} 
\usepackage{mathtools}

\usepackage[utf8]{inputenc} 
\usepackage[T1]{fontenc}    
\usepackage{url}            
\usepackage{booktabs}       
\usepackage{amsfonts}       
\usepackage{nicefrac}       
\usepackage{microtype}      
\usepackage{proof-at-the-end}  
\usepackage{multirow}
\usepackage{lscape}         

\usepackage{graphicx,wrapfig}
\usepackage{amsfonts}
\usepackage{url}
\usepackage{enumitem}
\usepackage{amsthm}
\usepackage{thmtools}
\usepackage{thm-restate}

\usepackage{xcolor}

\newcommand{\bc}{\begin{center}}
\newcommand{\ec}{\end{center}}

\newcommand{\bdm}{\begin{displaymath}}
\newcommand{\edm}{\end{displaymath}}

\newcommand{\beq}{\begin{equation}}
\newcommand{\eeq}{\end{equation}}

\newcommand{\bfl}{\begin{flushleft}}
\newcommand{\efl}{\end{flushleft}}

\newcommand{\bt}{\begin{tabbing}}
\newcommand{\et}{\end{tabbing}}

\newcommand{\beqn}{\begin{align}}
\newcommand{\eeqn}{\end{align}}

\newcommand{\beqs}{\begin{align*}} 
\newcommand{\eeqs}{\end{align*}}  

\newcommand{\actionspace}{\mathcal{A}}

\usepackage[most]{tcolorbox}
\newtcolorbox{boxK}[2][]{
    sharpish corners, 
    boxrule = 0pt,
    toprule = 4.5pt, 
    enhanced,
    fuzzy shadow = {0pt}{-2pt}{-0.5pt}{0.5pt}{black!35}, 
    fontupper = \ttfamily\small, 
    boxsep = 5pt, 
    left = 5pt, 
    right = 5pt, 
    top = 5pt, 
    bottom = 5pt, 
    #1                       
}

\newtcolorbox{takeawaybox}[2][]{
    enhanced,
    boxsep = 2pt, 
    left = 2pt, 
    right = 2pt, 
    top = 2pt, 
    bottom = 2pt, 
    #1                       
}

\newcommand{\SEpaper}{\cite{emami2026selfdrivingtriggerlhcadaptive}}

\newcommand{\HT}{\ensuremath{H_\mathrm{T}}}
\newcommand{\ttbar}{$t\bar{t}$}
\newcommand{\ttbarraw}{t\bar{t}}
\newcommand{\haaFourB}
{\ensuremath{h\to 4b}}
\newcommand{\threshold}{c}

\newtheorem{lemma}{Lemma}

%% file: introduction.tex
\section{Introduction}
\noindent High-throughput discovery science relies on real-time 
event selection under strict bandwidth and latency constraints 
\citep{albrecht2025summary, mahesh2021towards, cms2006cms, 
aad2022performance}. At the Large Hadron Collider (LHC), trigger menus are largely static and hand-tuned, requiring detailed simulation and repeated expert re-optimization as pileup and background composition drift~\citep{aad2022performance, emami2026selfdrivingtriggerlhcadaptive}. This workflow is costly and can quickly become miscalibrated under distribution shift, reducing efficiency for detecting rare phenomena. Only recently have efforts in dynamic bandwidth allocation ~\citep{evans2025automated} and low-latency menu adaptation ~\citep{mahesh2021towards} begun to explore learned policies that retune thresholds online. Yet whether such continuous adaptation is even \emph{feasible} in the more common regime, where a fixed menu of thresholds must be retuned online under shifting conditions, remains largely unexplored. We study the design of a \textit{self-driving} trigger: an autonomous, hardware-aware filtering system that adapts its selection thresholds online to maintain stable background rate with improved signal efficiency, in response to changing physical conditions.

Existing work~\citep{emami2026selfdrivingtriggerlhcadaptive} introduced an autonomous trigger-control framework that extends a PID baseline to a cost-function formulation jointly optimizing rate stability, signal efficiency, and computational cost.
This exhaustive search over discrete threshold candidates is effective for small menus but becomes computationally prohibitive as the action space grows, and neither strategy conditions on streaming context such as pileup~\citep{cms2014description}. Reinforcement
learning (RL) ~\citep{sutton1998reinforcement, hausknecht2015deep,
Hafner2020Dream} offers a principled alternative: a learned policy
over feature-rich states can anticipate distribution shifts, scale to
large action spaces, and navigate the multi-objective
trade-offs that the cost-function framework formalizes but cannot
efficiently solve.

We thus formulate adaptive thresholding as a sequential decision-making problem~\citep{matchev2021thickbrick, oh2019sequential, yang2024adt} and train an RL policy \citep{sutton1998reinforcement} that observes compact streaming summaries (e.g., recent rates and distributional statistics) and updates thresholds to maximize long-horizon physics utility while satisfying hard background-rate constraints \citep{albrecht2025summary, cerri2019variational,cella2024atlas}. By learning directly from time-dependent data, the policy can capture nonlinear and delayed effects and allocate rate budget more effectively, reducing manual retuning and improving performance in non-stationary environments.
 
We begin by studying two representative trigger classes trained on simulated samples with deployment on \textbf{real} collision data recorded by the CMS experiment~\citep{CMSOpenData}. Specifically, we ask: \emph{Can an RL-based trigger controller improve signal efficiency while maintaining comparable (or better) background-rate stability than control-based baselines on real collision data?} If so, \emph{is it preferable to train on physicist-generated simulation and deploy on real collision data, or to rely on test-time training?}

\paragraph{Contributions.}
Concretely, we make four contributions. 

(1) We introduce a unified evaluation framework with metrics that quantify both feasibility and performance. Feasibility is measured through the InBand fraction, defined as the proportion of events for which the trigger rate remains within the target bandwidth tolerance, while performance is assessed through signal efficiency. We formulate online trigger threshold optimization as a reinforcement learning problem with streaming observations and demonstrate that even a simple Deep Q-Network (DQN) controller can effectively adapt trigger thresholds in real time (Section~\ref{sec:single_trigger}).

(2) We introduce a sequence-based observation model representing the state as a
length-$K$ event sequence augmented with physics-informed features (See psuedocode  in Appendix~\ref{app:GFPO_Pseudocode}). Unlike PID
control~\citep{emami2026selfdrivingtriggerlhcadaptive}, which reacts to a scalar
error conflating distinct failure modes (a global luminosity shift versus a local
change in the score distribution near threshold), our representation preserves
distributional context, enabling the policy to diagnose why the rate has
drifted~\citep{Hafner2020Dream}. A recurrent encoder compresses the $K$-event
window into a fixed-size state vector (Appendix~\ref{sec:sequential_network_architectures}).


(3) We adapt Group Filtered Policy Optimization (GFPO)~\citep{shrivastava2025sample},
a variant of Group Relative Policy Optimization (GRPO)~\citep{shao2024deepseekmath},
to the streaming trigger setting and propose two new variants,
GFPO-F and GFPO-FR, that filter rollouts (sampled candidate actions and their
evaluated outcomes) by feasibility (accepted background rate) and signal efficiency.
Both variants improve stability without degrading efficiency (and vice versa) on
Monte Carlo (MC) simulated events and on \textbf{real} CMS collision data
(Run 283408), transfer from simulation to CMS data with negligible performance
loss (Appendix~\ref{appendix:real_data_adaptation}), eliminate the need for gradient-based fine-tuning at deployment, and remain
robust to large variations in anomaly-detection score scale
(Appendix~\ref{appendix:ablation_anomaly_score}).

(4) Our framework extends beyond the LHC. We apply the framework to online anomaly-detection benchmarks and show the combination of sequence-based neural state representations with GFPO-F and GFPO-FR yields consistent improvements under streaming data with distribution shifts, suggesting that our methods generalize well beyond the particle-physics domain (Section~\ref{sec:anomaly_detection_benchmarks}).

Building on the self-driving trigger framework of
\citet{emami2026selfdrivingtriggerlhcadaptive}, we demonstrate the
\emph{first} end-to-end RL-trained policy using real CMS collision data
(Runs 283408 and 283876): one that learns its threshold update rule
from streaming experience rather than minimizing a global cost
function.

%% file: related_work.tex
\section{Related Work}
\paragraph{Reinforcement learning for scientific decision making.} RL has been applied to self-driving labs \citep{volk2023alphaflow}, accelerator tuning \citep{kaiser2024reinforcement}, and tokamak regulation \citep{degrave2022magnetic}. LHC trigger control differs in being explicitly \emph{rate-constrained}: the controller must hold background acceptance within a strict tolerance while maximizing signal efficiency, under non-stationary pileup and detector drift.
\paragraph{Anomaly detection under drift.} Active anomaly detection typically learns \emph{querying} policies under a label budget \citep{zha2020meta}. We instead adapt the \emph{operating threshold} of an autoencoder-based score to satisfy rate constraints under drift. The closest formulation is ADT~\citep{yang2024adt}, which fixes a pretrained autoencoder and selects binary decision thresholds; we treat trigger thresholds as control variables in a real-time feedback loop. DSPOT~\citep{siffer2017anomaly} sets adaptive thresholds via Extreme Value Theory, with its false-positive level $q$ mapping to our target rate $r_{B}^{*}$; we include it as a baseline, though it has no mechanism for optimizing signal efficiency (Appendix~\ref{subsection:DSPOT}).
\paragraph{Selective updates.} Rejection-style filtering stabilizes RL training, particularly for LLMs~\citep{yu2025dapo,yue2026does, shrivastava2025sample}. Most directly, GFPO~\citep{shrivastava2025sample} samples larger candidate groups and filters them based on response length and token efficiency to control verbosity. We adapt this framework to streaming trigger control, where the filter must enforce rate feasibility under distribution shift, and develop two variants: \textbf{GFPO-F} retains the top-$K$ candidates with the smallest rate deviation; \textbf{GFPO-FR} first selects feasible candidates within an expanded tolerance band, then ranks them by signal utility.  Appendix~\ref{app:related_work_extended} provides an extended discussion.

%% file: method.tex
\section{RL formulation in Adaptive Thresholding}
\label{sec:rl_formulation_adaptive_thresholding}
We control a single scalar threshold $\threshold_{t}$
separately for two trigger paths: an $H_{T}$ trigger that selects events with high total hadronic activity, $H_{T} = \sum_{\text{jets}} p_{T}^{jet}$, and an anomaly-detection (AD) trigger that scores events by autoencoder reconstruction loss~\citep{govorkova2022autoencoders}. At each step, our agent chooses $\threshold_{t+1}$ to keep the background acceptance rate\footnote{We use \emph{background rate}, \emph{rate} and \emph{background acceptance rate} interchangeably throughout this paper.} near a target $r_{B}^{*}$ while maximizing signal efficiency on a held-out signal sample. The same formulation applies to both triggers. We model adaptive thresholding as a Markov Decision Process (MDP) $(S, \actionspace, T, R, \gamma)$~\citep{jayaraman2024primer}, where the transition
$T$ is the environment dynamics (the action updates the threshold and the next $K$
events arrive from the data stream, whose distribution drifts
with beam conditions, so the agent does not assume access to $T$) and
$\gamma \in [0,1)$ is the discount factor. We define the state, action, and reward
below. 
\begin{itemize}
    \item \emph{State.} $s_{t} = (h_{t}, c_{t})$, where $h_{t} \in \mathbb{R}^{d}$ is a fixed-size summary
    \footnote{$d=19$ in our LHC setting; it is a physics-feature summary of recent $K$ events.} of recent $K$ events 
    (See State representation design in Appendix~\ref{sec:sequential_network_architectures}), i.e., trigger scores and $c_{t} = (\threshold_{t}, \hat{r}_{t}, \Delta \threshold_{t-1})$ holds the current threshold, the observed rate, and the last threshold change.
\item \emph{Action. } $a_{t} \in \mathcal{A}$ is a discrete threshold increment. The threshold updates as $\threshold_{t+1} = \threshold_{t} + a_{t}$ and is applied to next batch of events.
\item \emph{Reward. } $R_{t}$ combines three terms designed to (a) maintain the background rate close to a target rate within tolerance; (b) maximizing signal efficiency for two representative signals, \ttbar\ and \haaFourB\ 
~\citep{emami2026selfdrivingtriggerlhcadaptive}
; (c) discouraging action update jitter for stability
~\citep{le2016smooth,aad2022performance} 
(Eq~\ref{equation:reward_design}).
\end{itemize}
We instantiate this MDP with three policies of increasing structure: DQN (Sec.~\ref{sec:single_trigger}), GRPO adapted to streaming control (Sec.~\ref{sec:from_dqn_to_grpo}), and feasibility-filtered GFPO (Section~\ref{sec:gfpo}).

\begin{figure*}[t]
    \centering
    \begin{subfigure}{0.49\textwidth}
        \includegraphics[width=\linewidth]{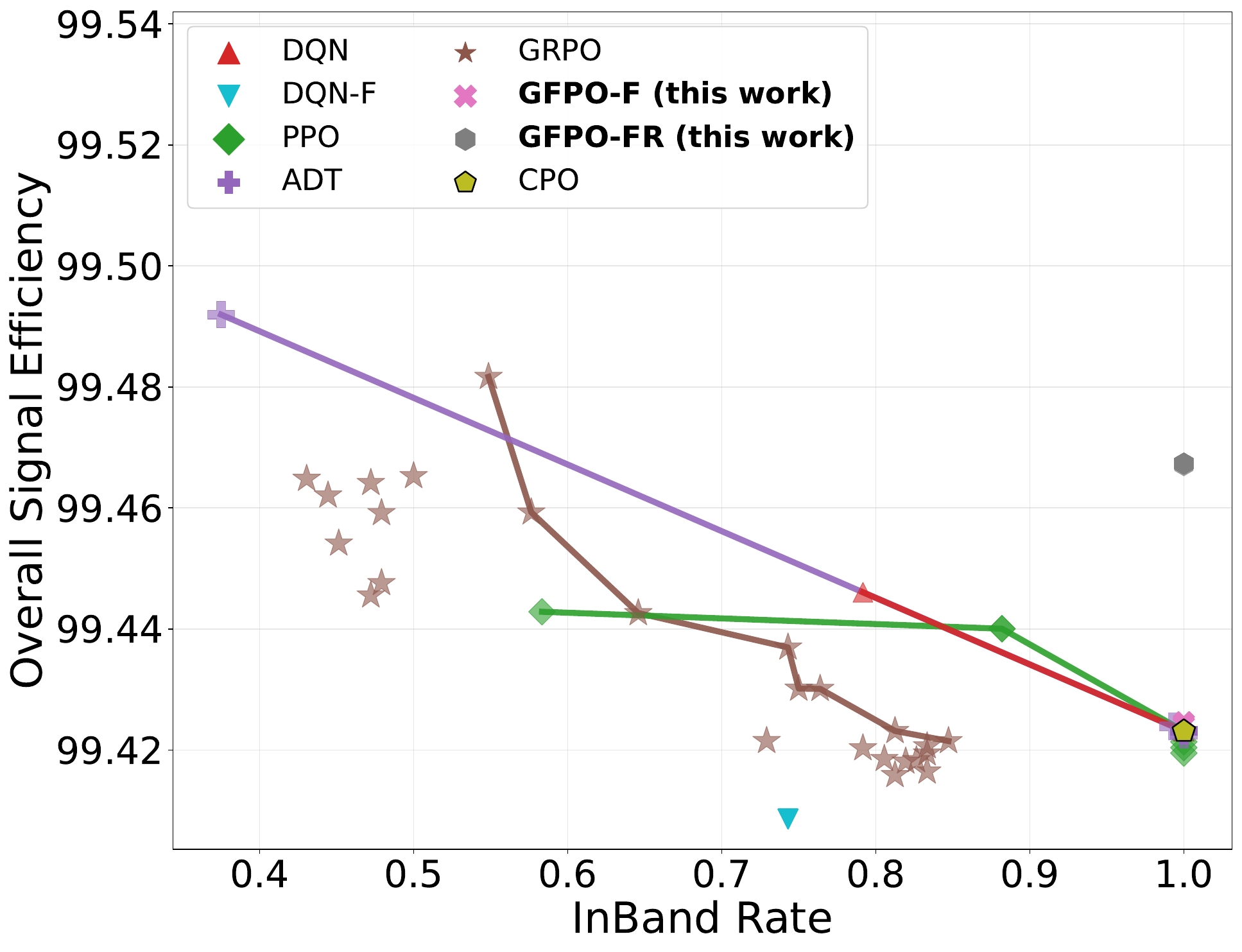}
        \caption{$H_{T}$ trigger (\ttbar)}
        \label{fig:pareto_ht_ttbar_appendix}
    \end{subfigure}
    \hfill
    \begin{subfigure}{0.49\textwidth}
        \includegraphics[width=\linewidth]{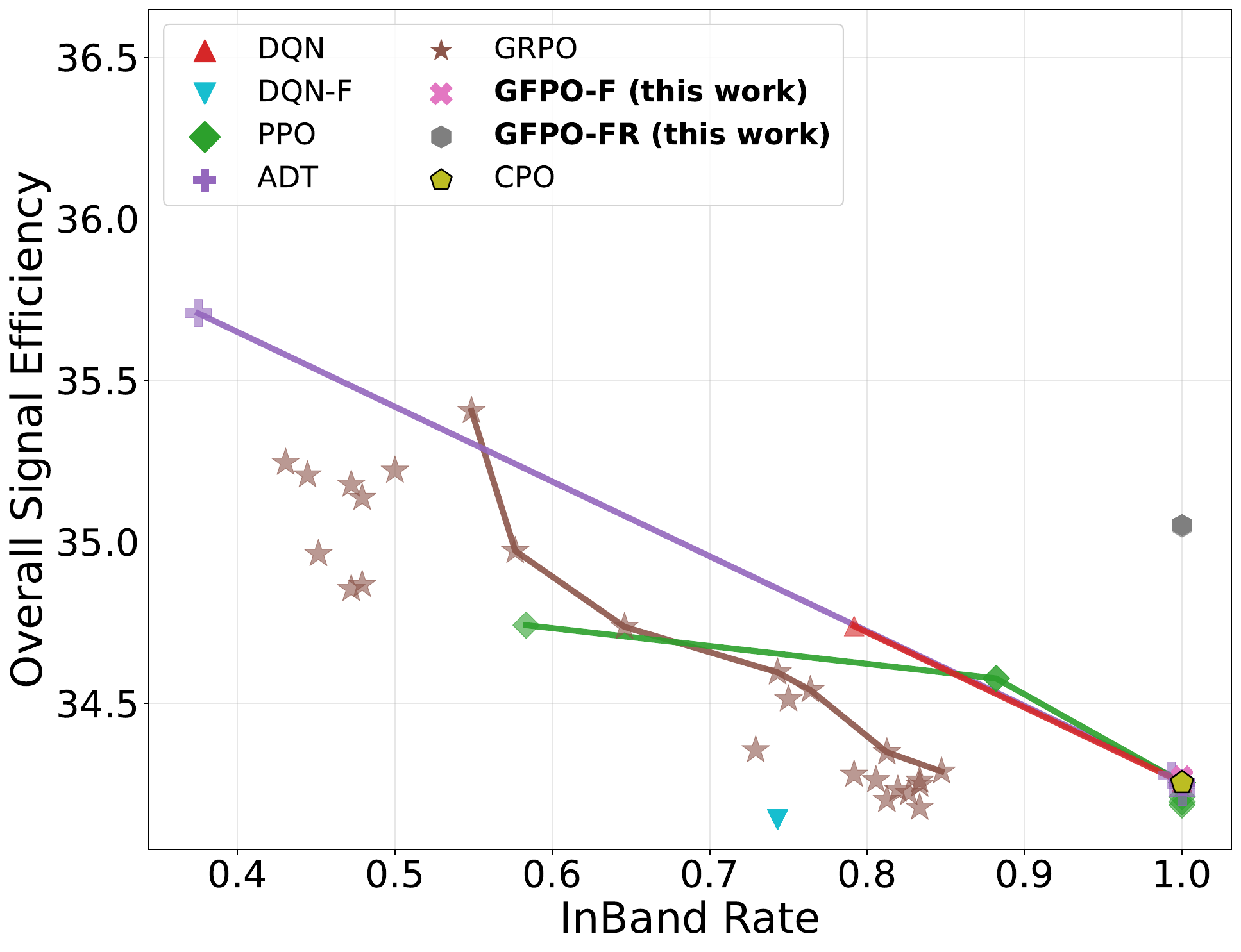}
        \caption{$H_{T}$ trigger (\haaFourB)}
        \label{fig:pareto_ht_haaForB}
    \end{subfigure}
    \caption{\textbf{Sensitivity Analysis of Reward Components for $H_{T}$ trigger (20\% MC).} Each point represents a ($\lambda_1$, $\lambda_2$) configuration from
  Equation~\ref{equation:reward_design}, with concave hulls connecting the upper envelope per method. The $x$-axis measures the fraction of chunks
  whose background rate falls within the tolerance band, and the $y$-axis measures overall signal efficiency. 
 Our methods (GFPO-F and GFPO-FR) collapse to
   a \emph{tighter} cluster in the upper-right corner, simultaneously achieving near-perfect inband rates and the highest signal efficiency, demonstrating
  strong robustness to the choice of $\lambda_{1,2}$.
  }
    \label{fig:pareto_plot_grid_ht_ttbar}
\end{figure*}



\subsection{Warmup: DQN for single trigger}
\label{sec:single_trigger}

\begin{table}[ht]
\centering
\scriptsize
\setlength{\tabcolsep}{3pt}
\renewcommand{\arraystretch}{1.25}
\sisetup{
  detect-weight=true,
  detect-family=true,
  table-number-alignment=center,
  round-mode=places,
  round-precision=3,
  separate-uncertainty=true
}
\caption{\textbf{Single-trigger control on MC.} Each RL policy is trained on the first 80\% of the MC chunks, then frozen and evaluated on the held-out remaining 20\%, on which all reported results are obtained. Background rates in percent units with target $r_{B}^*$ and tolerance $\tau$. Bold and \underline{underline} mark the best and second-best per column with standard deviation shown in parentheses. We repeat each method for 3 times and 0 standard deviation denotes $\approx$ 0 standard deviation. Our methods place top-two for all metrics on both triggers.}
\resizebox{\textwidth}{!}{%
\begin{tabular}{llccccccc
               }
\toprule 
Trigger & Method &
{MAE$\downarrow$} & {P95$|e|$ $\downarrow$} &
{InBand $\uparrow$} & {$\epsilon_{\text{ov}}^{\ttbarraw} \uparrow$} & {$\epsilon_{\text{ov}}^{\haaFourB} \uparrow$} &
{$\epsilon_{\text{in}}^{\ttbarraw} \uparrow$} &
{$\epsilon_{\text{in}}^{\haaFourB} \uparrow$}  \\
\midrule
\multicolumn{9}{l}{$H_{T}$ \textbf{trigger}}\\
& Constant & 0.083(0)  & 0.191(0) & 0.250(0) & 98.989(0) & 28.062(0)  & 99.129(0) & 29.418(0) \\
& PID \SEpaper & 0.029(0) & 0.070(0) & 0.521(0) & 99.388(0) & 33.289(0) & 99.294(0) & 31.444(0) \\
& DSPOT \citep{siffer2017anomaly} & 0.031(0) & 0.088(0) & 0.604(0) & 99.355(0) & 32.612(0) & 99.300(0) & 31.718(0) \\
& DQN \citep{mnih2015human}  &   0.010(0) & 0.023(0) & \textbf{1.000(0)} & 99.423(0) & 34.255(0) & 99.423(0) &34.255(0)  \\
& DQN-F & 0.018(0.007) & 0.038(0.013) & 0.743(0.223) & 99.409(0.069) & 34.141(1.085) & 99.417(0.011) & 34.184(0.186) \\
& PPO \citep{schulman2017proximal} & 0.013(0.005) & 0.029(0.010) & \underline{0.889(0.192)} & \underline{99.439(0.028)} & \underline{34.544(0.5)} & \underline{99.434(0.019)} & \underline{34.492(0.411)} \\
& ADT \citep{yang2024adt}     &
0.010(0) & 0.023(0) & \textbf{1.000(0)} & 99.423(0) & 34.255(0) & 99.423(0) &    
34.255(0)
\\
& GRPO \citep{shao2024deepseekmath}    &
0.024(0.001) & 0.052(0.006) & 0.632(0.067) & 99.422(0.21) & 34.396(0.318) & 99.401(0.054) & 34.106(1.218) \\
& L-GRPO & \underline{0.006(0.001)} & \underline{0.016(0.003)} & \textbf{1.000(0)} & 99.425(0.006) & 34.293(0.1) & 99.425(0.006) & 34.293(0.100)
\\
& CPO \citep{achiam2017constrained} & 0.010(0) & 0.023(0) & \textbf{1.000(0)} & 99.423(0) & 34.255(0) & 99.423(0) & 34.255(0) \\
\rowcolor{highlight}
& GFPO-F (Ours)    &
\bfseries 0.004(0) & \textbf{0.010(0)} & \textbf{1.000(0)} & 99.425(0.001) & 34.273(0.008) & 99.425(0.001) & 34.273(0.008) \\
\rowcolor{highlightDark}
& GFPO-FR (Ours)    &
0.016(0) & 0.024(0) & \textbf{1.000(0)} & \textbf{99.468(0)} & \textbf{35.057(0)} & \textbf{99.468(0)} & \textbf{35.057(0)}  \\
\midrule
\multicolumn{9}{l}{\textbf{AD trigger}} \\
& Constant & 0.042(0) & 0.131(0) & 0.479(0) & 94.279(0) & 25.404(0) & 94.758(0) & 25.285(0) \\
& PID \citep{emami2026selfdrivingtriggerlhcadaptive} & 0.020(0) & 0.051(0) & 0.729(0) & 95.233(0) & 27.298(0) & 94.990(0) & 26.287(0) \\
& DSPOT \citep{siffer2017anomaly} & 0.022(0) & 0.061(0) & 0.667(0) & 94.992(0) & 26.767(0) & 94.867(0) & 26.050(0) \\
& DQN \citep{mnih2015human} & 0.011(0.001) & 0.023(0.001) & \underline{0.993(0.012)} & 95.352(0.017) & 27.717(0.033) & 95.357(0.020) & 27.739(0.055) \\
& DQN-F & 0.016(0.009) & 0.030(0.015) & 0.847(0.265) & \underline{95.543(0.329)} & \underline{28.185(0.796)} & \underline{95.474(0.210)} & \underline{27.948(0.386)} \\
& PPO \citep{schulman2017proximal} & 0.011(0) & \underline{0.022(0)} & \textbf{1.000(0)} & 95.353(0) & 27.726(0) & 95.353(0) & 27.726(0) \\
& ADT \citep{yang2024adt}     &
0.011(0) & \underline{0.022(0)} & \textbf{1.000(0)} & 95.353(0) & 27.726(0) & 95.353(0) & 27.726(0) \\
& GRPO \citep{shao2024deepseekmath} &
0.013(0.001) & 0.025(0.001) & 0.958(0.055) & 95.347(0.027) & 27.733(0.071) & 95.356(0.011) & 27.740(0.041) \\
& L-GRPO & \underline{0.010(0)} & \underline{0.022(0)} & \textbf{1.000(0)} & 95.369(0.014) & 27.764(0.031) & 95.369(0.014) & 27.764(0.031) \\
& CPO \citep{achiam2017constrained} & 
0.011(0) & \underline{0.022(0)} & \textbf{1.000(0)} & 95.353(0) & 27.726(0) & 95.353(0) & 27.726(0) \\
\rowcolor{highlight}
& GFPO-F (Ours) &
\textbf{0.003(0)} & \textbf{0.007(0)} & \textbf{1.000(0)} & 95.442(0) & 27.926(0) & 95.442(0) &    
27.926(0) \\
\rowcolor{highlightDark}
& GFPO-FR (Ours) &
0.019(0) & 0.024(0) & 0.979(0) & \textbf{95.842(0)} & \textbf{28.875(0)} & \textbf{95.842(0)} & \textbf{28.868(0)} \\
\bottomrule
\end{tabular}
\label{tab:single_trigger_summary_compact}
}
\end{table}

\begin{table}[ht]
\centering
\scriptsize
\setlength{\tabcolsep}{3pt}
\sisetup{
  detect-weight=true,
  detect-family=true,
  table-number-alignment=center,
  round-mode=places,
  round-precision=3
}
\caption{
\textbf{Single-trigger control on CMS data (zero-shot transfer).} MC-trained policies from Table~\ref{tab:single_trigger_summary_compact} are frozen and deployed on real collision data without fine-tuning. GFPO-FR (ours) attains the highest signal efficiency on both triggers and the highest InBand on AD; GFPO-F (ours) attains the lowest MAE on both. See Appendix~\ref{appendix:real_data} for results on online training for CMS data.
}
\resizebox{\textwidth}{!}{%
\begin{tabular}{llccccccc
               }
\toprule
Trigger & Method &
{MAE$\downarrow$} & {P95$|e|$ $\downarrow$} &
{InBand $\uparrow$} & {$\epsilon_{\text{ov}}^{\ttbarraw} \uparrow$} & {$\epsilon_{\text{ov}}^{\haaFourB} \uparrow$} &
{$\epsilon_{\text{in}}^{\ttbarraw} \uparrow$} &
{$\epsilon_{\text{in}}^{\haaFourB} \uparrow$}  \\
\midrule
\multicolumn{9}{l}{\textbf{$H_{T}$ trigger}}\\
 & Constant & 0.118(0)  & 0.175(0) & 0.041(0) & 91.325(0) & 22.365(0) & 95.842(0) & 15.198(0) \\
 & PID \citep{emami2026selfdrivingtriggerlhcadaptive}      & 0.033(0) & 0.085(0) & 0.432(0) & 97.381(0) & 33.347(0)  & 97.497(0) & \textbf{35.242}(0)\\
 & DSPOT \citep{siffer2017anomaly}  & 0.163(0) & 0.072(0) & 0.432(0) & 97.211(0) & \underline{33.436(0)}  & 96.910(0) & 33.278(0) \\
 & DQN \citep{mnih2015human}   & 0.0173(0) & 0.037(0) & 0.865(0) & 97.510(0) & 33.213(0)  & 97.447(0) & 33.644(0) \\
& DQN-F  & 0.023(0) & 0.061(0) & 0.568(0) & 97.487(0) & 33.216(0)  & 97.360(0) & 32.672(0) 
\\
& PPO \citep{schulman2017proximal} & 0.016(0.003) & \underline{0.034(0.013)} & 0.838(0.11) & 97.491(0.15) & 33.222(0.20)  & 97.358(0.09) & 32.673(0.03) \\
& ADT \citep{yang2024adt}    &
0.019(0) & 0.045(0) & 0.743(0) & 97.412(0) & 33.286(0)  & 97.337(0) & 32.954(0) \\
& GRPO \citep{shao2024deepseekmath}     &
0.016(0.001) & 0.035(0.005) & 0.838(0.02) & 97.491(0.12) & 33.397(0.11) & 97.452(0.29) & \underline{33.698(0.90)} \\
& L-GRPO & 0.016(0.001) & 0.032(0.003) & 0.865(0.02) & 97.486(0.04) & 33.320(0.07) & 97.406(0.07) & 32.940(0.25) \\
& CPO \citep{achiam2017constrained} & \underline{0.014(0)} & \textbf{0.025(0)} &  0.905(0) & 97.498(0) & 33.139(0) & 97.392(0) & 32.901(0) \\
\rowcolor{highlight}
& GFPO-F (Ours)  
& \textbf{0.012(0)} & 0.042(0) & \textbf{0.99(0)} & \underline{97.587(0)} & 33.347(0) & \underline{97.508(0)} & 33.262(0) \\
\rowcolor{highlightDark}
& GFPO-FR (Ours) 
 & 0.015(0) & \underline{0.034(0)} & \underline{0.95(0)} & \textbf{97.734(0)} & \textbf{33.479(0)} & \textbf{97.821(0)} & 33.451(0) \\
\midrule
\multicolumn{9}{l}{\textbf{AD trigger}} \\
& Constant & 0.142(0) & 0.195(0) & 0.000(0) & 62.548(0) & 14.399(0) & NA & NA \\
& PID \citep{emami2026selfdrivingtriggerlhcadaptive}       & 0.035(0) & 0.080(0) & 0.405(0) & 75.053(0) & 39.191(0) & 76.210(0) & \textbf{44.573}(0) \\
& DSPOT \citep{siffer2017anomaly} & 0.106(0) & 0.087(0) & 0.419(0) & 74.622(0) & 38.297(0) & 75.387(0) & 41.843(0) \\
& DQN \citep{mnih2015human} &
0.027(0) & \underline{0.060(0.001)} & 0.527(0.02) & 75.093(0.06) & 39.930(0.13) & 76.013(0.51) & 44.178(1.51) \\
& DQN-F &
0.029(0) & 0.064(0) & 0(0) & 62.55(0) & 14.40(0) & NA & NA \\
& PPO \citep{schulman2017proximal} & 0.027(0) & 0.061(0) & 0.541(0) & 75.059(0) & 39.700(0)  & 75.601(0) & 43.153(0) \\
& ADT \citep{yang2024adt}     & 0.028(0) & 0.067(0) & 0.541(0) & 75.011(0) & 39.510(0)  & 75.614(0) & 42.930(0) \\
& GRPO \citep{shao2024deepseekmath} & 0.027(0) & 0.061(0.002) & 0.541(0.01) & 75.059(0.04) & 39.700(0.02) & 75.601(0.12) & 43.153(0.18) \\
& L-GRPO & 0.027(0) & 0.063(0.002) & 0.541(0.01) & 75.059(0.06) & 39.703(0.17) & 75.597(0.28) & 43.154(1.08) \\
& CPO \citep{achiam2017constrained} & 0.027(0) & 0.061(0) & 0.541(0) & 75.059(0) & 39.703(0) & 75.597(0) & 43.155(0) \\
\rowcolor{highlight}
& GFPO-F (Ours) 
& \textbf{0.022(0.001)} & \textbf{0.057(0.001)} & \textbf{0.689(0.04)} & \underline{75.562(0.26)} & \underline{40.083(0.25)} & \underline{76.239(0.25)} & 42.533(1.05) \\
\rowcolor{highlightDark}
& GFPO-FR (Ours) 
& \underline{0.025(0)} & \underline{0.060(0)} & \underline{0.608(0)} & \textbf{75.582(0)} & \textbf{40.309(0)} & \textbf{76.515(0)} & \underline{44.477(0)} \\
\bottomrule
\end{tabular}
}
\label{tab:single_trigger_summary_compact_realdata}
\end{table}

\textbf{Setting} Following \SEpaper, we study a single-trigger that accepts an event if a scalar score exceeds a threshold $\threshold_{t}$. The controller updates $\threshold_{t}$ once per chunk to keep the background acceptance rate at $r_{B}^{*} = 0.25\%$ - equivalent to 100 kHz under $r_{kHz} = 400 \cdot r$ - within tolerance $|r_t-r^{\star}_{B}|\le\tau$ with $\tau=0.025\%$ (i.e., $r_{t} \in [r^{-}, r^{+}] =  [90,110]\,\mathrm{kHz}$), while maximizing signal efficiency on \ttbar\ and \haaFourB~\citep{cella2024atlas}. We compare five controllers: (i) a constant menu threshold fixed from initial calibration, (ii) a PID controller updated per chunk on the background rate error (iii) DSPOT~\citep{siffer2017anomaly}, which fits a Generalized Pareto Distribution to the background score tail and sets thresholds at false-positive rate $q = r_{B}^{*}$ (Appendix~\ref{app:related_work_extended}), (iv) a DQN \citep{mnih2015human} controller, (v) DQN-F, an ablation in which the DQN is trained on a single calibration chunk and then frozen. We use $\gamma = 0.95$ for our valued-based DQN. Dataset description is in Appendix~\ref{appendix:dataset_summary}.

\noindent \textbf{RL formulation} The state has three parts (i) a length-$K$ window of recent background event scores together with the pileup context number of primary vertices $N_{pv}$; (ii) controller context ($\threshold_{t}, \hat{r}_{t}, \hat{r}_{t-1} , \Delta \threshold_{t-1}$) (iii) "near-threshold" occupancy features computed in narrow bands around the current cut, which estimate the local sensitivity $\frac{\partial r}{\partial \threshold}$.
The occupancy features carry information about distribution shape near the operating point that is unavailable to PID.

We present feature construction for state representations, sequential network architecture comparison (cf. Table~\ref{tab:rnn_trigger_comparison}) and the sequential network learning algorithm (Algorithm~\ref{alg:seq_state_and_controller}) in Appendix~\ref{sec:sequential_network_architectures}.  

\paragraph{Action and reward.} 
The agent picks a discrete increment $a_t=\Delta \threshold _t \in \actionspace$ and updates
\begin{equation}
\threshold_{t+1}=\mathrm{clip}\!\left(\threshold_t+\Delta\threshold_t,\,\threshold_{\min},\threshold_{\max} \right),
\qquad
r_{t+1}=r_B(\threshold_{t+1}),\quad
\epsilon_{t+1}^{i}=r_i(\threshold_{t+1}) \ 
\end{equation}

where $r_B(\cdot)$ is the background acceptance and $r_{i}$ the acceptance of signal $i \in \{ \ttbarraw\ , H\!\to\!AA\!\to\!4b \}$. 
We compute one reward per trigger $k\in\{H_T,\mathrm{AD}\}$. The trigger index $k$ and the signal index $i$ are distinct: $k$ denotes the trigger whose threshold the agent controls, while $i$ ranges over the physics signals in the efficiency term. All threshold-dependent quantities ($\threshold_t$, $r_{t+1}$, $\epsilon_{t+1}^{i}$) are evaluated for trigger $k$; the two triggers share the functional form below, so we state it once here and read $R_t^{k}$ off the trigger-specific quantities.

{\scriptsize
\begin{equation}
R_t^{k}
\;=\;
\lambda_1
\underbrace{
  \begin{cases}
    1 - \displaystyle\left(\frac{|r_{t+1} - r^\star_{B}|}{\tau}\right)^{\!2}
      & \text{if } |r_{t+1} - r^\star_{B}| \leq \tau \\[6pt]
    -\displaystyle\left(\frac{|r_{t+1} - r^\star_{B}|}{\tau} - 1\right)
      & \text{if } |r_{t+1} - r^\star_{B}| > \tau
  \end{cases}
}_{\text{rate tracking}}
\;+\;
(1 - \lambda_1)
\underbrace{
  \Big(\alpha\, \epsilon_{t+1}^{\ttbarraw}
  + (1-\alpha)\, \epsilon_{t+1}^{\haaFourB}\Big)
}_{\text{signal efficiency}}
\;-\;
\lambda_2
\underbrace{
  \left(\frac{|\Delta \threshold_{t-1}|}
             {\Delta \threshold_{\max}}\right)
}_{\text{move penalty}}.
\label{equation:reward_design}
\end{equation}
}
The capital $R_t^{k}$ is the reward, not the background rate $r$. Here $\Delta\threshold_{\max}$ is the maximum per-step update (the safety shield), $\tau$ is the tolerance width, and $\lambda_{1,2}>0$ trade off physics utility against smooth, rate-stable control. 
In real LHC operations, trigger thresholds cannot be adjusted arbitrarily. They must satisfy strict operational constraints arising from detector readout bandwidth~\citep{aaboud2017performance}, computing resources~\citep{atlas2020operation}, and detector safety limits~\citep{2137105}. For instance, the ATLAS Level-1 trigger rate must remain below approximately 100 kHz due to readout limitations, and trigger menus are designed such that thresholds and prescale factors maintain rates within these limits~\citep{aaboud2017performance}. Therefore, we introduce a safety shield $\Delta \threshold_{max}$ that restricts the controller to threshold configurations that satisfy these operational constraints~\citep{alshiekh2018safe, dawood2025dynamic}.

In practice, we select them via grid search over 
$\{0.0, 0.25, 0.5, 0.75, 1.0\}^{2}$. Figure~\ref{fig:pareto_plot_grid_ht_ttbar} reports the sensitivity analysis of $\lambda_{1,2}$ for $H_{T}$ trigger evaluated on 20\% MC.  Ablation study and hyperparameter 
sensitivity on both triggers are reported in Appendix~\ref{appendix:hyperparameter_ablation}. Sweeping $(\lambda_1, \lambda_2)$
over $\{0.0, 0.25, 0.5, 0.75, 1.0\}^2$ moves the convex weight between 3 coefficient terms for Equation~\ref{equation:reward_design} jointly, and DQN tracks this
trade-off rather than escaping it: its upper hull spans in-band rates from
$\approx 0.8$ to $1.0$ in Figure~\ref{fig:pareto_plot_grid_ht_ttbar}, so the operating point
is dictated by the weighting.

We will return to this figure in
Section~\ref{sec:gfpo}, where our proposed methods collapse the sweep to a single
high-feasibility, high-efficiency cluster. Here this figure establishes the baseline
difficulty that motivates the rest of the paper. 

\paragraph{Metrics.}
Let $\hat r_t$ be the observed background rate at chunk/step $t\in\{1,\dots,T\}$ and $e_t \;=\; \hat r_t - r_B^{*}.$
We report the mean absolute tracking error (MAE) $=T^{-1}\sum_{t} |e_t|,$
and the 95th percentile of the absolute error P$95|e|$, 
such that $95\%$ of control steps satisfy $|e_t|\le \mathrm{P95}|e|$, capturing rare large deviations. We report a tolerance band round the target $r^- \le \hat r_t \le r^+$, where typically $r^- = r_B^{*}-\tau$ and $r^+ = r_B^{*}+\tau$ for tolerance $\tau$.
Define the in-band fraction
$\mathrm{InBand} \;= T^{-1} \sum_{t} \mathbf{1}\!\left[r^- \le \hat r_t \le r^+\right]$
For each signal process $i\in\{\ttbarraw,\haaFourB\}$, let $N^{i}_{\text{pass}}(t)$ be the number of signal events accepted at step $t$ under the applied threshold $\threshold_t$, and $N^{i}_{\text{tot}}(t)$ be the total number of signal events available at step $t$.
We report two signal efficiencies, restricting to in-band steps or to all (i.e., overall) steps:
$
\epsilon_{\text{in}}^{i} = \frac{\sum_{t} \mathbf{1}[r^- \le \hat r_t \le r^+]\, N^{i}_{\text{pass}}(t)}{\sum_{t} \mathbf{1}[r^- \le \hat r_t \le r^+]\, N^{i}_{\text{tot}}(t)}$ and $
\epsilon_{\text{ov}}^{i} = \frac{\sum_{t} N^{i}_{\text{pass}}(t)}{\sum_{t} N^{i}_{\text{tot}}(t)}
$.

\noindent \textbf{Results}
Tables~\ref{tab:single_trigger_summary_compact} and~\ref{tab:single_trigger_summary_compact_realdata} report all controllers on MC and on CMS data with no fine-tuning; we discuss the non-learning baselines (constant menu and PID) and DQN here, and defer more variants to Section~\ref{sec:from_dqn_to_grpo} and Section~\ref{sec:gfpo}. For holdout MC samples, DQN improves InBand from 0.52 to 1.00 ($H_{T}$) and 0.73 to 0.99 (AD) over PID, with comparable or higher signal efficiency on both \ttbar\ and \haaFourB\ 
(Table~\ref{tab:single_trigger_summary_compact}). The frozen ablation DQN-F degrades sharply on $H_T$ as InBand only reaches 0.74 and AD InBand falls below the vanilla DQN controller (0.85 vs. 0.99). DQN-F shares the same network and inputs as DQN, with only the online updates removed. The gain thus comes from continued adaptation, not from a richer policy class. A basic DQN, however, does not suffice. Figures~\ref{fig:sensitivity_probe_over_time} and~\ref{fig:near_cut} show that the local slope $\frac{\partial r}{\partial \threshold}$ is non-stationary across the run and exhibits heteroskedastic, occasionally high-sensitivity regimes where small threshold changes induce large rate excursions. Two consequences follow. First, non-stationarity favors on-policy learning over off-policy value estimation with stale targets. Second, DQN requires storing thousands of $(s, a , r, s')$ transitions. For hardware implementations in LHC triggers, the memory is scarce and eliminating the replay buffer is a practical advantage. We present qualitative improvements on signal efficiency in Appendix~\ref{Appendix:dqn_additional}.

\subsection{Critic Free Policy Optimization}
\label{sec:from_dqn_to_grpo}
DQN's gain over PID isolates the value of online adaptation, but two properties of the trigger environment make value-based control fragile in deployment. First, background score distributions shift continuously with pileup and detector conditions (Figure~\ref{fig:score_summary} in Appendix~\ref{appendix:dataset_summary}), so replay data quickly becomes off-distribution. Second, as stated above, the local sensitivity $\frac{\partial r}{\partial \threshold}$ is non-stationary and occasionally large, forcing a single Q-function to extrapolate across rapidly changing local dynamics, a known failure mode for TD bootstrapping~\citep{andrychowicz2020matters}. PPO~\citep{schulman2017proximal,li2026turn} avoids replay by training a
separate value network, the \emph{critic}, estimating expected return
to compute advantages~\citep{konda1999actor, schulman2016gae}. Under drift
the critic itself becomes stale~\citep{pardo2018time}. We formalize this argument in Appendix~\ref{appendix:theory}. Under per-chunk drift, PPO's gradient bias scales with the
realized drift and compounds across chunks through critic bootstrapping.

We adapt Group Relative Policy Optimization (GRPO) \citep{shao2024deepseekmath, guo2025deepseek} to the streaming control setting, a memory-efficient PPO variant that removes the learned critic and instead uses \emph{within-group reward normalization}. By comparing a small group of candidate threshold updates under the current operating context, GRPO provides a relative learning signal that is less sensitive to reward-scale shifts and regime-dependent sensitivity, yielding more reliable updates under run drift. At each step $t$, GRPO samples $G$ candidate actions $a_{t}^{(g)} \sim \pi_{\theta_{\text{old}}} (\cdot | s_{t} )$ (we also write $\threshold_{t}^{(g)} = \threshold_{t} + a_{t}^{(g)}$ for the $g$-th candidate  when emphasizing its
role as a group member), evaluates each on a window $W_{t}$ of recent events, and forms advantages by within-group normalization:
\begin{equation}
A_t^{(g)} \triangleq  \frac{R_t^{(g)} - \mu_t}{\sigma_t + \varepsilon}, \qquad \mu_t \triangleq  \tfrac{1}{G}\sum_g R_t^{(g)},\; \sigma_t^2 \triangleq \tfrac{1}{G}\sum_g (R_t^{(g)} - \mu_t)^2.
\end{equation}
The reward $R_{t}^{(g)}$ reuses Equation~\ref{equation:reward_design} with window-level aggregates only, so each candidate evaluation costs a single pass over $W_{t}$, with no learned critic, no inference-time generation, and no auxiliary verifier model. The reward form is identical for the two triggers, so we run GRPO independently for each trigger $k\in\{H_T,\mathrm{AD}\}$ and suppress $k$ for the remainder of the method.
The policy is updated by maximizing \begin{equation}
\mathcal{L}_{\text{GRPO}}(\theta) = \mathbb{E}_t\!\left[\tfrac{1}{G}\sum_g A_t^{(g)} \log \pi_\theta(a_t^{(g)} \mid s_t)\right] - \beta\, \mathbb{E}_t\!\left[\mathrm{KL}\!\left(\pi_\theta \,\|\, \pi_{\theta_{\text{old}}}\right)\right].
\label{eq:loss_grpo}
\end{equation}
As advantages depend only on the \emph{relative} ranking of candidates within each group, GRPO is robust to absolute reward-scale shifts that destabilize critic-based methods.

\paragraph{Discussion on \boldmath$\gamma = 0$.} The default in RL is to
discount future reward, as our previous valued-based RL methods DQN and PPO use conventional non-zero discount $\gamma  = 0.95$. Here the discount is a
free choice rather than a property of the problem: three structural features make
this control problem myopic~\citep{jiang2015dependence}. First, reward is immediate: each action's effect lands on the
next batch, with no delayed rewards. Second, the only
action-dependent state is the threshold, which the policy resets at every step, so
look-ahead buys nothing that immediate correction does not. Third, the uncontrolled
part of the state drifts \emph{exogenously}, driven by beam conditions (changing pileup) rather than
by the policy. Together these make the greedy policy optimal, so every $\gamma$
induces the same optimum and $\gamma = 0$ is exact rather than approximate.
We verify this empircally: DQN at $\gamma = 0$ matches its $\gamma = 0.95$ counterpart,
which is why the conventional $\gamma = 0.95$ above is harmless. We set $\gamma = 0$
for GRPO because here it is not merely harmless but enabling: the return collapses
to the immediate reward, so the group-relative advantage estimates the per-step
advantage \emph{exactly}, with no value function. The one genuinely temporal concern is cross-step smoothness: a myopic policy could
chase the per-step optimum with large, oscillating threshold jumps. We suppress
this directly through the move-penalty term $\lambda_2\,(|\Delta c_{t-1}| /
\Delta c_{\max})$ in $R_t$ (Equation~\ref{equation:reward_design}), which charges each
step for the magnitude of its threshold change. Because this penalty enters the
immediate reward, smoothness is enforced at the step where the jump occurs rather
than through a discounted horizon, so $\gamma = 0$ for GRPO sacrifices nothing.
\paragraph{The constrained-RL response is insufficient.} The standard reaction to a constraint-violation problem is to make the constraint penalty adaptive. Indeed,
constrained RL~\citep{altman2021constrained, achiam2017constrained} replaces a fixed
penalty with a dual variable that adapts to observed violations. A natural
instantiation here replaces the fixed $\lambda_1$ (the background-rate penalty) in
Eq.~\ref{equation:reward_design} with a dual variable $\lambda_t$ updated by
projected dual ascent~\citep{tessler2018reward, stooke2020responsive}:
\begin{equation}
\lambda_{t+1} = \big[\lambda_t + \alpha_{\text{step}}(|r_t - r_B^\star| - \tau)\big]^+
\end{equation} 
This yields \emph{Lagrangian GRPO} (L-GRPO). On MC, L-GRPO roughly matches GRPO for AD ($0.958$ vs.\ $1.000$). Under distribution shift to CMS data, however, L-GRPO collapses to InBand $0.541$ for AD trigger (Table~\ref{tab:single_trigger_summary_compact_realdata}), and enabling test-time training provides no improvement in AD trigger ($0.541$ in Table~\ref{tab:single_trigger_summary_compact_realdata_online}). The failure mode is structural, not a hyperparameter issue (see Appendix~\ref{appendix:L_GRPO} for details). In $\sim\!30.8\%$ ($H_T$) and $\sim\!32\%$ (AD) of MC steps, the sampled group of $G$ candidates contains \emph{zero} rate-feasible actions (Figure~\ref{fig:grpo_failure_mode_ht} and Appendix~\ref{appendix:GRPO_Failure_Mode}). When this happens, every candidate is rate-infeasible (out-of-band), but group-relative normalization still assigns positive advantage to the above-mean candidates, all of which violate the rate budget, and the policy update \emph{reinforces} a rate violation. 

This limitation is not specific to the adaptive penalty. The natural response is
to reach for a more principled constrained method, and Constrained Policy
Optimization (CPO)~\citep{achiam2017constrained} is the canonical one: rather
than fold the rate constraint into the reward, it enforces the constraint
directly, solving a linearized program at each update subject to a trust-region
constraint on the Kullback--Leibler (KL) divergence between consecutive
policies~\citep{schulman2015trust}, which approximately satisfies the
expected-cost constraint with a bounded per-iteration violation
(Appendix~\ref{appendix:cpo}).  Does moving the
constraint out of the reward and into the update fix the failure?

It turns out that it does not. To isolate why, we give CPO the same critic-free,
group-relative advantage estimation used by GRPO, so that the
value-estimation step (unreliable in the rare-signal, high-imbalance regime)
cannot confound the comparison. CPO and GRPO then differ in
exactly one place: where the background rate constraint is enforced (See 
Appendix~\ref{appendix:cpo}). 
CPO's guarantee is
genuinely stronger than L-GRPO's. But the guarantee concerns the \emph{expected}
cost under a step in parameter space, not the feasibility of the candidates the
policy actually proposes, and here in LHC settings it is the latter that matters.

Consider what happens when the current policy concentrates on the
(background) rate-infeasible region, the same condition that produced the all-infeasible
groups above. CPO's linearized constraint becomes infeasible, and the algorithm
falls back to a recovery step~\citep[Eq.~14]{achiam2017constrained} that
descends purely along the cost direction and discards the reward gradient.
Under distribution shift the feasible region moves, and this fallback activates
on roughly half of the CMS chunks (AD trigger in Table~\ref{tab:single_trigger_summary_compact_realdata}). Intuitively, CPO spends half its updates
merely climbing back towards feasibility, never improving inband rate.

In our trigger the consequence is sharper still. Because the executed action is
screened from the policy (the shield out of band, and the dominant zero-move
action in band), these updates cannot change behavior at all. Indeed, test-time training reproduces the frozen-policy rollout to
floating-point precision (the CPO rows of
Tables~\ref{tab:single_trigger_summary_compact_realdata}
and~\ref{tab:single_trigger_summary_compact_realdata_online} are nearly identical;
see Appendix~\ref{app:cpo-degeneracy}). Like the adaptive penalty, CPO
acts on the objective and the parameter update, but \emph{never} on the proposal
distribution: the agent is asked to choose among bad options, and the choice it
makes is irrelevant to whether good options exist. The fix is to enforce feasibility \emph{before} advantage normalization, not via the reward. We develop this in Section~\ref{sec:gfpo}.
\vspace{-1mm}

\input{GFPO}
\input{Anomaly_detection_benchmarks}
\label{sec:RL_multitrigger}

%% file: GFPO.tex
\section{Group-Filtered Policy Optimization}
\label{sec:gfpo}
The previous section showed that GRPO and L-GRPO fail under distribution shift because the sampled candidate group frequently contains zero feasible actions. We address this by enforcing feasibility \emph{before} advantage normalization. This is the same algorithmic skeleton as Group Filtered Policy Optimization (GFPO)~\citep{shrivastava2025sample}, originally proposed for LLM reasoning, where candidates are filtered by response length and token efficiency. We adapt this skeleton to streaming trigger control, where the natural filter is rate feasibility, and develop two variants (GFPO-F, GFPO-FR) tailored to the rate--signal trade-off. The resulting method has two design choices: (i) which candidates to keep, and (ii) how to handle steps where too few feasible candidates exist.\footnote{Rollouts are sampled at temperature $T{=}1.0$, i.e., directly
from the policy's softmax distribution
$\pi_\theta(a \mid s_t) \propto \exp(z_a / T)$, where $z_a$ is the
unnormalized score (logit) the policy network assigns to action $a$,
without sharpening ($T{<}1$) or flattening ($T{>}1$).} 
\begin{figure}
    \centering 
    \begin{subfigure}{1.0\textwidth}
\includegraphics[width=\linewidth, trim={0.1cm 0.2cm 0 0.2cm}, clip]{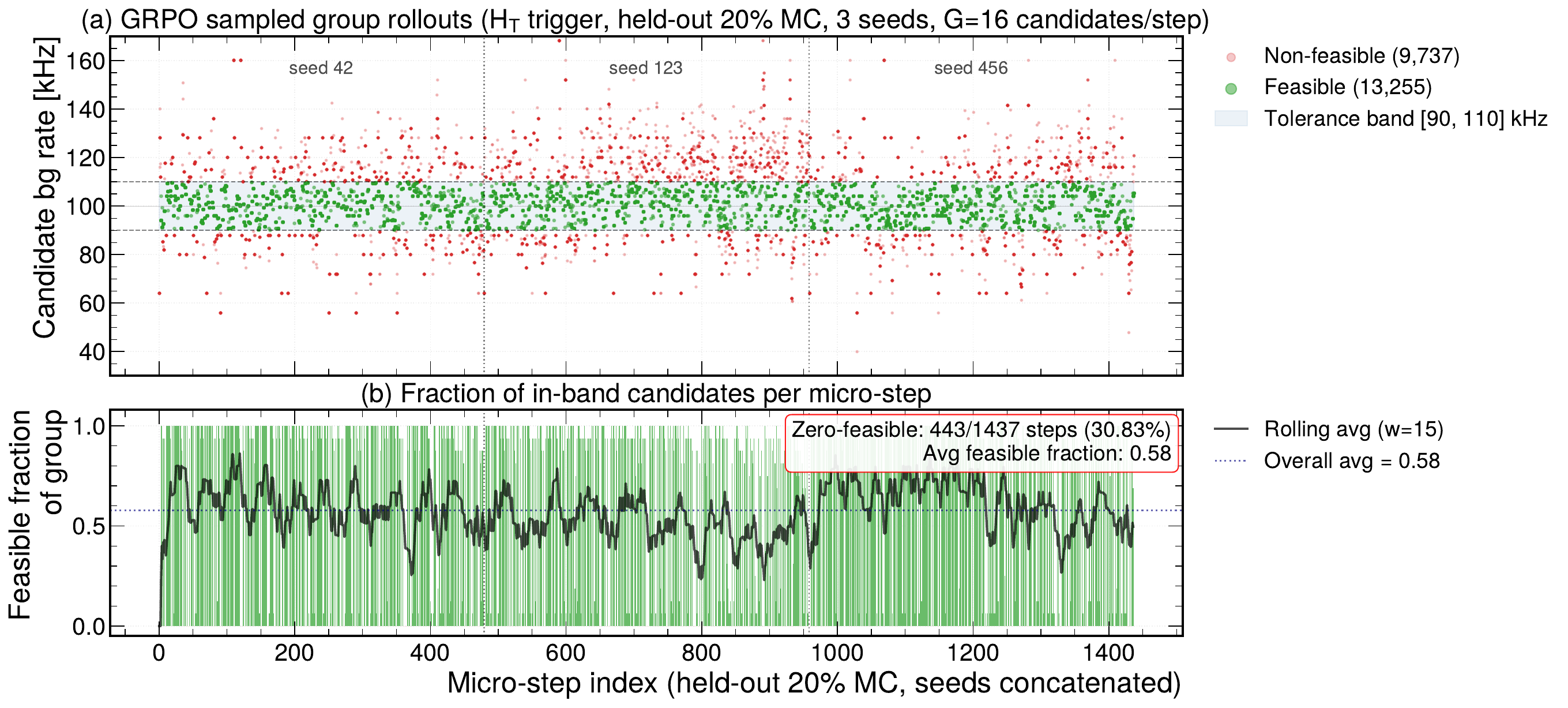}
    \end{subfigure}
    \caption{\textbf{GRPO's group-feasibility failure on $H_{T}$ (latter 20\% of MC, $G=16$) averaged across 3 seeds.} (a) Candidate background rates; \textcolor{darkgreen}{green}/\textcolor{red}{red} denote feasible/infeasible w.r.t. the tolerance band. (b) Per-step feasible fraction $f_t = n_{\text{feas}}/G$ (run mean $\langle f_t\rangle = 0.58$). $f_t{=}0$ on $30.8\%$ of steps, so the group has no in-band sample and GRPO reinforces the least-infeasible (still out-of-band) action. AD trigger is in Figure~\ref{fig:grpo_failure_mode_ad}. 
  }
  \label{fig:grpo_failure_mode_ht}
\end{figure}
\vspace{-1mm}
\paragraph{Feasibility set and kept set.} At step $t$, the feasibility set is
\begin{equation}
\mathcal{F}_t(\tau) \triangleq \{ g \in [G] : |\hat{e}_t(\threshold_t^{(g)})| \le \tau \},
\end{equation}
where $\hat{e}_t=r_{t,W_t} -r_{B}^{*}$ is the window-aggregated rate error. GFPO selects a kept set $\mathcal{K}_t \subseteq [G]$ with $|\mathcal{K}_t| = K$ and performs GRPO-style normalization over $\mathcal{K}_{t}$ with the same reward as Eq.~\ref{equation:reward_design}: 
\vspace{-1mm}
\begin{align}
A_t^{(g)} &\triangleq \frac{R_t^{(g)}-\mu_t}{\sigma_t+\varepsilon}, \mu_t \triangleq \frac{1}{|\textcolor{red}{\mathcal{K}_t}|}\sum_{g\in\textcolor{red}{\mathcal{K}_t}} R_t^{(g)},
\\
&\sigma_t^2 \triangleq \frac{1}{|\textcolor{red}{\mathcal{K}_t}|}\sum_{g\in\textcolor{red}{\mathcal{K}_t}}\left(R_t^{(g)}-\mu_t\right)^2.
\label{eq:gfpo_advantage}
\end{align}
and updates the policy with the GRPO objective restricted to $\mathcal{K}_t$  (Eq.~\ref{eq:loss_grpo} with
$g \in \mathcal{K}_{t}$). When $\mathcal{F}_{t} = \emptyset$ for either variant, the update is \emph{shielded}: skipped entirely to prevent reinforcing infeasible actions.
\paragraph{GFPO-F: rate-error ranking.} GFPO-F keeps the $K$ candidates with smallest absolute rate error (i.e. TopK), regardless of feasibility, biasing learning toward the band boundary and stabilizing training when feasible candidates are scarce:
\vspace{-0.5em}
\begin{equation}
\textcolor{red}{\mathcal{K}_t} = \mathrm{TopK}\big(\{-|\hat{e}_t(\threshold_t^{(g)})|\}_{g \in [G]}\big).
\label{eq:gfpo_f}
\end{equation}
\paragraph{GFPO-FR: feasible first then ranked by signal efficiency.}
GFPO-FR strictly enforces feasibility when possible, and then ranks by signal efficiency $\hat{\epsilon}_{t}^{\text{mix}} (\threshold_{t}^{g}) = \alpha \hat{\epsilon}_{t}^{\ttbarraw} + (1-\alpha) \hat{\epsilon}_{t}^{\haaFourB}$:
{\small
\begin{equation}
\textcolor{red}{
\mathcal{K}_t}
\triangleq
\begin{cases}
\mathrm{TopK}\Big(\ \hat\epsilon_t^{\mathrm{mix}}(\threshold_t^{(g)})\ \Big) \subseteq \mathcal{F}_t(\tau),
& \text{if } |\mathcal{F}_t(\tau)| \geq K,\\[4pt]
  \mathcal{F}_t(\tau) \cup \mathrm{TopK}\Big(\ -|\hat{e}_t(\threshold_t^{(g)})|\ \Big) \subseteq [G]\setminus\mathcal{F}_t(\tau),
& \text{if } 0 < |\mathcal{F}_t(\tau)| < K,\\[4pt]
  \mathrm{TopK}\Big(\ -|\hat{e}_t(\threshold_t^{(g)})|\ \Big) \subseteq [G],
& \text{if } \mathcal{F}_t(\tau) = \emptyset.
\end{cases}
\label{eq:gfpo_fr}
\end{equation}
}

When $0 < |\mathcal{F}_t| < K
$, the kept set is \emph{padded} with the closest-to-feasible out-of-band candidates. GFPO-FR seeks signal efficiency within the safe region; GFPO-F trades signal for tighter rate control.



\begin{wrapfigure}{l}{0.52\textwidth}
    \centering
    \includegraphics[width=\linewidth]{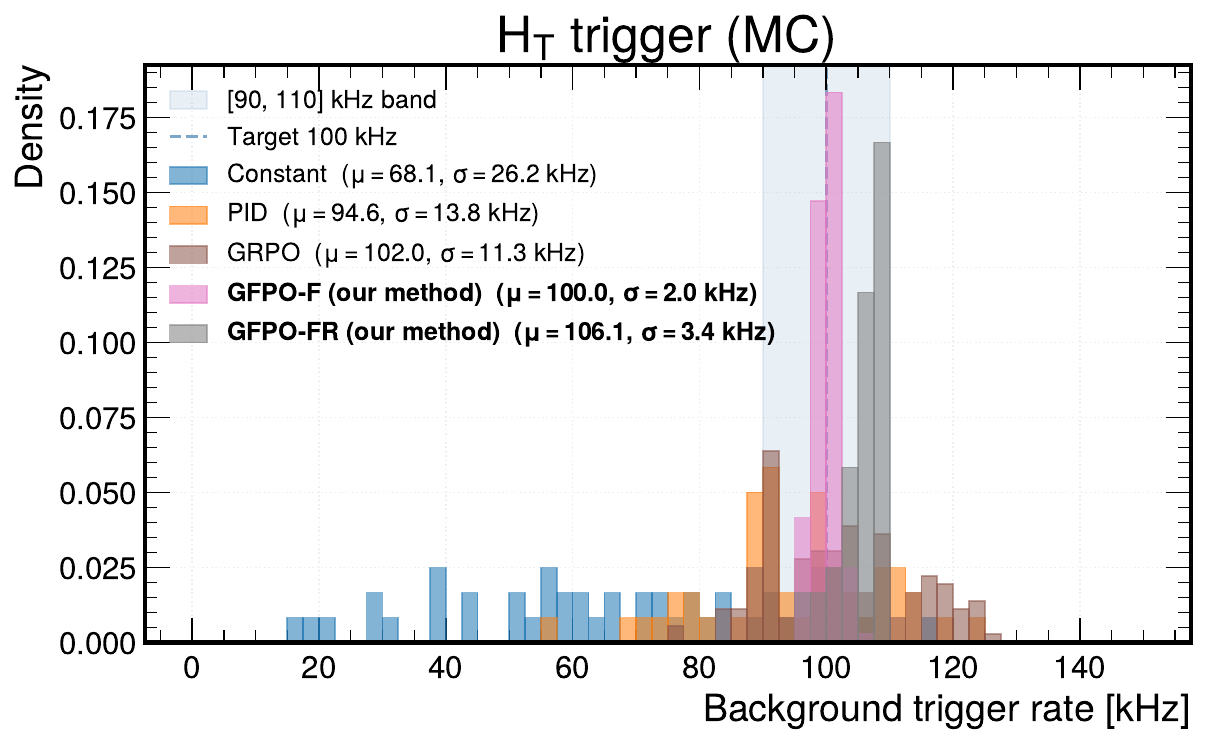}
    \caption{Background rates for $H_{T}$ (evaluated on 20\% MC). Our methods (GFPO-F and GFPO-FR) concentrate the background rate inside the $[90, 110]$~kHz
tolerance band around the $100$~kHz target, while the baselines spread well \emph{outside}
it. GFPO-F sits on target
($\mu = 100.0$~kHz) with the tighter distribution ($\sigma = 2.0$~kHz). GFPO-FR
runs at a slightly higher mean ($\mu = 106.1$~kHz) with a marginally wider spread
($\sigma = 3.4$~kHz).}
    \label{fig:bkg_rate_5_methods_histogram_mc}
\end{wrapfigure}
\paragraph{Diagnostics.} 
Figure~\ref{fig:bkg_rate_5_methods_histogram_mc} 
shows our methods 
concentrate background rates \emph{well} inside the $[90, 110]$~kHz band. Figure~\ref{fig:grpo_failure_mode_ht} directly quantifies this failure on the $H_T$
trigger over the held-out 20\% of MC: the average feasible fraction per
group $G$ is only $0.58$, and on $30.8\%$ of steps the sampled group contains
\emph{zero} rate-feasible candidates. On these zero-feasible steps,
group-relative normalization would still assign positive advantage to
above-mean candidates, such that vanilla GRPO reinforces the least-infeasible
(but still out-of-band) action.
\begin{wrapfigure}[22]{r}{0.48\textwidth}
    \centering
    \includegraphics[width=\linewidth, trim=0 0 0 5pt, clip]{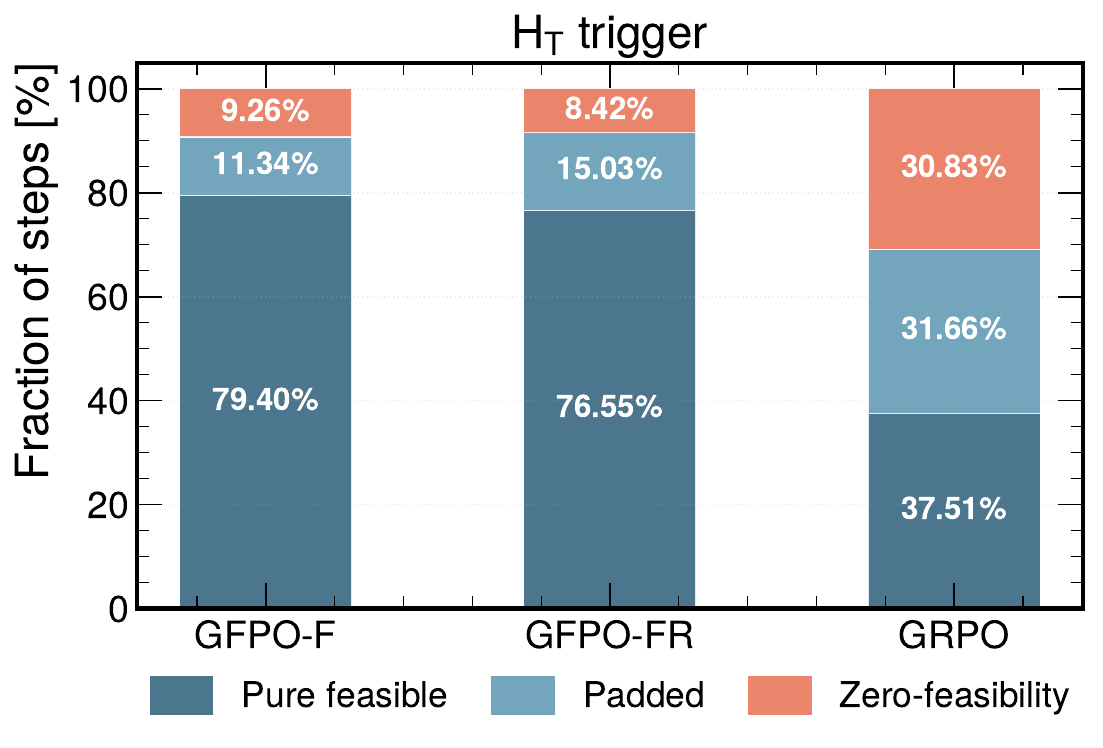}
    \caption{Average step composition for $H_T$ ($K{=}16$, $G{=}64$ for GFPO,
    $G{=}K{=}16$ for GRPO). Each step is classified by $|F_t|$ relative to $K$:
    Pure feasible ($|F_t|\geq K$, kept set $\mathcal{K}_t \subseteq \mathcal{F}_t$), Padded ($0<|F_t|<K$, kept set mixes feasible and out-of-band candidates), Zero feasibility ($\mathcal{F}_t = \emptyset$). 
    AD trigger is in Figure~\ref{fig:gfpo_intermediate_ad}.
    }
    \label{fig:gfpo_intermediate_ht}
\end{wrapfigure}
Figure~\ref{fig:gfpo_intermediate_ht} compares both filters with vanilla GRPO on $H_{T}$ (20\% MC).\footnote{Vanilla GRPO does not pad; We report its \textit{Padded} rate as the fraction of steps with any out-of-band candidate.} The two filters share the same architecture and differ only in their keep rule, which
is what induces the \emph{accuracy--coverage} trade-off we observe. GFPO-F keeps
the candidates closest to the target rate, collapsing the policy onto a tight,
target-centered action distribution; this yields the highest pure-feasible update
rate (79.40\%), at the cost of a slightly higher zero-feasibility rate (9.26\%),
since a narrow spread occasionally misses the band when the background events drift.
GFPO-FR instead ranks by signal, biasing the policy toward looser cuts that favor signal efficiency. Note that its executed rate deviation is about $4\times$ larger (MAE $0.016$ vs.\
$0.004$ on $H_T$ in Table~\ref{tab:single_trigger_summary_compact}), reflecting
operation nearer the edge of the band. Crucially, zero-feasibility groups, those with no
in-band candidate and hence no learning signal, remain rare for both
(8.42\% for GFPO-FR vs.\ 9.26\% for GFPO-F): GFPO-FR trades a few perfectly-feasible
groups for marginally fewer infeasible ones, buying signal efficiency at a small in-band cost (97.9\% vs.\ 100\% on AD in Table~\ref{tab:single_trigger_summary_compact}).

%% file: Anomaly_detection_benchmarks.tex
\section{Beyond HEP: Anomaly Detection Benchmarks}
\label{sec:anomaly_detection_benchmarks}
The LHC trigger problem is an instance of a more general one: a streaming controller must hold an aggregate false-positive rate inside a hard tolerance band while retaining rare signal events under drifting score distributions. \emph{Is the GRPO failure mode of Section~\ref{sec:gfpo} specific to particle physics domains, or a more generic property of group-sampled policy optimization under hard rate constraints?} 

We argue the latter. The failure structure is rate-constrained threshold control under drift, which appears wherever (i) decisions reduce to thresholding a scalar score, (ii) the background distribution drifts, and (iii) a hard, non-tunable rate budget governs allowable positives. The CMS trigger chain instantiates (iii) sharply: a collision rate $40\,\mathrm{MHz}$
 to $\sim 2-3 \mathrm{kHz}$ for permanent storage, bounded by the write speed and buffer capacity of the tape~\citep{arsuaga2021lhc}. 
Bandwidth is zero-sum, collisions are irreversible, and every false positive permanently displaces a signal event. Thus violations induce detector deadtime. None of (i), (ii), (iii) is physics-specific. The zero-feasibility failure mode arises in any setting satisfying them once the policy drifts off the feasible manifold.

We test this on two benchmarks. \textbf{UNSW-NB15}~\citep{moustafa2015unsw} preserves the full rate-constrained protocol and tests whether the failure mode and GFPO's resolution transfer to a non-physics domain. \textbf{NAB}~\citep{lavin2015evaluating} drops the rate constraint and reports standard anomaly-detection metrics, isolating whether our components (sequence-based state, adaptive thresholding) are competitive on more standard problems.

\subsection{Experimental Setup}
\label{sec:benchmark_setup}
\paragraph{Common protocol.} We fix every component except the dataset and context variable: an autoencoder trained on normal-class data provides the scalar score $s(x_i)$; thresholds are updated once per temporal chunk; a target false-positive rate $r^\star$ and tolerance $\tau$ with $\tau/r^\star \approx 10\%$ are pre-registered to match LHC stringency. We report InBand fraction, MAE, $\mathrm{P95}|e|$, and signal true-positive rate (TPR) using metrics defined in Section~\ref{sec:single_trigger}. Physics-specific features (pileup $N_{\mathrm{PV}}$, distance-to-cut in GeV) are replaced by domain-specific context variables. 
\vspace{-1mm}
\paragraph{UNSW-NB15.} UNSW-NB15 is a network
intrusion detection benchmark: each record is a network connection
(a flow of packets between two hosts) described by flow-level features
(duration, byte and packet counts, protocol), labeled as benign traffic or
as one of nine attack categories. UNSW-NB15 preserves the full rate-constrained protocol, providing a sim-to-real transfer test analogous to MC$\,\to\,$CMS: we train on Period~1 and deploy the frozen policy on Period~2. Benign network traffic is the background.  \emph{Exploits} serve as the easy signal and \emph{Backdoors} as the hard signal (analogous to $t\bar{t}$ and $h\to4b$). The context variable is the per-window connection rate, which tracks operating-regime shifts analogously to pileup. In the LHC setting we call this the background rate; in UNSW-NB15 the corresponding quantity is the \emph{false alert rate} (FAR), the per-chunk fraction of benign records flagged as anomalies. Preprocessing details and training setup are in Appendix~\ref{appendix:unsw_nb15}. 

\paragraph{NAB.} The Numenta Anomaly Benchmark is a standard benchmark for streaming anomaly detection, comprising real-world univariate time series
  that are sampled at regular intervals, hand-labeled with anomaly windows, and processed strictly online with no lookahead. We evaluate on its
  \texttt{realKnownCause} and \texttt{realAWSCloudwatch} categories, 24 streams of server and machine telemetry such as AWS CloudWatch CPU/network metrics,
  machine and ambient temperatures, and server CPU under known faults, since these best match the trigger setting of flagging anomalies on a single live
  metric under a controlled alert budget. Each stream is converted to a sliding-window z-score, split temporally into 70\% calibration and 30\% evaluation,
  and consumed in chunks of 100 timesteps, with detections scored against the labeled windows. NAB drops the rate constraint and reports standard detection metrics (precision, recall, F1). The agent optimizes a pure detection-quality reward (Appendix~\ref{app:nab}). The purpose is to isolate our sequence-based state and adaptive thresholding from feasibility-filtered policy optimization: if our gains came solely from the latter, NAB results would be flat. We define \emph{Constant-opt} as the oracle static threshold maximizing F1 in hindsight and methods underperforming it do not meaningfully leverage adaptivity. Constant-opt is well-defined for NAB but not for UNSW-NB15, where rate regulation and signal TPR create a tension no static threshold shall resolve.
\subsection{Experimental Results}
\label{sec:anomaly_detection_experiment_results}
\vspace{-1mm}
\noindent \paragraph{UNSW-NB15 results.} 


\begin{figure}[t]
    \centering
    \begin{subfigure}{0.496\textwidth}
        \centering
        \includegraphics[width=\linewidth]{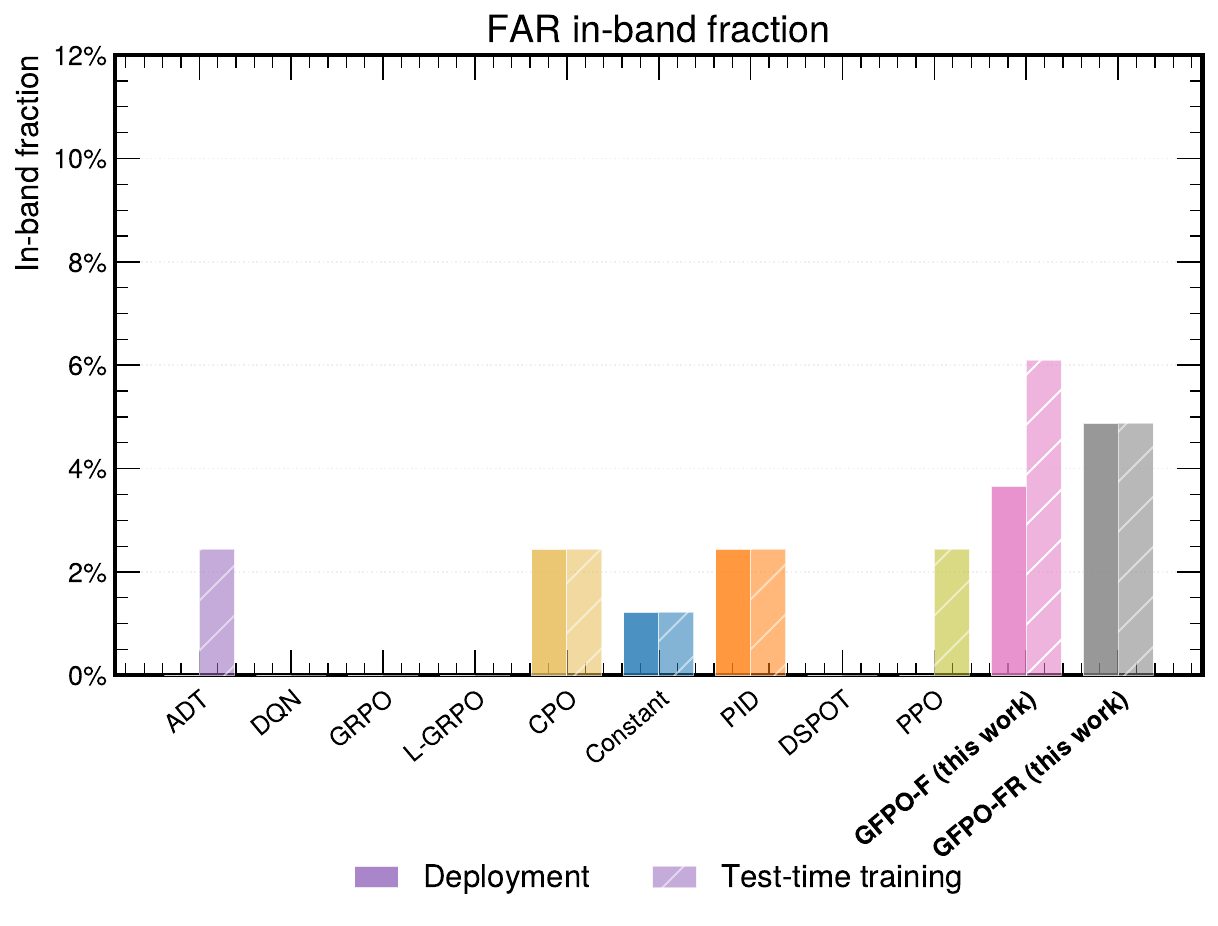}
        \caption{\textbf{GFPO dominates rate control on UNSW-NB15.} Period~1-trained policies deployed frozen on Period~2 (solid) or with online training (hatched). Fraction of chunks in the 
        \emph{FAR} (False Alert Rate) target band.}
        \label{fig:anomaly_usnw_nb15}
    \end{subfigure}
    \hfill
    \begin{subfigure}{0.482\textwidth}
        \centering
        \includegraphics[width=\linewidth]{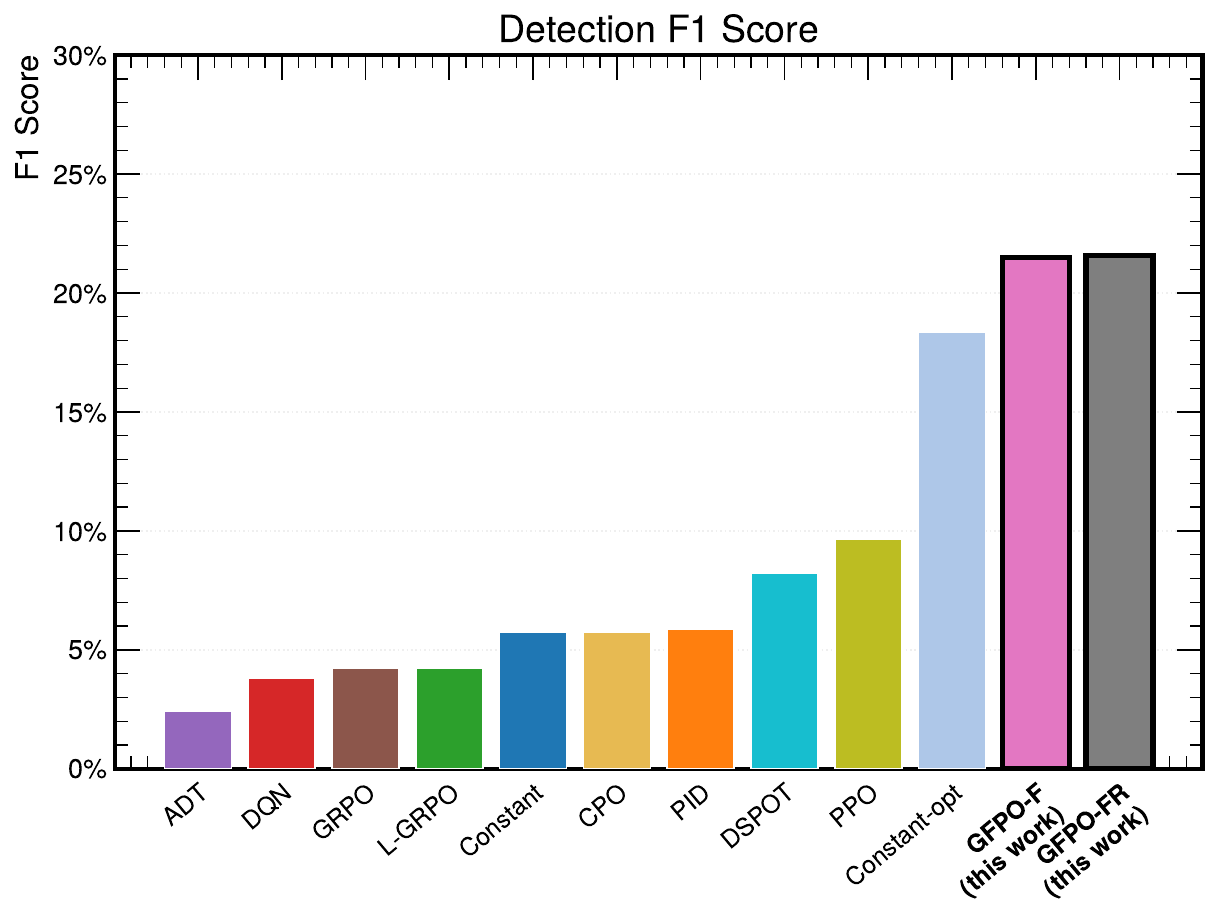}
        \caption{\textbf{GFPO improves detection quality on NAB even without a rate constraint.} F1 across methods. GFPO-F and GFPO-FR reach the highest F1 (0.215, 0.216), a 3\% gain over the oracle Constant-opt (F1=0.184), driven by precision.}
        \label{fig:anomaly_nab_f1}
    \end{subfigure}
    \caption{Anomaly detection results on UNSW-NB15 (left) and NAB (right).}
    \label{fig:anomaly_combined}
\end{figure}

Both predictions hold (Figure~\ref{fig:anomaly_usnw_nb15}). The zero-feasibility failure mode transfers: GRPO, L-GRPO, DQN, and DSPOT collapse to 0\% in-band FAR (False Alert Rate), with no recovery under online training.  GFPO's resolution transfers: GFPO-F and GFPO-FR reach nearly 4\% and 4.9\% in-band frozen, ranging between $1.5$ to $2\times$ the best classical baseline (2.5\% PID). Signal TPR on Backdoors, the hard class, is dominated by GFPO-F/FR (Appendix~\ref{appendix:unsw_nb15}).

\paragraph{NAB without rate regulation.} Does GFPO's advantage depend on an active rate constraint? NAB removes the constraint. GFPO-F and GFPO-FR attain the highest F1 ($0.215$, $0.216$; Figure~\ref{fig:anomaly_nab_f1}), approximately $3\%$ gain over the oracle Constant-opt (F1$=0.184$); DQN, GRPO, and L-GRPO underperform Constant-opt despite training on the same reward. The contribution is structural: group-based policy optimization with feasibility filtering, not RL itself, drives the gain. Precision and recall are dominated by GFPO-F/FR (Appendix~\ref{app:nab}).
\vspace{-1mm}

%% file: limitations.tex
\section{Conclusion}
\label{sec:summary_and_outlook}
\vspace{-1mm}

We cast online trigger threshold tuning as a streaming RL problem and evaluate it on Monte Carlo simulation and CMS data for $H_{T}$ and AD triggers. A simple DQN controller improves over constant menus and PID, but trigger non-stationarity (drifting backgrounds, sharp local rate sensitivity) undermines both critic-based DQN and Lagrangian-GRPO. Group-Filtered Policy Optimization 
enforces feasibility \emph{before} advantage normalization and yields consistent gains in rate stability and signal efficiency, with sim-to-real transfer to CMS data. The same failure mode and resolution carry over to UNSW-NB15 and NAB, suggesting feasibility-first filtering is the operative mechanism rather than a physics-specific artifact. Our framework has \emph{limitations} we discuss explicitly here. Our framework regulates a single scalar threshold per trigger, so signal efficiencies are coupled: \ttbar\ exceeds 95\% in-band efficiency while \haaFourB\ is limited to 29-36\% under the same cut (Table~\ref{tab:single_trigger_summary_compact}); the reward mixing parameter $\alpha$ offers only partial control. Our HEP evaluation is restricted to CMS Open Data (Run 283408 for 
deployment; Run 283876 for autoencoder training); generalization to 
other LHC experiments such as ATLAS~\citep{aad2024atlas} and 
LHCb~\citep{aaij2022comparison} is not tested. 

Overall, single-trigger results in this paper indicate why RL, and feasibility-first filtering in
particular, shall scale to larger menus better than PID. PID regulates each trigger
as an independent scalar loop, so an $N$-item menu becomes $N$ hand-tuned
controllers blind to the shared rate budget that couples triggers, the regime our
introduction flags as where per-loop tuning breaks down. The natural extension is
multi-trigger coordination with the feasibility filter over a joint threshold vector
under a shared budget~\citep{boggia2025review}. There the feasibility region only
tightens, so the L-GRPO failure intensifies and feasibility-first filtering
matters more, rather than less. Therefore, the single-trigger case is a conservative estimate of the
gap at trigger-menu scale.


%% file: Appendix/related_work_extended.tex
\section{Extended Related Work}
\label{app:related_work_extended}

\paragraph{RL for scientific decision making.}
RL is increasingly used across scientific domains, including experiment
planning in self-driving laboratories~\citep{volk2023alphaflow}, online
control of particle accelerators~\citep{kaiser2024reinforcement}, and
tokamak plasma regulation~\citep{degrave2022magnetic}. These settings
share the structural challenges of streaming data and non-stationary
dynamics, but they target a desired physical state. Our objective is different: \emph{rate-constrained operation},
maintaining a fixed background acceptance rate (e.g., $0.25\%$) within
tolerance while maximizing physics utility (signal efficiency). Search-heavy approaches
such as MCTS-based design~\citep{lutz2023top} are also a poor fit at
our deployment latency, which requires a single forward pass per
threshold update; training can nonetheless reuse logged or streaming
experience via off-policy learning.

The closest GRPO-adjacent work is~\citet{khanda2025extending}, which
extends GRPO to continuous control via trajectory-based clustering and
state-aware advantage estimation. Our setting differs in two ways: the
action is one-dimensional (a threshold update) rather than a
robotics-style continuous trajectory, and the deployment is compute-
and latency-constrained, so trajectory clustering ($k$-clustering in
particular) is incompatible with the per-chunk update budget. The rate
constraint itself motivates a different family of methods: constrained
policy optimization via Lagrangian relaxation~\citep{stooke2020responsive,tessler2018reward}
replaces fixed penalty weights with adaptive dual variables. We include
Lagrangian-GRPO (L-GRPO) as a direct baseline
(Section~\ref{sec:from_dqn_to_grpo},
Appendix~\ref{appendix:L_GRPO}) and show that dual adaptation alone
cannot resolve the structural sampling failure under distribution
shift, motivating our feasibility-filtering approach (GFPO-F and GFPO-FR).

\paragraph{Anomaly detection under drift and threshold control.}
Two families of work are most directly comparable. Active anomaly
detection~\citep{zha2020meta} learns acquisition policies under a
label-query budget; we instead adapt the operating threshold of an
autoencoder-based detector online, so the agent's interface is a
continuous threshold update rather than a discrete query decision. The
closest sequential-decision formulation is ADT~\citep{yang2024adt},
but its deployment assumptions are too permissive for trigger
operation: ADT assumes anomaly scores in $[0,1]$ with known global
extrema and pre-normalizes scores before inference. We train the
autoencoder on \emph{background-only} data, cannot bound score scale
\emph{a priori} (Figure~\ref{fig:score_summary}), and must adapt under
non-stationary shifts where any fixed threshold fails. We include ADT
as a baseline with the modifications described in
Appendix~\ref{appendix:baseline_description}.

We deliberately exclude the Anomaly Transformer~\citep{xu2022anomaly}.
Its core assumption is that anomalies in a time series can be detected
because their self-attention concentrates on nearby time steps, while
normal points attend broadly across the whole series, i.e., the Association
Discrepancy criterion. This only works when neighboring points are
informative about each other, i.e. when the series is regularly
sampled and autocorrelated. LHC bunch crossings are not such a series:
each collision is an independent event~\citep{evans2008lhc}, so the previous event tells us
nothing about the current one and ``neighbors'' in the input window
have no special status. The Association Discrepancy criterion therefore
has no signal to exploit on this data. Beyond this mismatch, our task
is rate regulation over streaming windows rather than point-wise
classification of individual events, which makes threshold-based RL
methods the appropriate comparison class.

A complementary non-learning baseline is Extreme Value Theory:
DSPOT~\citep{siffer2017anomaly} sets streaming thresholds by fitting a
Generalized Pareto distribution to the tail of the score distribution,
with false-positive level $q$ as the sole parameter. Identifying $q$
with the target background rate $r_B^*$ maps DSPOT directly onto our
setting, and its online tail re-estimation gives a principled
statistical response to score drift without any learned policy.
However, DSPOT is \emph{rate-agnostic}: it tracks observed score
quantiles rather than closing the loop on the measured background
acceptance rate, and offers no mechanism for signal-efficiency
optimization. We include DSPOT by using its adaptive threshold $z_q$
directly as the per-chunk trigger cut
(Appendix~\ref{appendix:baseline_description}); it is competitive with
PID in some regimes but falls short of RL-based controllers in both
rate stability and signal efficiency under pileup drift.

\paragraph{Data filtering and selective updates in RL.}
A line of recent RL methods for LLMs uses rejection sampling and
selective updates to stabilize training.
GFPO~\citep{shrivastava2025sample} oversamples candidates per prompt
and updates only on a filtered subset; DAPO~\citep{yu2025dapo}
discards prompts whose rollouts carry trivial learning signal;
analyses by~\citet{yue2026does} confirm that training on a selected
subset of trajectories can outperform updates on the full sample. In
all of these cases, filtering is task-agnostic: the criterion is
signal quality (informative reward, non-trivial rollouts), not
problem-side feasibility.

We adapt GFPO to the streaming trigger setting and replace its
task-agnostic filter with constraint-aware selection rules. GFPO-F
samples $G$ candidate threshold updates per control step and retains
the top-$K$ by smallest rate deviation $|r - r_B^*|$. GFPO-FR applies
a feasibility-first rule: it selects candidates inside an expanded
tolerance band, ranks the feasible set by signal utility (a weighted
mixture of signal efficiencies), and pads with closest-to-feasible
candidates when the feasible set is too small. The shared move across
both variants is that the filter becomes a soft realization of the
rate constraint rather than a generic data-quality heuristic for LLM training.

\paragraph{Physics-informed sequential state representations.}
A related line of work improves robustness in non-stationary physical systems by embedding domain structure into sequence models \citep{hausknecht2015deep,jia2019physics,daw2022physics,tang2022physics}. For example, \citet{jia2019physics} propose physics-guided recurrent architectures that inject physical priors into hidden-state evolution for more data-efficient and physically consistent time-series prediction. Similarly, \citet{tang2022physics} develop a physics-guided, physics-explainable recurrent model for optical resonance dynamics using multi-fidelity training with synthetic physics-model data and a smaller set of real measurements. PGNN-style approaches also integrate mechanistic knowledge via simulator-augmented inputs and physics-consistency regularization in supervised forecasting \citep{daw2022physics}. In contrast, we study \emph{online threshold control} under explicit rate constraints. Rather than forecasting system states under a known (or approximate) dynamical law, our policy operates a real-time trigger where the action is an operating threshold that must satisfy a hard background-rate tolerance under pileup-driven drift. We therefore construct a physics- and control-informed sliding-window observation: a length-$K$ sequence of recent event summaries augmented with threshold-relative features (distance-to-cut, pass/near-cut occupancy) and chunk-level control context (rate error, drift, last action, feasibility). This representation exposes the local ``mass near the cut'' that governs rate sensitivity (Figure~\ref{fig:near_cut}), enabling a sequence policy to anticipate distribution shifts and retune thresholds online without assuming a closed-form governing equation for the evolving score distributions.

\paragraph{Choice of policy optimizer.}
PPO with a learned critic and GAE-bootstrapped advantages~\citep{schulman2017proximal, schulman2016gae} is the standard choice for continuous-control policy optimization under stationary state-action distributions. Our streaming LHC trigger setting violates this assumption: the background score distribution and the local rate--threshold sensitivity drift across chunks, so value targets fit on past chunks are miscalibrated on the current chunk, and multi-epoch reuse compounds the error. GRPO~\citep{shao2024deepseekmath} removes the source of this compounding by replacing the critic with a within-chunk group baseline over $G$ on-policy rollouts, which is recomputed at every update. Because our control loop already evaluates multiple candidate threshold updates per chunk, GRPO's algorithmic structure aligns naturally with the deployment loop. Appendix~\ref{appendix:theory} formalizes the comparison: the critic carries stale value estimates across chunks, giving a PPO gradient bias of $\Theta(\delta_k/((1-\gamma)(1-\gamma\lambda)))$ in the realized drift, while the per-chunk group baseline yields a GRPO bias of $\mathcal{O}(G^{-1})$ that is invariant under per-chunk affine rescaling of returns. The same robustness extends to the rate-constrained Lagrangian formulation.

\input{Appendix/theory}

%% file: Appendix/theory.tex
\newpage
\section{Theoretical Justification}\label{appendix:theory}

\theoremstyle{remark}
\newtheorem{remark}{Remark}
\newtheorem{corollary}{Corollary}

\paragraph{Overview.}
Here, we formalize the claim that GRPO-style updates remain
calibrated under the streaming, non-stationary distribution shifts
characteristic of LHC trigger control, while PPO with a learned critic
acquires a drift-dependent bias that compounds across chunks.
Lemma~\ref{lem:critic-bias} bounds the PPO advantage error in terms of
critic mismatch; Lemma~\ref{lem:grpo-unbiased} establishes within-chunk
unbiasedness of the GRPO gradient at rate $\mathcal{O}(G^{-1})$;
Lemma~\ref{lem:affine} shows that the GRPO learning signal is invariant
under per-chunk affine rescaling of returns; Corollary~\ref{cor:bounds} combines these into per-chunk gradient bias
bounds for the two estimators.

\paragraph{Streaming non-stationary setting.}
We model the LHC trigger as a sequence of chunks $k=1,2,\dots$, each
inducing an MDP $\mathcal{M}_k=(\mathcal{S},\mathcal{A},P_k,r_k,\gamma)$
with chunk-specific value $V_k^\pi$, advantage $A_k^\pi$, and policy
objective $J_k(\theta) = \mathbb{E}_{\pi_\theta,\mathcal{M}_k}[\sum_{t\ge 0}\gamma^t r_t]$;
we assume $r_k\in[0,R_{\max}]$. Successive chunks are not identical:
both the reward and the transition kernel shift as luminosity, pileup,
and detector conditions evolve (Figs.~\ref{fig:score_summary} and~\ref{fig:score_summary_realdata}). We summarize the gap between
$\mathcal{M}_{k-1}$ and $\mathcal{M}_k$ by the single scalar
\begin{equation}\label{eq:delta-k}
\begin{split}
\delta_k \;\triangleq\;\; & \underbrace{\sup_{s,a}\big|r_k(s,a)-r_{k-1}(s,a)\big|}_{\text{reward drift}} \\
& {}+ \underbrace{\tfrac{\gamma R_{\max}}{1-\gamma}\sup_{s,a}\big\|P_k(\cdot|s,a)-P_{k-1}(\cdot|s,a)\big\|_{\mathrm{TV}}}_{\text{transition drift}},
\end{split}
\end{equation}
chosen so that one Bellman-difference step controls the induced
value-function gap. The prefactor $\gamma R_{\max}/(1-\gamma)$ equals an
upper bound on the span of any value function with $r\in[0,R_{\max}]$,
so a TV-norm perturbation in the kernel translates into a value
perturbation of matching units; the two terms of Equation~\ref{eq:delta-k}
then sum cleanly into the simulation-lemma bound
(\citealt{kearns2002near}, Lemma~4; \citealt{strehl2008analysis},
Lemma~1): for any policy $\pi$,
\begin{equation}\label{eq:value-perturb}
\big\|V_{k-1}^\pi - V_k^\pi\big\|_\infty \;\le\; \frac{\delta_k}{1-\gamma}.
\end{equation}
We sketch the argument for self-containedness. Let $T_k^\pi$ denote the
chunk-$k$ policy-evaluation Bellman operator,
$(T_k^\pi V)(s)=\mathbb{E}_{a\sim\pi(\cdot|s)}[r_k(s,a)+\gamma\,\mathbb{E}_{s'\sim P_k(\cdot|s,a)}V(s')]$,
of which $V_k^\pi$ is the unique fixed point. Adding and subtracting
$T_{k-1}^\pi V_k^\pi$ and applying the triangle inequality,
\begin{equation*}
\begin{aligned}
\|V_{k-1}^\pi-V_k^\pi\|_\infty
&= \|T_{k-1}^\pi V_{k-1}^\pi-T_k^\pi V_k^\pi\|_\infty \\
&\le \underbrace{\|T_{k-1}^\pi V_{k-1}^\pi-T_{k-1}^\pi V_k^\pi\|_\infty}_{\le\,\gamma\,\|V_{k-1}^\pi-V_k^\pi\|_\infty\ \text{(}\gamma\text{-contraction)}} \\
&\quad +\underbrace{\|T_{k-1}^\pi V_k^\pi-T_k^\pi V_k^\pi\|_\infty}_{\le\,\delta_k}.
\end{aligned}
\end{equation*}
The bound on the second term uses
$\sup_{s,a}|r_{k-1}-r_k|$ on the reward channel and the standard
inequality $|\mathbb{E}_P V_k^\pi-\mathbb{E}_Q V_k^\pi|\le\mathrm{span}(V_k^\pi)\,\|P-Q\|_{\mathrm{TV}}$
together with $\mathrm{span}(V_k^\pi)\le R_{\max}/(1-\gamma)$ on the
transition channel, recovering exactly $\delta_k$. Solving
$\|V_{k-1}^\pi-V_k^\pi\|_\infty\le\gamma\|V_{k-1}^\pi-V_k^\pi\|_\infty+\delta_k$
for the left-hand side yields Equation~\ref{eq:value-perturb}.

This is not a worst-case abstraction.
Figs.~\ref{fig:score_summary} and~\ref{fig:score_summary_realdata}
show that the background $H_T$ and AD-score distributions drift
monotonically across a run in both Monte Carlo and CMS Open Data,
exhibiting the affine location-and-scale shifts invoked in
Lemma~\ref{lem:affine}.

\paragraph{Critic drift biases the PPO advantage.}
Let $V_{\phi}$ be the critic fit on chunks $\{1,\dots,k-1\}$, and let
$\hat A_t^{\mathrm{GAE}}=\sum_{l\ge 0}(\gamma\lambda)^l\delta_{t+l}$,
where $\lambda\in[0,1]$ is the GAE parameter, with
$\delta_t = r_t + \gamma V_{\phi}(s_{t+1}) - V_{\phi}(s_t)$.
The following bound follows directly from the GAE estimator of
\citet{schulman2016gae}.
\begin{lemma}[Critic-induced advantage bias
]\label{lem:critic-bias}
For any $\pi$ and any chunk $k$,
\[
\Big|\mathbb{E}\!\left[\hat A_t^{\mathrm{GAE}} \,\big|\, s_t,a_t\right]
- A_k^{\pi}(s_t,a_t)\Big|
\;\le\; \frac{1+\gamma}{1-\gamma\lambda}\,\|V_{\phi}-V_k^{\pi}\|_{\infty}.
\]
\end{lemma}
\begin{proof}
Let $e(s)=V_\phi(s)-V_k^\pi(s)$. For the chunk-$k$ MDP,
\[
\mathbb{E}[\delta_t\mid s_t,a_t]
=
A_k^\pi(s_t,a_t)
+
\gamma\mathbb{E}[e(s_{t+1})\mid s_t,a_t]
-
e(s_t).
\]
For $l\ge 1$, since $a_{t+l}\sim\pi(\cdot\mid s_{t+l})$,
\[
\mathbb{E}[A_k^\pi(s_{t+l},a_{t+l})\mid s_t,a_t]=0.
\]
Therefore,
\begin{equation*}
\begin{aligned}
& \mathbb{E}[\hat A_t^{\mathrm{GAE}}\mid s_t,a_t]-A_k^\pi(s_t,a_t) = \\
&\qquad \sum_{l\ge 0}(\gamma\lambda)^l \mathbb{E}\!\left[\gamma e(s_{t+l+1})-e(s_{t+l}) \mid s_t,a_t\right].
\end{aligned}
\end{equation*}
Taking absolute values and using $\|e\|_\infty=\|V_\phi-V_k^\pi\|_\infty$ gives
\begin{equation*}
\begin{aligned}
\Big|\mathbb{E}[\hat A_t^{\mathrm{GAE}}\mid s_t,a_t]-A_k^\pi(s_t,a_t)\Big|
&\le \sum_{l\ge 0}(\gamma\lambda)^l(1+\gamma)\|e\|_\infty \\
&= \frac{1+\gamma}{1-\gamma\lambda}\|V_\phi-V_k^\pi\|_\infty.
\end{aligned}
\end{equation*}
\end{proof}
A simulation-lemma argument~\citep{kearns2002near} bounds
$\|V_{k-1}^{\pi}-V_k^{\pi}\|_{\infty}\le \delta_k/(1-\gamma)$, so even an
oracle critic on past chunks satisfies
$\|V_{\phi}-V_k^{\pi}\|_{\infty}=\mathcal{O}(\delta_k/(1-\gamma))$,
with a matching lower bound under monotone drift
(Remark~\ref{rmk:lower-bound}). Through
bootstrapping in $\delta_t$ and multi-epoch reuse of trajectories, this
bias propagates and \emph{compounds} across chunks.

\paragraph{Group baseline as a within-chunk control variate.}
On chunk $k$, draw $a_t^{(g)}\sim \pi_\theta(\cdot|s_t)$ and execute under
$\mathcal{M}_k$ to obtain returns $R_t^{(g)}$, $g=1,\dots,G$.

Define the within-chunk preconditioned policy gradient
\begin{equation}\label{eq:precond-grad}
\begin{split}
\widetilde\nabla J_k(\theta)
&\triangleq \mathbb{E}_{s\sim\pi_\theta,\mathcal{M}_k}\biggl[
\frac{1}{\sqrt{\Sigma_k(s)}+\varepsilon} \\
&\hspace{2em} \cdot \mathbb{E}_{a\sim\pi_\theta(\cdot\mid s)}\!\left[A_k^\pi(s,a)\nabla_\theta\log\pi_\theta(a\mid s)\right]
\biggr],
\end{split}
\end{equation}
where $\Sigma_k(s)\triangleq\mathrm{Var}_k(R\mid s)$. The state-dependent
factor $1/(\sqrt{\Sigma_k(s)}+\varepsilon)>0$ rescales each state's
contribution but preserves the per-state ascent direction; in the
sense of \citet{mei2022role} the within-chunk z-score acts as an
adaptive step-size.

\begin{lemma}[Group-normalized advantage consistency of GRPO]\label{lem:grpo-unbiased}
Let $\mu_t=\tfrac{1}{G}\sum_g R_t^{(g)}$,
$\sigma_t^{2}=\tfrac{1}{G}\sum_g (R_t^{(g)}-\mu_t)^2$, and
$\hat A_t^{(g)}=(R_t^{(g)}-\mu_t)/(\sigma_t+\varepsilon)$ (with $\varepsilon>0$ a small numerical-stability constant). Then
$\mathbb{E}[\mu_t\mid s_t]=V_k^{\pi}(s_t)$, and the group-normalized
advantage is asymptotically unbiased in conditional expectation,
\[
\mathbb{E}\!\left[\hat A_t^{(g)} \,\big|\, s_t, a_t^{(g)}\right]
\;\xrightarrow[G\to\infty]{}\;
\frac{A_k^{\pi}(s_t,a_t^{(g)})}{\sqrt{\mathrm{Var}_k(R\mid s_t)} + \varepsilon}.
\]

Consequently, the empirical per-state GRPO estimator
\[
\widehat g_{k,G}^{\mathrm{GRPO}}(s_t)
\triangleq
\frac{1}{G}\sum_{g=1}^G
\hat A_t^{(g)}\,
\nabla_\theta\log\pi_\theta(a_t^{(g)}\mid s_t)
\]
is biased by $\mathcal{O}(G^{-1})$ relative to the population
group-normalized update
\[
g_{k,\mathrm{norm}}^\star(s_t)
\triangleq
\mathbb{E}_{a\sim\pi_\theta(\cdot\mid s_t)}\!
\left[
\frac{A_k^\pi(s_t,a)}{\sqrt{\Sigma_k(s_t)}+\varepsilon}\,
\nabla_\theta\log\pi_\theta(a\mid s_t)
\right],
\]
with conditional MSE of order $\mathcal{O}(G^{-1})$, depending only on
within-chunk statistics of $\mathcal{M}_k$. Averaging $g_{k,\mathrm{norm}}^\star(s_t)$
over $s_t\sim\pi_\theta,\mathcal{M}_k$ recovers the preconditioned
gradient $\widetilde\nabla J_k(\theta)$ in~\eqref{eq:precond-grad}.
\end{lemma}
\begin{proof}
Conditional on $s_t$, the returns $\{R_t^{(g)}\}_{g=1}^{G}$ are i.i.d.\ over
the sampling $a_t^{(g)}\sim\pi_\theta(\cdot\mid s_t)$ and the chunk-$k$
rollout, with mean $V_k^\pi(s_t)$ and variance
$\Sigma_k(s_t)\triangleq\mathrm{Var}_k(R\mid s_t)$. Returns are bounded,
$R_t^{(g)}\in[0,R_{\max}/(1-\gamma)]$.

\emph{(i) Unbiasedness of $\mu_t$.} By linearity,
\[
\mathbb{E}[\mu_t\mid s_t]
=\frac{1}{G}\sum_{g=1}^{G}\mathbb{E}[R_t^{(g)}\mid s_t]
=V_k^\pi(s_t).
\]

\emph{(ii) Conditional numerator.} Fix $g$ and condition on
$(s_t,a_t^{(g)})$. For the realized index,
$\mathbb{E}[R_t^{(g)}\mid s_t,a_t^{(g)}]=Q_k^\pi(s_t,a_t^{(g)})$.
For $g'\neq g$, $a_t^{(g')}\sim\pi_\theta(\cdot\mid s_t)$ is sampled
independently of $a_t^{(g)}$, so
$\mathbb{E}[R_t^{(g')}\mid s_t,a_t^{(g)}]=V_k^\pi(s_t)$. Hence
\begin{equation*}
\begin{aligned}
\mathbb{E}[\mu_t\mid s_t,a_t^{(g)}]
&= \tfrac{1}{G}\,Q_k^\pi(s_t,a_t^{(g)})+\tfrac{G-1}{G}\,V_k^\pi(s_t), \\
\mathbb{E}\!\left[R_t^{(g)}-\mu_t\,\big|\,s_t,a_t^{(g)}\right]
&= \tfrac{G-1}{G}\,A_k^\pi(s_t,a_t^{(g)}).
\end{aligned}
\end{equation*}
\emph{(iii) Asymptotic limit.} The strong law gives
$\mu_t\xrightarrow{a.s.}V_k^\pi(s_t)$ and
$\sigma_t^{2}\xrightarrow{a.s.}\Sigma_k(s_t)$ as $G\to\infty$, so
$\sigma_t\xrightarrow{a.s.}\sqrt{\Sigma_k(s_t)}$. Boundedness of returns
yields the deterministic bound
$|\hat A_t^{(g)}|\le R_{\max}/[(1-\gamma)\varepsilon]$, hence uniform
integrability. Slutsky's theorem applied to the continuous map
$(x,y)\mapsto(R_t^{(g)}-x)/(y+\varepsilon)$ then gives
\[
\mathbb{E}\!\left[\hat A_t^{(g)}\,\big|\,s_t,a_t^{(g)}\right]
\;\xrightarrow[G\to\infty]{}\;
\frac{A_k^\pi(s_t,a_t^{(g)})}{\sqrt{\Sigma_k(s_t)}+\varepsilon}.
\]

\emph{(iv) Bias rate.} The conditional numerator carries the
multiplicative factor $(G-1)/G=1-\mathcal{O}(G^{-1})$ from (ii). For the
denominator, the sample variance with divisor $G$ has bias
$\mathbb{E}[\sigma_t^{2}\mid s_t]=\Sigma_k(s_t)\,(G-1)/G$, and a
first-order Taylor expansion of $x\mapsto 1/(\sqrt{x}+\varepsilon)$
around $\Sigma_k(s_t)$, valid by smoothness on $(0,\infty)$ and bounded
moments of $\sigma_t^{2}$, gives
\[
\mathbb{E}\!\left[\frac{1}{\sigma_t+\varepsilon}\,\bigg|\,s_t\right]
=\frac{1}{\sqrt{\Sigma_k(s_t)}+\varepsilon}+\mathcal{O}(G^{-1}).
\]
The remaining covariance term $\mathrm{Cov}(R_t^{(g)}-\mu_t,\,1/(\sigma_t+\varepsilon)\mid s_t,a_t^{(g)})$ vanishes to leading order after integration against $\nabla_\theta\log\pi_\theta(a_t^{(g)}\mid s_t)$ by the score-function identity, contributing $\mathcal{O}(G^{-1})$ to $g_k^{\mathrm{GRPO}}$ \citep{zhou2025demystifying}.
Combining the numerator and denominator expansions,
\[
\mathbb{E}\!\left[\hat A_t^{(g)}\,\big|\,s_t,a_t^{(g)}\right]
=\frac{A_k^\pi(s_t,a_t^{(g)})}{\sqrt{\Sigma_k(s_t)}+\varepsilon}
+\mathcal{O}(G^{-1}).
\]

Multiplying by $\nabla_\theta\log\pi_\theta(a_t^{(g)}\mid s_t)$ and
integrating: since $\Sigma_k(s_t)$ does not depend on $a$, the
denominator factors out of the action expectation, leaving the
per-state preconditioning $1/(\sqrt{\Sigma_k(s_t)}+\varepsilon)$.
Averaging $g_{k,\mathrm{norm}}^\star(s_t)$ over
$s_t\sim\pi_\theta,\mathcal{M}_k$ recovers the preconditioned
policy gradient $\widetilde\nabla J_k(\theta)$ defined
in Equation~\ref{eq:precond-grad}, so
$\bigl\|\mathbb{E}[g_k^{\mathrm{GRPO}}]-\widetilde\nabla J_k(\theta)\bigr\|=\mathcal{O}(G^{-1})$.

\emph{(v) Conditional MSE.} Each summand $\hat A_t^{(g)}\nabla_\theta\log\pi_\theta(a_t^{(g)}\mid s_t)$ is bounded; the $\mathcal{O}(G^{-1})$ MSE on chunk $k$ then follows from the exchangeable-sample analysis of the GRPO estimator in \citet{zhou2025demystifying}. Every quantity above depends only on $\mathcal{M}_k$, so no information from chunks $\{1,\dots,k-1\}$ enters the gradient.
\end{proof}

\paragraph{Affine equivariance.}
\begin{lemma}[Per-chunk affine equivariance]\label{lem:affine}
If $R_t^{(g)\prime}=\alpha_k R_t^{(g)}+\beta_k$ with $\alpha_k>0$, then
\[
\hat A_t^{(g)\prime}
=\frac{\alpha_k\bigl(R_t^{(g)}-\mu_t\bigr)}{\alpha_k\sigma_t+\varepsilon},
\]
\[
\hat A_t^{(g)\prime}-\hat A_t^{(g)}
=(R_t^{(g)}-\mu_t)\,\frac{(\alpha_k-1)\,\varepsilon}{(\alpha_k\sigma_t+\varepsilon)(\sigma_t+\varepsilon)}.
\]
In particular, $\hat A_t^{(g)\prime}=\hat A_t^{(g)}$ exactly when
$\varepsilon=0$, and otherwise
$\bigl|\hat A_t^{(g)\prime}-\hat A_t^{(g)}\bigr|=\mathcal{O}\!\bigl(\varepsilon\,|\alpha_k-1|/\sigma_t^{2}\bigr)$.
\end{lemma}
\begin{proof}
Under $R_t^{(g)\prime}=\alpha_k R_t^{(g)}+\beta_k$ with $\alpha_k>0$, the
group statistics transform as
\[
\mu_t'=\alpha_k\mu_t+\beta_k,\qquad
R_t^{(g)\prime}-\mu_t'=\alpha_k\bigl(R_t^{(g)}-\mu_t\bigr),
\]

\[
\sigma_t^{2\prime}=\alpha_k^{2}\sigma_t^{2},\qquad
\sigma_t'=\alpha_k\sigma_t,
\]
so
$\hat A_t^{(g)\prime}
=\alpha_k\bigl(R_t^{(g)}-\mu_t\bigr)\big/(\alpha_k\sigma_t+\varepsilon)$.
Subtracting $\hat A_t^{(g)}=(R_t^{(g)}-\mu_t)/(\sigma_t+\varepsilon)$ and
combining over the common denominator,
\begin{align*}
\hat A_t^{(g)\prime}-\hat A_t^{(g)}
&=(R_t^{(g)}-\mu_t)\!\left[\frac{\alpha_k}{\alpha_k\sigma_t+\varepsilon}
-\frac{1}{\sigma_t+\varepsilon}\right]\\
&=(R_t^{(g)}-\mu_t)\cdot
\frac{\alpha_k(\sigma_t+\varepsilon)-(\alpha_k\sigma_t+\varepsilon)}
{(\alpha_k\sigma_t+\varepsilon)(\sigma_t+\varepsilon)}\\
&=(R_t^{(g)}-\mu_t)\cdot
\frac{(\alpha_k-1)\,\varepsilon}{(\alpha_k\sigma_t+\varepsilon)(\sigma_t+\varepsilon)},
\end{align*}
which is the displayed identity. For $\varepsilon=0$ the difference
vanishes identically, giving exact equivariance.
For $\varepsilon>0$, boundedness of returns yields
$|R_t^{(g)}-\mu_t|=\mathcal{O}(R_{\max}/(1-\gamma))$, and the elementary
inequality $(\alpha_k\sigma_t+\varepsilon)(\sigma_t+\varepsilon)\ge\alpha_k\sigma_t^{2}$ gives
\[
\bigl|\hat A_t^{(g)\prime}-\hat A_t^{(g)}\bigr|
=\mathcal{O}\!\left(\frac{\varepsilon\,|\alpha_k-1|}{\sigma_t^{2}}\right).\qedhere
\]
\end{proof}
The PPO advantage in Lemma~\ref{lem:critic-bias} admits no analogous
equivariance because $V_\phi$ is not re-scaled at chunk $k$. From the
viewpoint of \citet{mei2022role}, the within-chunk z-score acts as an
\emph{adaptive step-size}.

\paragraph{Per-chunk bias bounds.}\begin{corollary}[Per-chunk bias of PPO and GRPO estimators]\label{cor:bounds}
Let $g_k^{\mathrm{PPO}}, g_k^{\mathrm{GRPO}}$ be the gradient estimators
on chunk $k$, $g_k^{\star}=\nabla J_k(\theta)$ the vanilla policy gradient
under $\mathcal{M}_k$, and $\widetilde\nabla J_k(\theta)$ the
preconditioned target~\eqref{eq:precond-grad}. Combining
Lemmas~\ref{lem:critic-bias} and~\ref{lem:grpo-unbiased}, for any fixed
group size $G$,
\[
\big\|\mathbb{E}[g_k^{\mathrm{PPO}}]-g_k^{\star}\big\|
=\mathcal{O}\!\Big(\tfrac{\delta_k}{(1-\gamma)(1-\gamma\lambda)}\Big),
\]
\[
\big\|\mathbb{E}[g_k^{\mathrm{GRPO}}]-\widetilde\nabla J_k(\theta)\big\|
=\mathcal{O}(G^{-1}).
\]
By Lemma~\ref{lem:affine}, the GRPO bias is additionally invariant
under per-chunk affine rescaling of $r_k$.
\end{corollary}
\begin{proof}
The PPO bound follows by combining Lemma~\ref{lem:critic-bias} with
$\|V_\phi - V_k^\pi\|_\infty = \mathcal{O}(\delta_k/(1-\gamma))$ from
Equation~\ref{eq:value-perturb}, then integrating against the bounded
score $\nabla_\theta\log\pi_\theta(a_t\mid s_t)$ over $\pi_\theta$. The
GRPO bound follows by integrating the per-state bias of
Lemma~\ref{lem:grpo-unbiased} over $s_t\sim\pi_\theta,\mathcal{M}_k$;
the integration step uses that $\Sigma_k(s_t)$ does not depend on $a$,
so the per-state denominator factors out of the action expectation and
the result matches the definition~\eqref{eq:precond-grad}. Affine
invariance is immediate from Lemma~\ref{lem:affine}.
\end{proof}
\begin{remark}[Tightness under monotone drift]\label{rmk:lower-bound}
Suppose the drift is monotone in the sense that there exists
$(s^\star,a^\star)$ with
$r_j(s^\star,a^\star)-r_{j-1}(s^\star,a^\star)\ge c\,\delta_j$ for all
$j\le k$ and some $c>0$. Then for any critic $V_\phi$ fit on chunks
$\{1,\dots,k-1\}$,
$\|V_\phi-V_k^\pi\|_\infty\ge c\,\delta_k/(2(1-\gamma))$,
since no convex combination of past values can match the chunk-$k$
value at $(s^\star,a^\star)$. The monotonic background-distribution
drift in Figs.~\ref{fig:score_summary},~\ref{fig:score_summary_realdata}
satisfies this condition.
\end{remark}

\begin{remark}[Rate-constrained objective]\label{rmk:rate-constrained}
Recall in Section~\ref{sec:rl_formulation_adaptive_thresholding}, we
model the trigger problem as a constrained MDP. We maximize
$J_k(\theta)$ subject to
$\mathbb{E}_{\pi_\theta,\mathcal{M}_k}[\rho_t] \le \rho_{\max}$, with
Lagrangian $\tilde R_t = r_t - \lambda_k(\rho_t - \rho_{\max})$.
Luminosity drift scales $r_t$ and $\rho_t$ jointly, so $\tilde R_t$
inherits the affine drift of Lemma~\ref{lem:affine}. The GRPO update
absorbs $\mathcal{L}_k$-rescaling (per-chunk luminosity) without
re-tuning $\lambda_k$; the PPO critic acquires the
Lemma~\ref{lem:critic-bias} bias on both reward and rate channels.
This matches the experiments: GRPO holds the rate--efficiency Pareto
front across luminosity sweeps, while PPO with GAE drifts off it.
\end{remark}
\paragraph{Discussion and limitations.}
Corollary~\ref{cor:bounds} characterizes the per-chunk bias of the two
estimators against their respective targets: PPO's bias against the
vanilla policy gradient $\nabla J_k(\theta)$ scales as
$\mathcal{O}(\delta_k)$ in the per-chunk drift, while GRPO's bias
against the preconditioned target $\widetilde\nabla J_k(\theta)$ scales
as $\mathcal{O}(G^{-1})$ in the group size, drift-independent. The
preconditioning factor $1/(\sqrt{\Sigma_k(s)}+\varepsilon)>0$ preserves
per-state ascent direction, so $\widetilde\nabla J_k$ shares ascent
direction with $\nabla J_k$ at every state. The two bounds measure
different things and we draw no direct comparative conclusion.

In summary, on chunk $k$ the PPO estimator has bias
$\mathcal{O}(\delta_k/((1-\gamma)(1-\gamma\lambda)))$ against
$\nabla J_k(\theta)$, while the GRPO estimator has bias
$\mathcal{O}(G^{-1})$ against the preconditioned target
$\widetilde\nabla J_k(\theta)$ — drift-independent and additionally
invariant under per-chunk affine rescaling of $r_k$. The cost is that
GRPO targets a state-preconditioned gradient rather than the vanilla
one, and requires $G$ on-policy samples per state.

%% file: Appendix/pseudocode.tex
\newpage
\section{GFPO for Streaming Trigger Control: Pseudocode}
\label{app:GFPO_Pseudocode}
\newcommand{\sctt}[1]{\textnormal{\textsc{#1}}}
\begin{algorithm}[ht]
\caption{GFPO for streaming trigger control at LHC: MC training and CMS deployment}
\label{alg:gfpo_pipeline}
\KwIn{Data Streams $\mathcal{D}_{\text{MC}}, \mathcal{D}_{\text{CMS}}$; variant $v \in \{\text{GFPO-F}, \text{GFPO-FR}\}$;
group size $G$, kept size $K$; target background rate $r_B^*$, tolerance $\tau$; MC passes $E$;
deployment mode $\in \{\sctt{Frozen/Deployment}, \sctt{TTT (Test-time Training)}\}$.}
\KwOut{Trained policy $\pi_\theta$; threshold trajectory $\{\threshold_t\}_{t \in \mathcal{D}_{\text{CMS}}}$.}
\BlankLine
\textbf{Procedure} \sctt{GFPO-Step}$(s_t, \threshold_t)$\;
\Indp
Sample actions $\{a_t^{(g)}\}_{g=1}^{G} \sim \pi_{\theta_{\text{old}}}(\cdot \mid s_t)$;
form candidates $\threshold_t^{(g)} \leftarrow \threshold_t + a_t^{(g)}$\;
Evaluate background rate errors for $\threshold_t^{(g)}$ and calculate rewards $\{R_t^{(g)}\}$
\tcp*{Eq.~\ref{equation:reward_design}}
$\mathcal{F}_t(\tau) \leftarrow \{g \in [G] : |\hat{e}_t(\threshold_t^{(g)})| \leq \tau\}$\;
$\mathcal{K}_t \leftarrow$ kept set per variant $v$
\tcp*{Eqs.~\ref{eq:gfpo_f}/\ref{eq:gfpo_fr}}
\If{$\mathcal{K}_t \neq \emptyset$}{
  $\theta \leftarrow \theta + \eta \nabla_\theta \mathcal{L}_{\text{GRPO}}(\theta; \mathcal{K}_t)$
  \tcp*{Eqs.~\ref{eq:gfpo_advantage},~\ref{eq:loss_grpo}}
}
\KwRet{$\Delta\threshold_t$ from $\mathcal{K}_t$ (or sampled from $\pi_\theta$ if $\mathcal{K}_t = \emptyset$)}\;
\Indm
\BlankLine
Initialize $\pi_\theta$\;
\textbf{Phase 1: MC training}\;
\For{$e \leftarrow 1$ \KwTo $E$, chunk $t \in \mathcal{D}_{\text{MC}}$}{
  $s_t \leftarrow \sctt{BuildSequentialState}$ \tcp*{Alg.~\ref{alg:seq_state_and_controller}}
  $\Delta\threshold_t \leftarrow \sctt{GFPO-Step}(s_t, \threshold_t)$\;
  $\threshold_{t+1} \leftarrow \threshold_t + \mathrm{clip}\!\bigl(\Delta\threshold_t,\,[-\Delta\threshold_{\max},\,\Delta\threshold_{\max}]\bigr)$\;
}
\BlankLine
\textbf{Phase 2: CMS deployment}\;
\For{chunk $t \in \mathcal{D}_{\text{CMS}}$}{
  $s_t \leftarrow \sctt{BuildSequentialState}$\;
  \eIf{mode $=$ \sctt{Frozen/Deployment}}{
    $\Delta\threshold_t \leftarrow \arg\max_a \pi_\theta(a \mid s_t)$ \tcp*{no gradient update}
  }{
    $\Delta\threshold_t \leftarrow \sctt{GFPO-Step}(s_t, \threshold_t)$ \tcp*{TTT: one update per chunk}
  }
  $\threshold_{t+1} \leftarrow \threshold_t + \mathrm{clip}(\Delta\threshold_t, [-\Delta \threshold_{\max}, \Delta\threshold_{\max}])$\;
}
\end{algorithm}

%% file: Appendix/Experimental_setup.tex
\newpage
\section{Experimental Setup for LHC}
\subsection{Dataset Summary}
\label{appendix:dataset_summary}

\paragraph{Datasets.} We use the datasets of \citet{emami2026selfdrivingtriggerlhcadaptive}. Sample sizes are listed in Table~\ref{tab:dataset_summary}.
\ttbar\ and \haaFourB\ are shared for both MC and real CMS collision data setting. 

\begin{table*}[!htbp]
\centering
\caption{Dataset summary}
\label{tab:dataset_summary}
\begin{tabular}{cc}
\toprule
Sample & Events \\
\midrule
MC Background & 9,794,099 \\
\ttbar\ Signal & 2,233,999 \\
\haaFourB\ Signal & 1,102,412 \\
CMS Run 283876 Background (AD Trigger Training) & 664,475 \\
CMS Run 283408 Background (Evaluation) & 1,987,943 \\
\bottomrule
\end{tabular}
\end{table*}

\paragraph{Score drift.} Figures~\ref{fig:score_summary}and~\ref{fig:score_summary_realdata} report the running mean, median, and central 5--95\% band of background $H_T$ and AD scores over MC and CMS Run 283408. Both streams drift downward over the run. MC drifts more strongly toward zero, while CMS Run 283408 declines mildly and retains a heavy upper tail. CMS Run 283408 is also globally higher-scoring than MC, a domain shift on top of the within-run drift.

\paragraph{Autoencoder protocol.} Following \citet{emami2026selfdrivingtriggerlhcadaptive}, we use latent dimension $d=2$. For real data, we train on CMS Run 283876~\citep{cms_opendata_67840} (the second-longest 2016 run~\citep{cernopendata_aboutcms}) and evaluate AD scores on CMS Run 283408 (the largest 2016 run; 1.99M background events). For MC, we train on one background sample and evaluate on an independent background sample.

\paragraph{Signal/background overlap (Figure~\ref{fig:signal_background_overlap}).} \ttbar\ events are kinematically well-separated from MinBias background under both $H_T$ and the AD score, lying predominantly to the right of the oracle cut. \haaFourB\ events overlap heavily with the background: each $b$-quark carries $\sim m_H/4 \approx 31$~GeV~\citep{de2017cern, aaboud2018search}, so $H_T$ and AD scores fall in the bulk of the background distribution and a large fraction of signal events sits below the oracle cut at any operating point. Notably, this ceiling is \emph{intrinsic} to the signals: for $H_{T}$ trigger, every method in 
Tables~\ref{tab:single_trigger_summary_compact} 
and~\ref{tab:single_trigger_summary_compact_realdata} reaches $> 97\%$ on 
\ttbar\ (except Constant in Table~\ref{tab:single_trigger_summary_compact_realdata}) but stays below $\sim\!35\%$ on \haaFourB\ for $H_T$ (and analogously 
for AD), tracking the oracle bounds in 
Table~\ref{tab:oracle_signal_eff_mc_realcms}. Note to keep it consistent with the main paper, we only report results for 20\% MC with Table~\ref{tab:oracle_signal_eff_mc_realcms} as it involves RL training.

Unlike the RL evaluations, which train on the first 80\% of MC chunks and report performance on the held-out final 20\% to prevent train-test leakage, Figure~\ref{fig:signal_background_overlap} uses the full MC sample because essentially \emph{no} learning is performed. That figure serves only as an oracle study of the underlying signal-background separation achievable with each trigger. Same setup applies to Figure~\ref{fig:signal_background_overlap_cms} for real CMS data. Therefore, the maximum achievable signal efficiency for both Figures~\ref{fig:signal_background_overlap} and ~\ref{fig:signal_background_overlap_cms} are different from oracle signal efficiencies for Tables~\ref{tab:oracle_signal_eff_mc} and ~\ref{tab:oracle_signal_eff_mc_realcms}.

\paragraph{Oracle cut.} The oracle cut $r^{+}$ is the threshold $\threshold$ for which the background rate equals $r^{+} = r_B^{*}+\tau = 0.275\%$, the tightest cut still inside the upper edge of the tolerance band. Its signal efficiency is the maximum achievable by any deterministic threshold rule. Table~\ref{tab:oracle_signal_eff_mc_realcms} reports oracle efficiencies at both $r_B^{*}$ and $r^{+}$ across triggers and signals. Same as \citet{emami2026selfdrivingtriggerlhcadaptive}, we exclude the first 10 chunks (initial threshold calibration) and report results for only events past the first 500K (MC) and 200K (CMS Run 283408).

\paragraph{Analyzing the signal efficiency gap between oracle and GFPO variants.} Tables~\ref{tab:oracle_signal_eff_mc} and~\ref{tab:oracle_signal_eff_realcms} report the gap between our methods and the oracle at $r^{+}$. For MC, GFPO-FR achieves the smallest gap with the oracle compared to GFPO-F. On CMS Run 283408 for \haaFourB\ signal, both GFPO variants (42.533 and 44.477) \emph{exceed} the oracle (41.028) in-band efficiency: in-band chunks are not uniformly distributed but concentrate in the high-efficiency tail of the run (Figure~\ref{fig:cms_inband_drift_h4b}).

\begin{table*}[!htbp]     
\centering                                       
\setlength{\tabcolsep}{4pt}                      
\renewcommand{\arraystretch}{0.95}    \caption{Overall oracle signal efficiencies for 20\% MC and CMS Run 283408 at $r_{B}^{*}=0.250\%$ and the      
loosest feasible rate $r^{+} = r_B^{*}+\tau=0.275\%$. To avoid training and evaluation overlapping, we train on 80\% MC and report below results on 20\% MC. For oracle performance, the default way to calculate signal efficiency is to assume every threshold is in-band (i.e., the background rate stays within the tolerance band), and dashed lines represent the same values for both overall and inband signal efficiency.} 
{\footnotesize
\begin{tabular}{l l cccc cccc}                   
\hline                                            
& & \multicolumn{4}{c}{20 \% MC} & \multicolumn{4}{c}{CMS Run 283408} \\                               
\cmidrule(lr){3-6}\cmidrule(lr){7-10}             
& & \multicolumn{2}{c}{$\epsilon_{\text{ov}}$} & \multicolumn{2}{c}{$\epsilon_{\text{in}}$}      
& \multicolumn{2}{c}{$\epsilon_{\text{ov}}$} & \multicolumn{2}{c}{$\epsilon_{\text{in}}$} \\     
\cmidrule(lr){3-4}\cmidrule(lr){5-6}\cmidrule(lr){7-8}\cmidrule(lr){9-10}                          
Trigger & Method                                  
& $\ttbarraw$ & $\haaFourB$                       
& $\ttbarraw$ & $\haaFourB$                       
& $\ttbarraw$ & $\haaFourB$                      
& $\ttbarraw$ & $\haaFourB$ \\                    
\hline                                            
$H_{T}$ & Oracle ($r^{+} = r_B^{*}+\tau$) & \textbf{99.560} & \textbf{36.844} & --- & --- &  \textbf{98.049} & \textbf{33.693} & --- & --- \\ 
& PID & 99.388 & 33.289 & 99.294 & 31.444 & 97.381 & 33.347 & 97.497 & \textbf{35.242} \\
\rowcolor{highlight}                             
& GFPO-F (Ours) & 99.425 & 34.273 & 99.425 & 34.273 & 97.587 & 33.347 & 97.508 & 33.262 \\             
\rowcolor{highlightDark}                          
& GFPO-FR (Ours) & \underline{99.468} & \underline{35.057} & \underline{99.468} & \underline{35.057} & \underline{97.734} & \underline{33.479} & \underline{97.821} & \underline{33.451} \\   \hline                                            
 AD & Oracle ($r^{+} = r_B^{*}+\tau$) & \textbf{96.092} & \textbf{29.910} & --- & --- & \textbf{76.664} & \textbf{41.028} & --- & --- \\ 
& PID & 95.233 & 27.298 & 94.990 & 26.287  & 75.053 & 39.191 & 76.210 & \textbf{44.573} \\
\rowcolor{highlight} & GFPO-F (Ours) & 95.442 & 27.926 & 95.442 & 27.926 & 75.562 & 40.083 & 76.209 & 42.533 \\   
\rowcolor{highlightDark}                     
& GFPO-FR (Ours) & \underline{95.842}  & \underline{28.875} & \underline{95.842} & \underline{28.868} & \underline{75.582} & \underline{40.309} &     
\underline{76.515} & \underline{44.477} \\              \hline                          \end{tabular}                   }   \label{tab:oracle_signal_eff_mc_realcms}          
\end{table*}

\begin{table*}[t]
\centering
\caption{Oracle signal efficiencies (20 \% MC) at $r_{B}^{*}=0.250\%$ and $r^{+} = r_B^{*}+\tau=0.275\%$. To avoid training and evaluation overlapping, we train on 80\% MC and report below results on 20\% MC. $\Delta$ columns show improvement over Oracle($r^{+}$). GFPO-FR has smaller signal efficiency gap with Oracle ($r^{+}$).}
\resizebox{\columnwidth}{!}{%
\setlength{\tabcolsep}{3pt}
\renewcommand{\arraystretch}{1}
\begin{tabular}{l l cccc cccc}
\hline
& & \multicolumn{2}{c}{$\epsilon_{\text{overall}}$} & \multicolumn{2}{c}{$\Delta_{\text{overall}}$ (Oracle ($r^{+}$) - Ours)} &
\multicolumn{2}{c}{$\epsilon_{\text{inband}}$} & \multicolumn{2}{c}{$\Delta_{\text{inband}}$ (Oracle ($r^{+}$) - Ours)} \\
\cmidrule(lr){3-4}\cmidrule(lr){5-6}\cmidrule(lr){7-8}\cmidrule(lr){9-10}
Trigger & Method
& $\ttbarraw$ & $\haaFourB$
& $\ttbarraw$ & $\haaFourB$
& $\ttbarraw$ & $\haaFourB$
& $\ttbarraw$ & $\haaFourB$ \\
\hline

$H_{T}$ & Oracle ($r^{+}$) & \textbf{99.560} & \textbf{36.844} &  &  & \textbf{99.560} & \textbf{36.844} &  &  \\
\rowcolor{highlight}
& GFPO-F & 99.425 & 34.273 & $+0.135$ & $+2.571$ & 99.425 & 34.273 & $+0.135$ & $+2.571$ \\
\rowcolor{highlightDark}
& GFPO-FR & \underline{99.468} & \underline{35.057} & $+0.093$ & $+1.787$ & \underline{99.468} &             
\underline{35.057} & $+0.092$ & $+1.787$ \\
\hline

AD & Oracle ($r^{+}$) & \textbf{96.092} & \textbf{29.910} & & & \textbf{96.092} & \textbf{29.910} & & \\
\rowcolor{highlight}
& GFPO-F & 95.442 & 27.926 & $+0.650$ & $+1.984$ & 95.442 & 27.926 & $+0.650$ & $+1.984$ \\
\rowcolor{highlightDark}
& GFPO-FR & \underline{95.842} & \underline{28.875} & $+0.250$ & $+1.035$ & \underline{95.842} & \underline{28.868} & $+0.251$ & $+1.042$ \\
\hline
\end{tabular}
}
\label{tab:oracle_signal_eff_mc}
\end{table*}

\begin{table*}[t]
\centering
\caption{Oracle signal efficiencies (CMS Run 283408) at $r_{B}^{*}=0.250\%$ and $r^{+} = r_B^{}+\tau=0.275\%$. $\Delta$ columns show improvement over Oracle ($r^{+}$). GFPO-FR has smaller signal efficiency gap with Oracle ($r^{+}$) except for \haaFourB\ inband. Figure~\ref{fig:cms_inband_drift_h4b} presents why inband \haaFourB\ signal efficiency is higher than oracle ($r^{+}$).}
\label{tab:oracle_signal_eff_realcms}
\resizebox{\columnwidth}{!}{%
\setlength{\tabcolsep}{3pt}
\renewcommand{\arraystretch}{0.85}
\begin{tabular}{l l cccc cccc}
\hline
& & \multicolumn{2}{c}{$\epsilon_{\text{overall}}$} & \multicolumn{2}{c}{$\Delta_{\text{overall}}$ (Oracle ($r^{+}$) - Ours)} &
\multicolumn{2}{c}{$\epsilon_{\text{inband}}$} & \multicolumn{2}{c}{$\Delta_{\text{inband}}$ (Oracle ($r^{+}$) - Ours)} \\
\cmidrule(lr){3-4}\cmidrule(lr){5-6}\cmidrule(lr){7-8}\cmidrule(lr){9-10}
Trigger & Method
& $\ttbarraw$ & $\haaFourB$
& $\ttbarraw$ & $\haaFourB$
& $\ttbarraw$ & $\haaFourB$
& $\ttbarraw$ & $\haaFourB$ \\
\hline
$H_{T}$ & Oracle ($r^{+}$) & \textbf{98.049} & \textbf{33.693} &  &  & \textbf{98.049} & \textbf{33.693} & & \\
\rowcolor{highlight}
& GFPO-F (Ours) & 97.587 & 33.347 & $+0.462$ & $+0.346$ & 97.508 & 33.262 & $+0.541$ & $+0.431$ \\
\rowcolor{highlightDark}
& GFPO-FR (Ours) & \underline{97.734} & \underline{33.479} & $+0.315$ & $+0.214$ & \underline{97.821} & \underline{33.451} & $+0.228$ & $+0.242$ \\
\hline
AD & Oracle ($r^{+}$) & \textbf{76.664} & \underline{41.028} &  & & \textbf{76.664} & 41.028 & & \\
\rowcolor{highlight}
& GFPO-F (Ours) & 75.562 & 40.083 & $+1.102$ & $+0.945$ & 76.209 & \underline{42.533} & $+0.455$ & $-1.505$ \\
\rowcolor{highlightDark}
& GFPO-FR (Ours) & \underline{75.582} & \underline{40.309} & $+1.082$ & $+0.719$ & \underline{76.515} & \textbf{44.477} & $+0.149$ & $-3.449$ \\
\hline
\end{tabular}
}
\end{table*}

\begin{figure}
    \centering
    \begin{subfigure}{0.45\textwidth}
    \includegraphics[width=\linewidth]{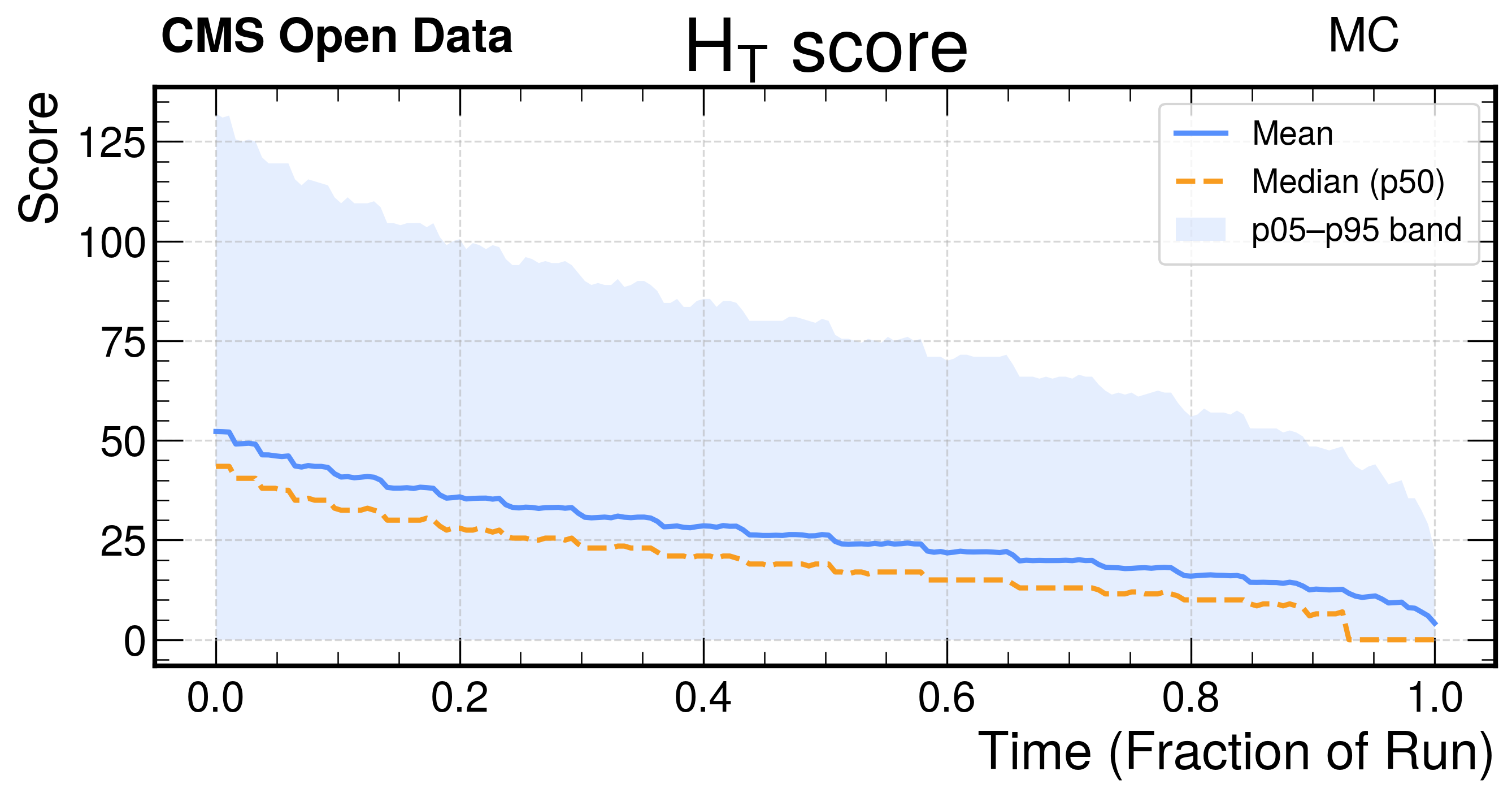}
    \caption{$H_{T}$ trigger}
    \end{subfigure}
    \qquad
    \begin{subfigure}{0.45\textwidth}
    \includegraphics[width=\linewidth]{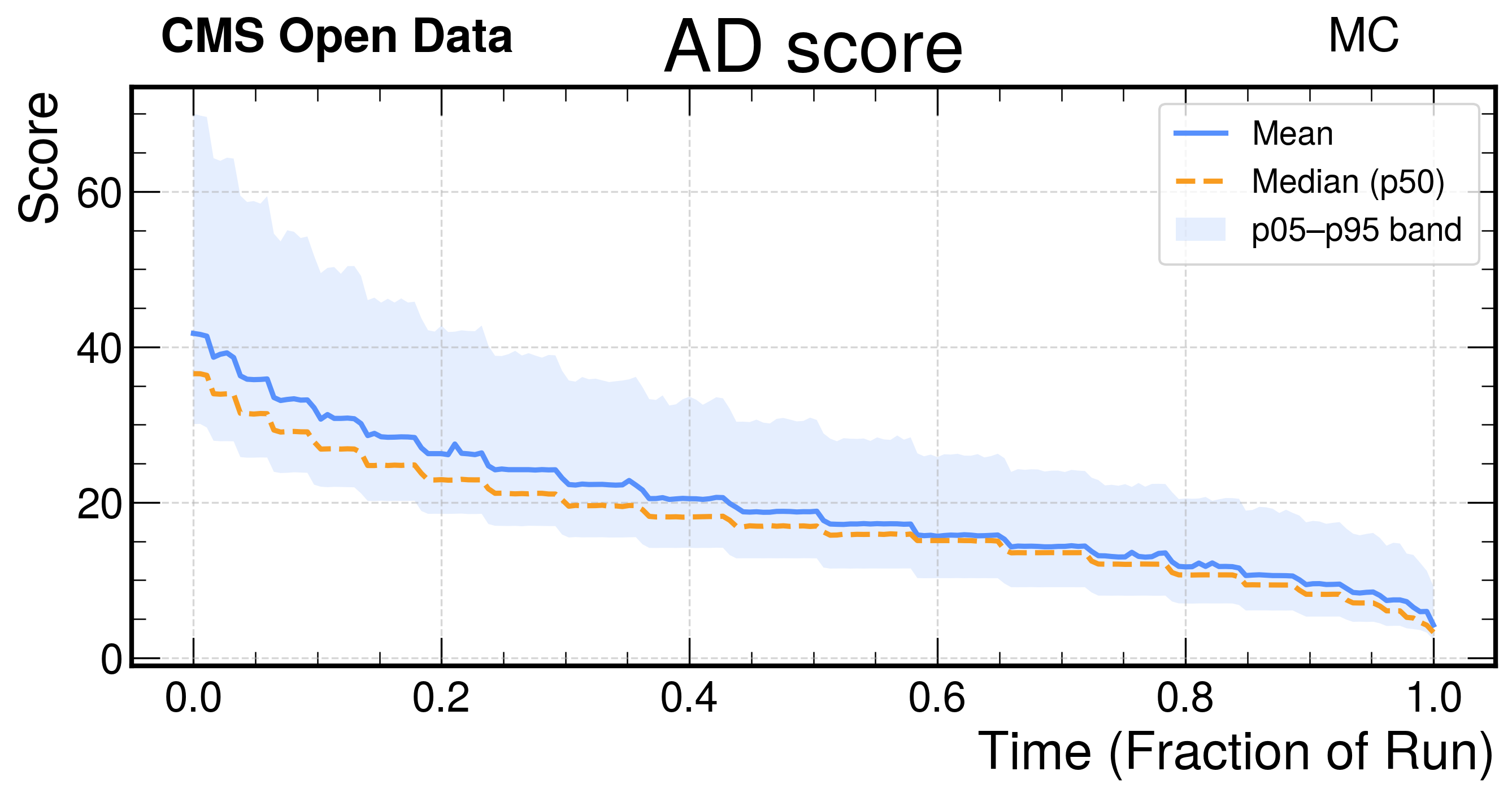}
    \caption{AD trigger}
    \end{subfigure}
    \caption{Background scores drift over time for both triggers for MC. Running mean (solid), median (dashed), and the central 5-95\% band (shaded) of the background $H_{T}$ score (left) and AD score (right) as a function of run time (fraction of run).}
    \label{fig:score_summary}
\end{figure}

\begin{figure}
    \centering
    \begin{subfigure}{0.45\textwidth}
    \includegraphics[width=\linewidth]{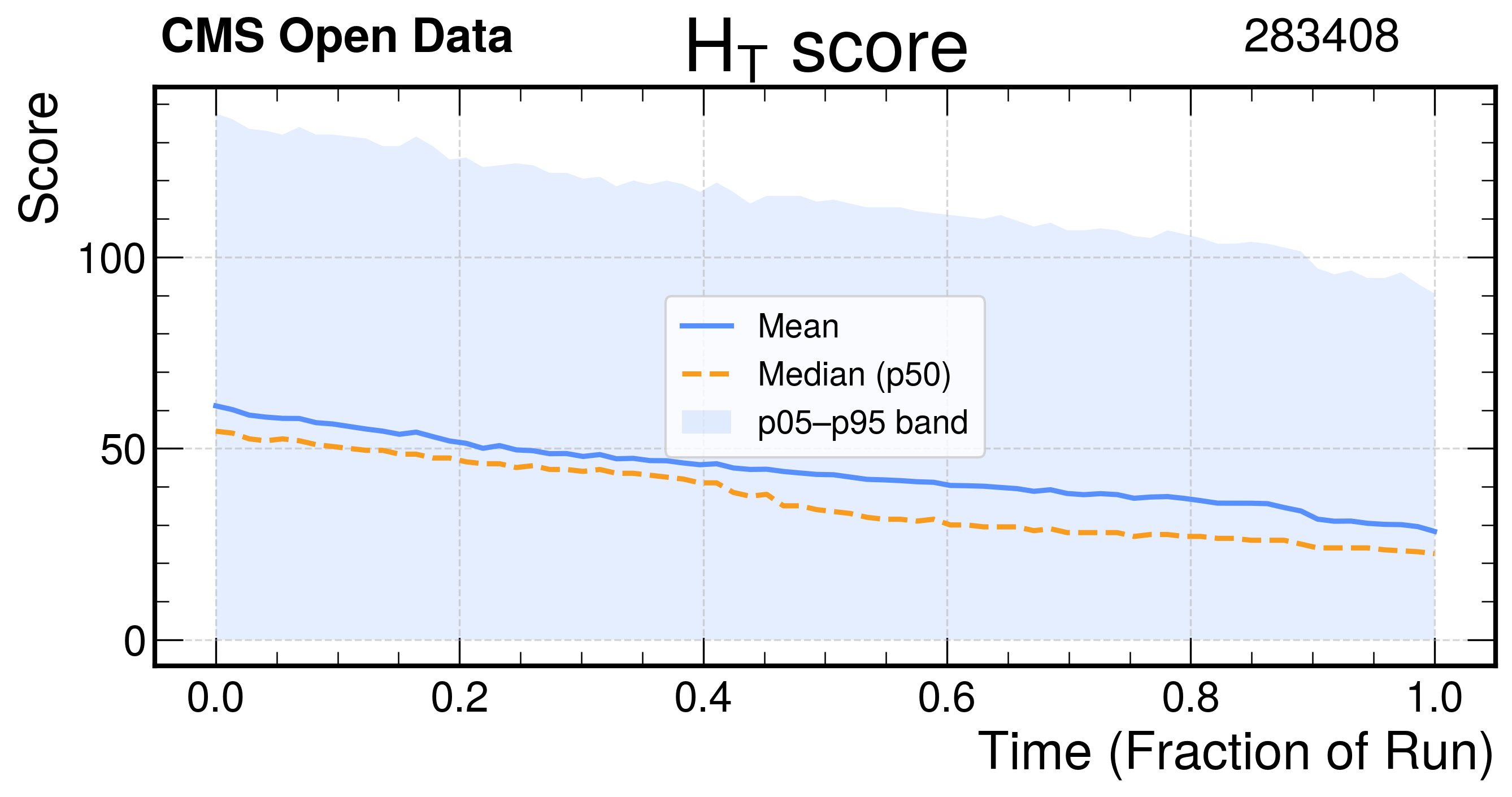}
    \caption{$H_{T}$ trigger}
    \end{subfigure}
    \qquad
    \begin{subfigure}{0.45\textwidth}
    \includegraphics[width=\linewidth]{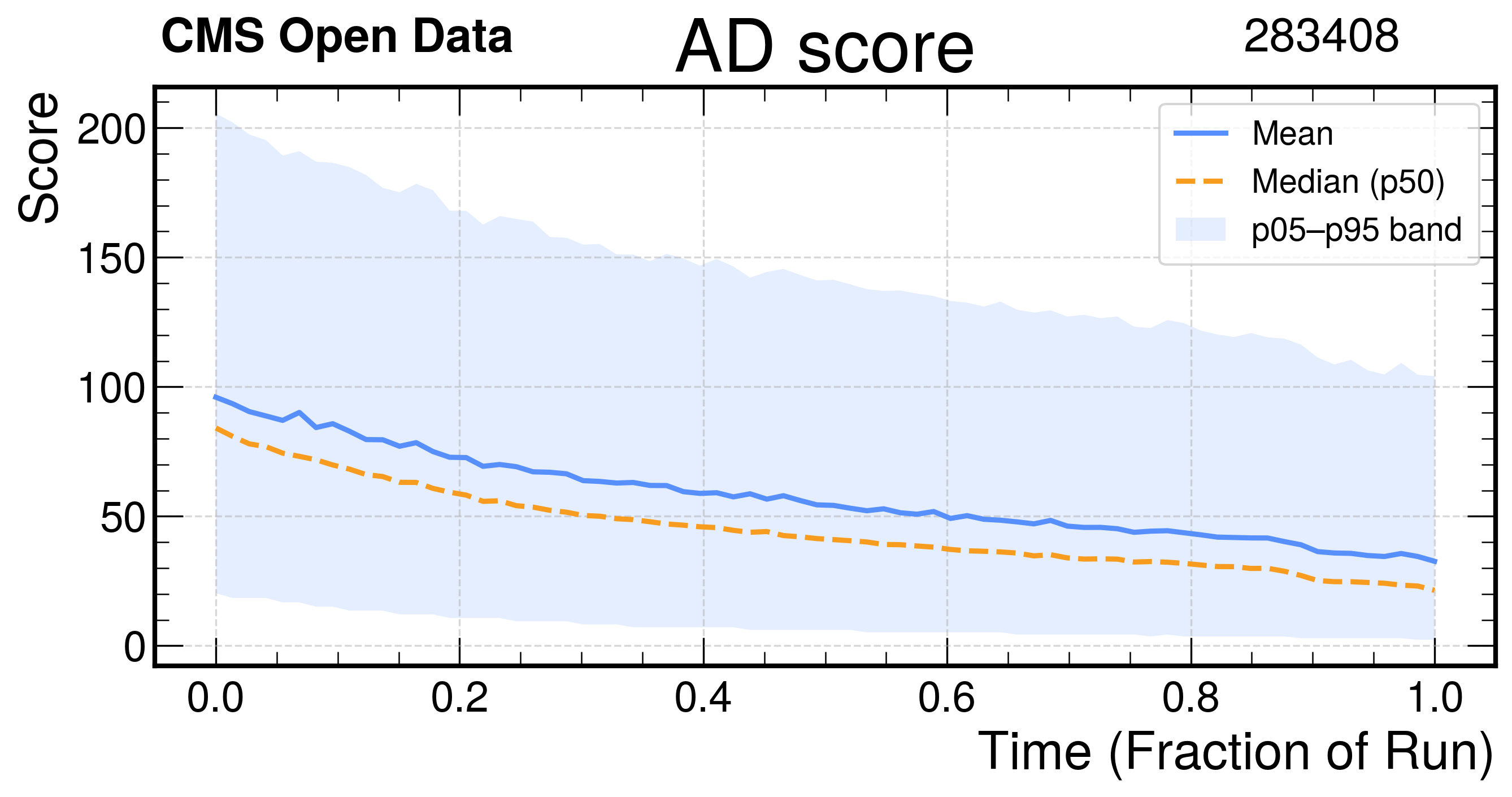}
    \caption{AD trigger}
    \end{subfigure}
    \caption{Background scores drift over time for both triggers for CMS Run 283408. Running mean (solid), median (dashed), and the central 5-95\% band (shaded) of the background scores as a function of run time (fraction of run).}
    \label{fig:score_summary_realdata}
\end{figure}

\begin{figure}
    \centering
    \includegraphics[width=\linewidth]{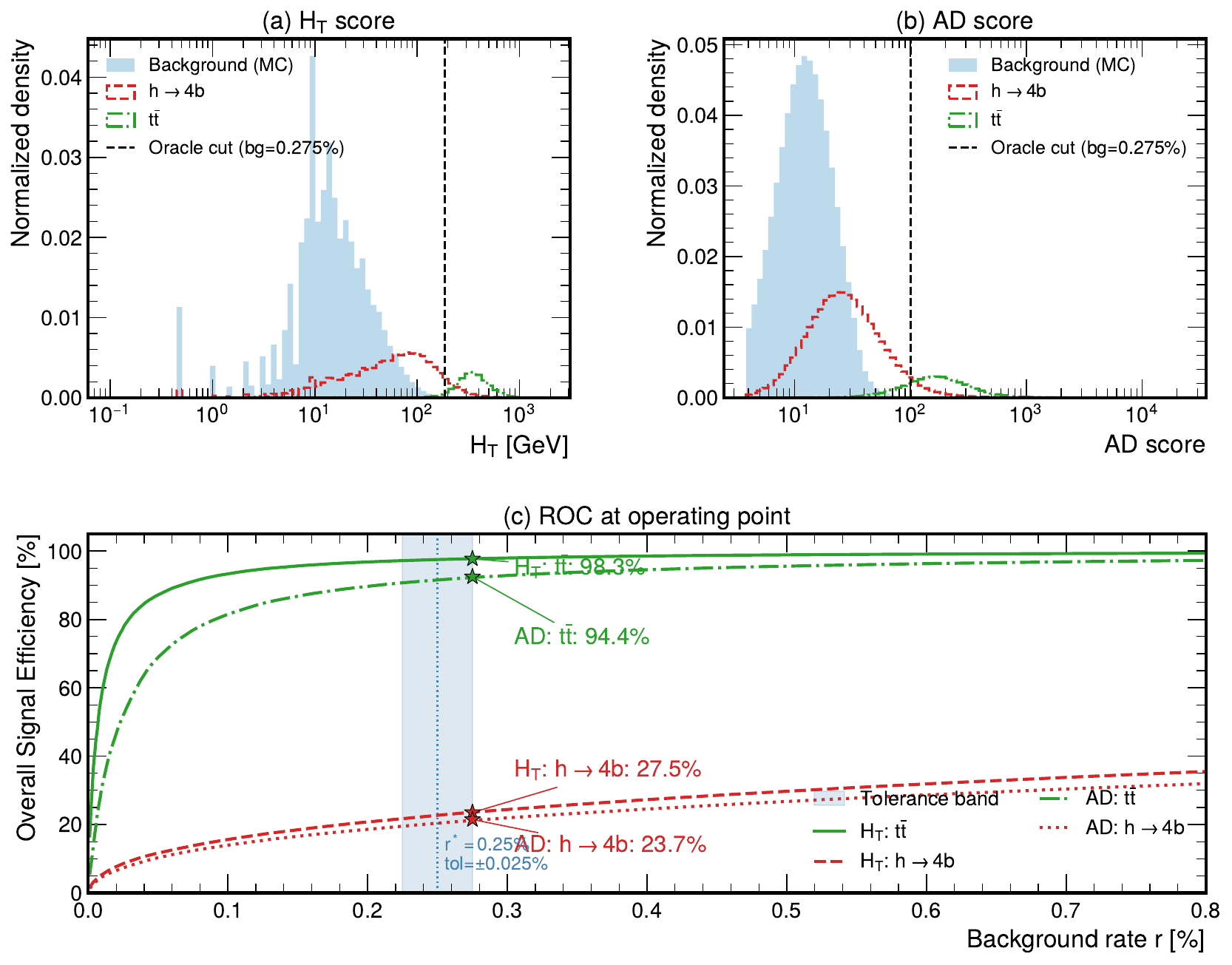}
    \caption{Signal-background discrimination for the $H_T$ and AD triggers (\textbf{Full MC}, no train/test split). (a) $H_T$ score and
   (b) AD score distributions for background (MinBias, shaded blue), \haaFourB\ signal (red 
  solid), and \ttbar\ signal (green dashed). The dashed black line indicates the oracle cut corresponding to the
   upper edge of the tolerance band ($r^{+} = 0.275\%$). (c) ROC curves showing signal efficiency versus    
  background rate $r$ for both triggers and both signals. The light blue shaded region denotes the tolerance band
  ($r^* \pm \tau = 0.25\% \pm 0.025\%$). Markers at the upper band edge show the maximum achievable signal efficiency: $H_T$
  trigger yields 98.3\% for $t\bar{t}$ and 27.5\% for $h \rightarrow 4b$, while the AD trigger yields 94.4\% for
  $t\bar{t}$ and 23.7\% for $h \rightarrow 4b$. The large gap between the two signals reflects the inherent overlap
  of the $h \rightarrow 4b$ distribution with the MinimumBias background in both feature spaces, setting a physics-limited
  upper bound on achievable signal efficiency at this operating point.}
    \label{fig:signal_background_overlap}
\end{figure}

\begin{figure}[!htbp]
    \centering
    \includegraphics[width=\linewidth]{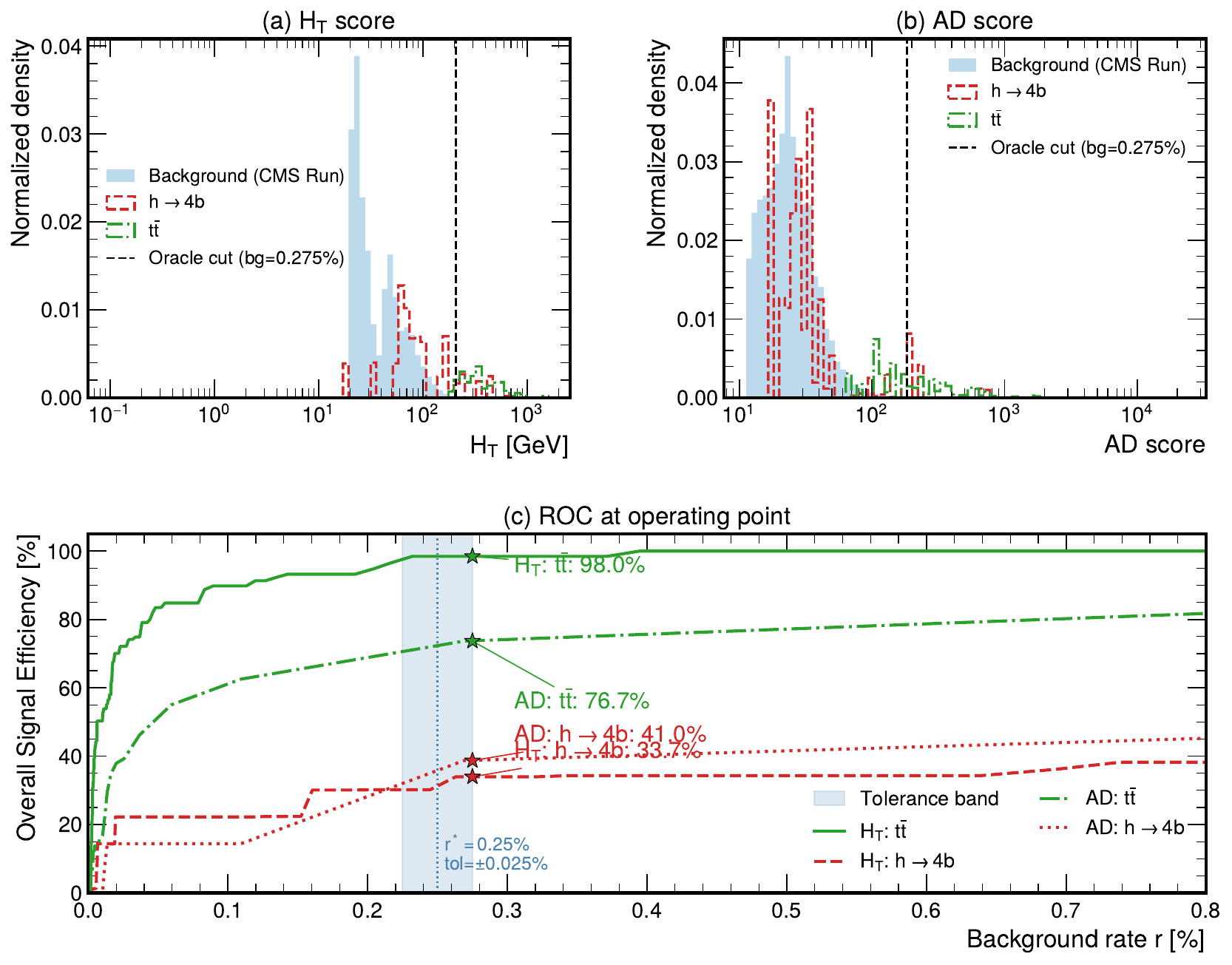}
    \caption{Signal-background discrimination for the $H_T$ and AD triggers (CMS Run 283408). (a) $H_T$ score and
   (b) AD score distributions for background (MinBias, shaded blue), \haaFourB\ signal (red 
  solid), and \ttbar\ signal (green dashed). The dashed black line indicates the oracle cut corresponding to the
   upper edge of the tolerance band ($r^{+} = 0.275\%$). (c) ROC curves showing signal efficiency versus    
  background rate $r$ for both triggers and both signals. The light blue shaded region denotes the tolerance band
  ($r^* \pm \tau = 0.25\% \pm 0.025\%$). Markers at the upper band edge show the maximum achievable signal efficiency: $H_T$
  trigger yields 98.0\% for $t\bar{t}$ and 33.7\% for \haaFourB, while the AD trigger yields 76.7\% for
  $t\bar{t}$ and 41.0\% for \haaFourB. }
    \label{fig:signal_background_overlap_cms}
\end{figure}

\begin{figure}
    \centering
    \includegraphics[width=\linewidth]{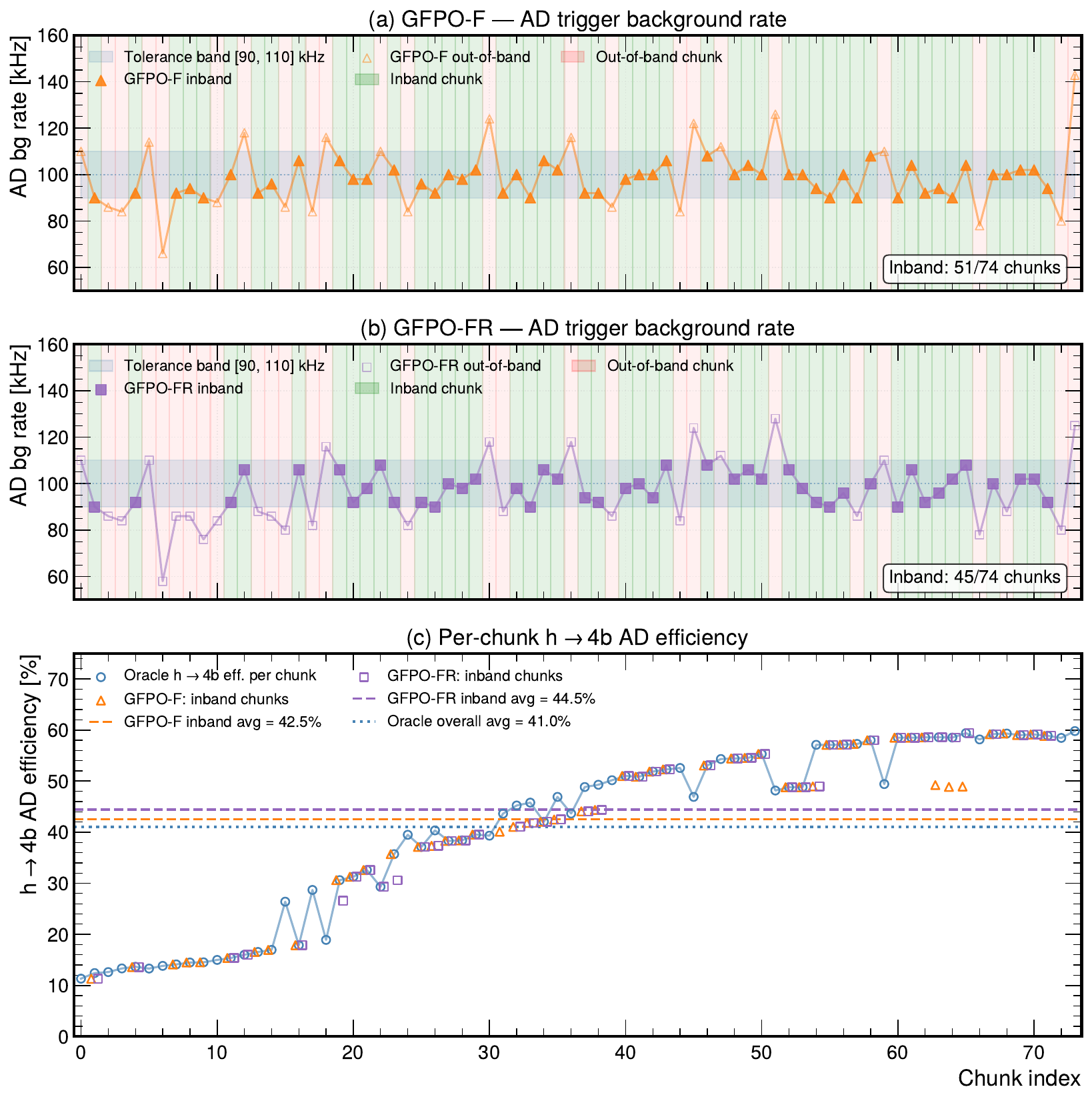}
    \caption{(a) GFPO-F for AD trigger (CMS Run 283408) (b) GFPO-FR for AD trigger (CMS Run 283408) (c) \haaFourB\ signal efficiency for CMS Run 283408 for GFPO-F (inband) and GFPO-FR (inband) and oracle ($r^{+}$) for \haaFourB\ (overall). The \haaFourB\ efficiency    
  exhibits strong temporal drift, rising from $\sim11\%$ in early high-pileup chunks to $\sim59\%$ in late low-pileup chunks. Because inband
   chunks are not \textit{uniformly} distributed but concentrate in the high-efficiency tail of the run as (a) and (b) shows inband chunks tend to concentrate at the tail of the run, the inband average for GFPO-F ($42.5\%$) and  
  GFPO-FR ($44.5\%$) both exceed the oracle overall average ($41.0\%$), which averages uniformly over all chunks excluding first 10 chunks.}
\label{fig:cms_inband_drift_h4b}
\end{figure}

\subsection{Training setup}    
\label{sec:training_setup_lhc}

\paragraph{Training pass.} Single chronological pass over the 148-chunk training partition (80\% of the 185-chunk MC stream after the 10-chunk calibration window). Each chunk: 50{,}000 events; micro-step: 5{,}000 events.

\paragraph{Optimizer.} All policies use Adam. Learning rates: GRPO and L-GRPO $2{\times}10^{-4}$; PPO $3{\times}10^{-4}$; DQN and DQN-F $1{\times}10^{-4}$ with mini-batch size 32.

\paragraph{Group sizes.} GRPO, L-GRPO: $G=16$. GFPO-F, GFPO-FR: $G=64$ candidates per micro-step, $K=16$ kept (cf. Figure~\ref{fig:gfpo_intermediate_ht}).

\paragraph{Reward.} $\lambda_1 = 0.25$, $\lambda_2 = 1.00$, $\alpha = 0.7$. Recall in Section~\ref{sec:rl_formulation_adaptive_thresholding}, we define target background rate $r_B^{*} = 0.25\%$, tolerance $\tau = \pm 0.025\%$ ($\approx \pm 10$~kHz).

\paragraph{Reporting.} Mean over 3 random seeds per method.

\paragraph{Compute.} Single Apple M4 CPU (10 cores, 32~GB unified memory; no GPU). Wall-clock $\approx 10$ minutes per method per seed for the full pass.

\subsection{Models and Hyperparameters}
\label{appendix:hyperparameter_ablation}

\paragraph{Hyperparameters in reward design.} We sweep $\lambda_1, \lambda_2 \in \{0.0, 0.25, 0.5, 0.75, 1.0\}$ in Equation~\ref{equation:reward_design} (25 configurations per trigger). Figure~\ref{fig:pareto_plot_grid} plots in-band rate against signal efficiency across all configurations and all four trigger$\times$signal pairs. GFPO-F and GFPO-FR collapse to a tight cluster in the upper-right (high in-band rate, high signal efficiency) regardless of $(\lambda_1, \lambda_2)$; baselines (DQN, PPO, GRPO, ADT) exhibit a pronounced trade-off, with hulls spanning a wide region of the objective space. The optimal one that has highest signal efficiency and inband rate is $\lambda_1=0.25$, $\lambda_2=1.00$.

\begin{figure*}[t]
    \centering
    \begin{subfigure}{0.485\textwidth}
        \includegraphics[width=\linewidth]{figures/Pareto_plot_lambda_backup/pareto_HT_ttbar.pdf}
        \caption{$H_{T}$ trigger (\ttbar)}
        \label{fig:pareto_ht_ttbar_appendix}
    \end{subfigure}
    \hfill
    \begin{subfigure}{0.485\textwidth}
        \includegraphics[width=\linewidth]{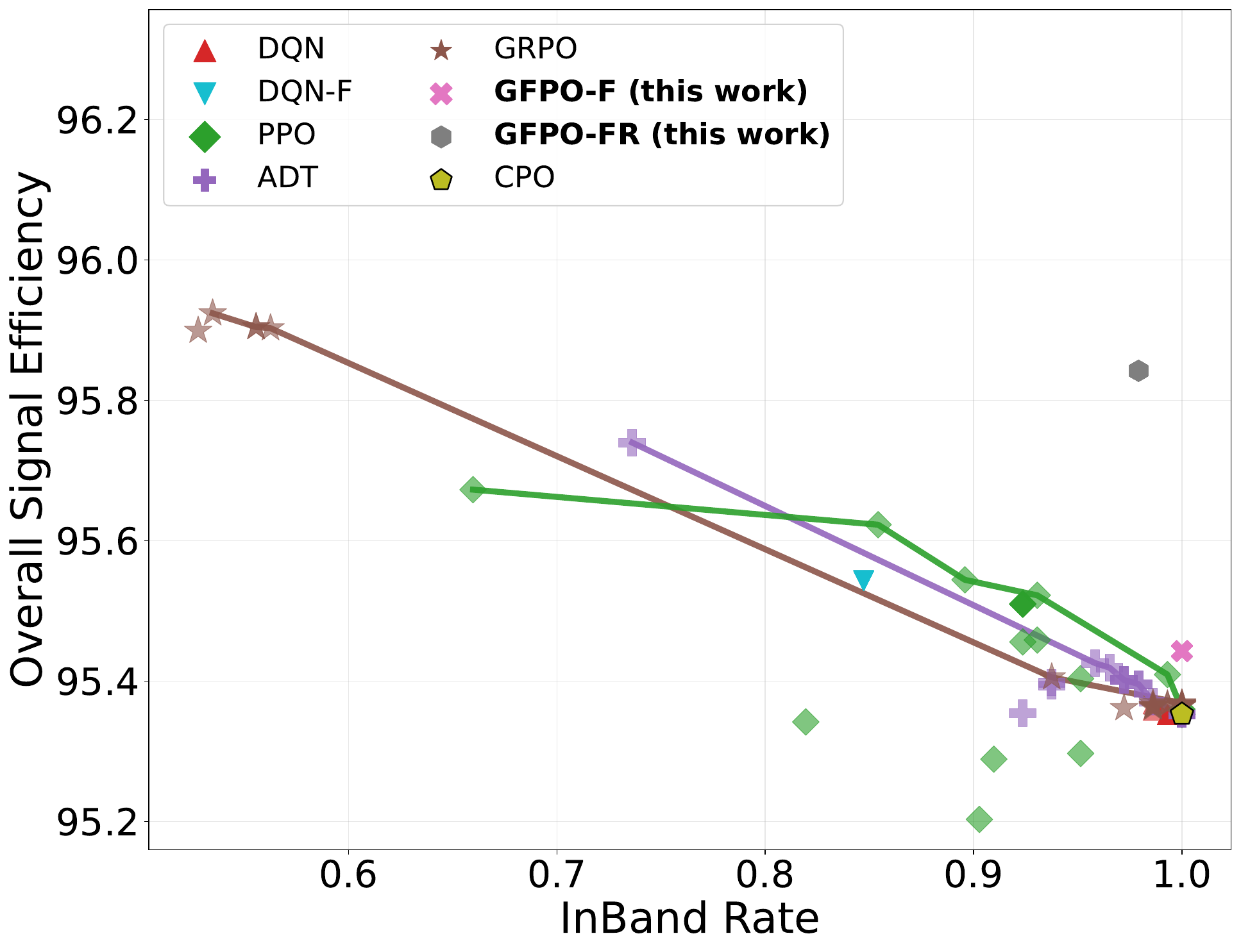}
        \caption{AD trigger (\ttbar)}
        \label{fig:pareto_ad_ttbar}
    \end{subfigure}

    \vspace{0.3cm} 

    \begin{subfigure}{0.485\textwidth}
        \includegraphics[width=\linewidth]{figures/Pareto_plot_lambda_backup/pareto_HT_h4b.pdf} 
        \caption{$H_{T}$ trigger (\haaFourB)}
        \label{fig:pareto_ht_haaForB}
    \end{subfigure}
    \hfill
    \begin{subfigure}{0.485\textwidth}
        \includegraphics[width=\linewidth]{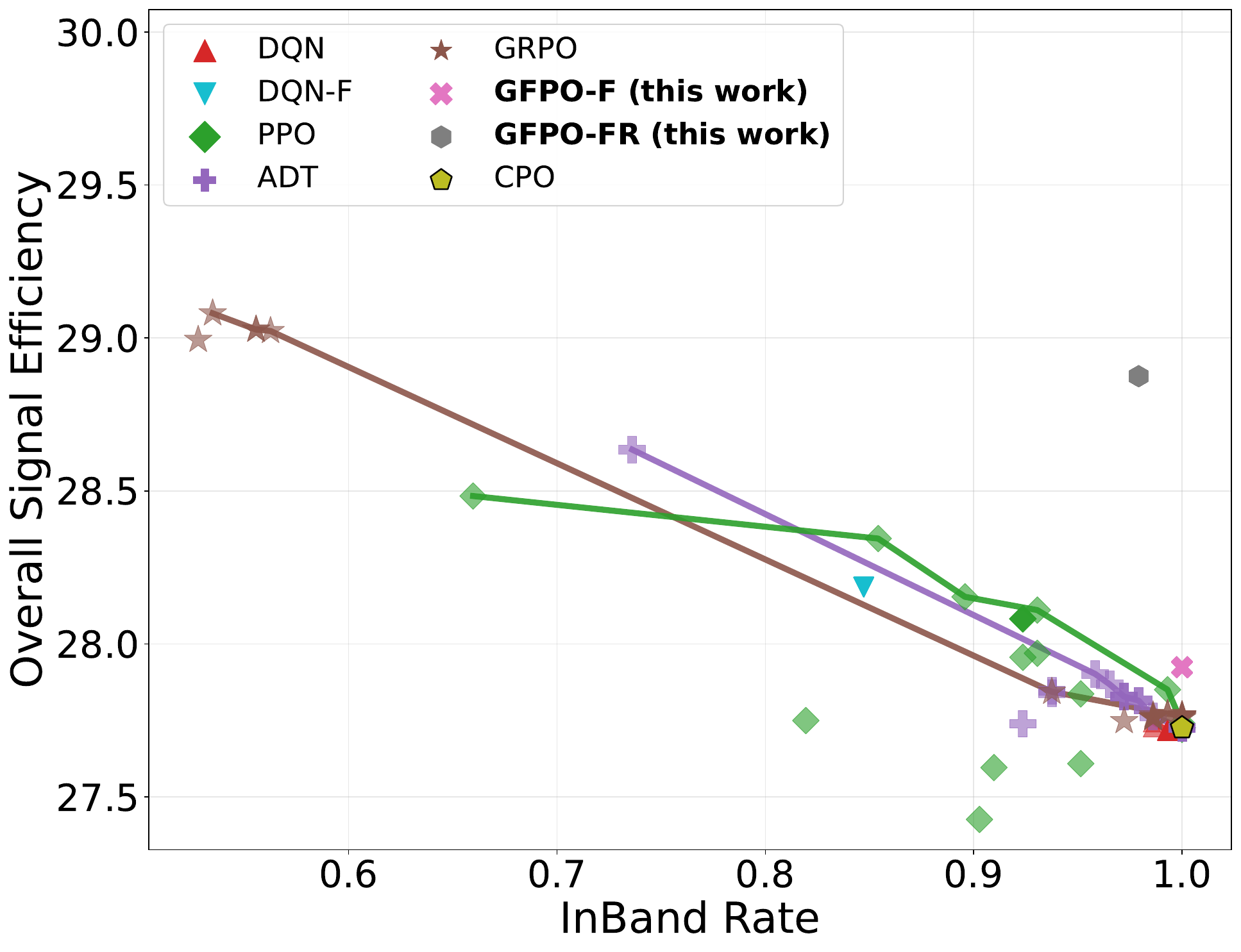} 
        \caption{AD trigger (\haaFourB)}
        \label{fig:pareto_ad_haaFourB}
    \end{subfigure}
    \caption{\textbf{Sensitivity Analysis of Reward Components (MC).} Each point represents a ($\lambda_1$, $\lambda_2$) configuration from
  Equation~\ref{equation:reward_design}, with concave hulls connecting the upper envelope per method. The $x$-axis measures the fraction of chunks
  whose background rate falls within the tolerance band, and the $y$-axis measures overall signal efficiency. Across all four trigger, signal
  combinations ($H_T$/$\mathrm{AD}$ $\times$ \ttbar/\haaFourB), baseline methods (DQN, PPO, GRPO, ADT) exhibit a pronounced trade-off between
  inband stability and signal efficiency, with their hulls spanning a wide region of the objective space. In contrast, our methods (GFPO-F and GFPO-FR) collapse to
   a tight cluster in the upper-right corner, simultaneously achieving near-perfect inband rates and the highest signal efficiency, demonstrating
  strong robustness to the choice of $\lambda_{1,2}$.}
    \label{fig:pareto_plot_grid}
\end{figure*}

\subsection{Standard anomaly detection metrics for LHC}
\label{sec:standard_anomaly_detection_metrics_lhc}

\paragraph{Our method achieves strong signal efficiency while maintaining the highest InBand rate.} In this subsection, we examine standard anomaly detection metrics commonly used in the literature, including precision, recall, and F1 score, within the context of the LHC trigger setting. We emphasize that these conventional anomaly detection metrics primarily evaluate signal-selection performance and do not account for operational constraints such as background-rate requirements, which are central to realistic LHC trigger deployment.
Tables~\ref{tab:trigger_summary}, \ref{tab:trigger_summary_realdata_test_time_training}, and \ref{tab:trigger_summary_realdata} summarize the anomaly detection results for the three experimental settings corresponding to Tables~\ref{tab:single_trigger_summary_compact}, \ref{tab:single_trigger_summary_compact_realdata_online}, and \ref{tab:single_trigger_summary_compact_realdata}, respectively, which report the LHC-specific operational metrics discussed in the main text.

\begin{landscape}
\begin{table}[p]
\centering
\caption{Anomaly detection benchmark for MC. All metrics are averaged over per chunk. Bold and \underline{underline} mark the best and second-best per column. Our method GFPO-FR places top-two for 9 out of 14 metrics on $H_T$ trigger and 7 out of 14 metrics on AD trigger, outperforming or performing on par with all other baseline methods. Our methods also achieve highest InBand rate among all methods (Table~\ref{tab:single_trigger_summary_compact}).}
\label{tab:trigger_summary}
\scriptsize
\setlength{\tabcolsep}{1pt}
\renewcommand{\arraystretch}{0.88}
\resizebox{\linewidth}{!}
{\begin{tabular}{lcccccccccrrrrrr}
\toprule
Trigger & Method & TPR(Combined)/Recall $\uparrow$ & TPR(\ttbar) $\uparrow$ & TPR(\haaFourB) $\uparrow$ & FPR$\downarrow$ & FPR (\ttbar) $\downarrow$ & FPR (\haaFourB) $\downarrow$ & TNR$\uparrow$ & FNR$\downarrow$ & Prec.(Combined) $\uparrow$ & Prec.(\ttbar) $\uparrow$ & Prec.(\haaFourB) $\uparrow$ & F1(Combined) $\uparrow$ & F1(\ttbar) $\uparrow$ & F1(\haaFourB) $\uparrow$ \\
\midrule
$H_{T}$ & Constant & 0.655 & 0.91 & 0.136 & \textbf{0.001} &\textbf{0.001} & \textbf{0.001} & \textbf{0.999} & 0.345 & \textbf{1.000} & \textbf{1.000} & \textbf{0.996} &0.791 & 0.952 & 0.240 \\
$H_{T}$  & PID & 0.740 & 0.98 & 0.252 & \underline{0.002} & \underline{0.002} & \underline{0.002}  & \underline{0.998} & 0.260 & \underline{0.999} & \underline{0.999} & \underline{0.992} & 0.850 & 0.989 & 0.402 \\
$H_{T}$ &DSPOT & \underline{0.748} & 0.971 & 0.245 & \underline{0.002} & \underline{0.002} & \underline{0.002} & \underline{0.998} & \textbf{0.252} & \underline{0.999} & \underline{0.999} & 0.991 & \textbf{0.855} & \textbf{0.992} & \textbf{0.427} \\
$H_{T}$  & ADT & 0.739 & 0.98 & 0.250 & \underline{0.002} & \underline{0.002} & \underline{0.002} & \underline{0.998} & 0.261 & \underline{0.999} & \underline{0.999} & \underline{0.992} & 0.850 & 0.989 & 0.399 \\
$H_{T}$  & DQN & 0.740 & 0.98 & 0.252 & \underline{0.002} & \underline{0.002} & \underline{0.002}  & \underline{0.998} & 0.260 & \underline{0.999} & \underline{0.999} & \underline{0.992} & 0.850 & 0.989 & 0.402 \\
$H_{T}$  & PPO & 0.740 & 0.98 & 0.252 & \underline{0.002} & \underline{0.002} & \underline{0.002} & \underline{0.998} & 0.260 & \underline{0.999} & \underline{0.999} & \underline{0.992} & 0.850 & 0.989 & 0.401 \\
$H_{T}$  & GRPO & 0.740 & 0.98 & 0.252 & \underline{0.002} & \underline{0.002} & \underline{0.002} & \underline{0.998} & 0.260 & \underline{0.999} & \underline{0.999} & \underline{0.992} & 0.850 & 0.989 & 0.402 \\
$H_{T}$ & L-GRPO & 0.741 & 0.981 & \underline{0.264} & \underline{0.002} & \underline{0.002} & \underline{0.002} & \underline{0.998} & 0.259 & \underline{0.999} & \underline{0.999} & 0.990 & 0.851 & 0.989 & \underline{0.413} \\
$H_{T}$ & CPO & 0.733 & 0.977 & 0.240  & 0.002 & 0.002 & 0.002 & 0.998 & 0.267 & 0.999 & 0.999 & 0.992 & 0.845 & 0.988 & 0.386 \\
$H_{T}$  & GFPO-F & 0.741 & 0.98 & 0.254 & 0.003 & 0.003 & 0.003 &0.997 & 0.259 & \underline{0.999} & \underline{0.999} & \underline{0.992} & 0.851 & \underline{0.990} & 0.404  \\
$H_{T}$  & GFPO-FR & 0.744 & \underline{0.982}  & 0.260 & 0.003 & 0.003 & 0.003 & 0.997 & \underline{0.256} & \underline{0.999} & \underline{0.999} & \underline{0.992} & \underline{0.853}
& \underline{0.990} & 0.412 \\
\midrule
AD & Constant & 0.576 & 0.795 & 0.131 &\textbf{0.001} & \textbf{0.001} & \textbf{0.001} &\textbf{0.999} & 0.424 & \textbf{1.000} & \textbf{1.000} & \textbf{0.995} & 0.731 & 0.886 & 0.231 \\
AD & PID & 0.700 & 0.936 & 0.219 & \underline{0.002} & \underline{0.002} & \underline{0.002} & \underline{0.998} & 0.300 & \underline{0.999} & \underline{0.999} & \underline{0.991} & 0.823 & 0.966 & 0.359 \\
AD & DSPOT & \underline{0.704} & 0.939 & \textbf{0.231} & \underline{0.002} & \underline{0.002} & \underline{0.002} & \underline{0.998} & \underline{0.296} & \underline{0.999} & \underline{0.999} & 0.990 & \underline{0.826} & \underline{0.968} & \textbf{0.374} \\
AD & ADT & 0.700 & 0.936 & 0.220 & 0.003 & 0.003 & 0.003 & 0.997 & 0.300 & \underline{0.999} & \underline{0.999} & \underline{0.991} & 0.824 & 0.967 & 0.360 \\
AD & DQN & 0.699 & 0.935 & 0.219 & \underline{0.002} & \underline{0.002} & \underline{0.002} & \underline{0.998} & 0.301 & \underline{0.999} & \underline{0.999} & \underline{0.991} & 0.823 & 0.966 & 0.358\\
AD & PPO & \underline{0.704} & \underline{0.940} & 0.225 &0.003 & 0.003 & 0.003 & 0.997 & \underline{0.296} & \underline{0.999} & \underline{0.999} & \underline{0.991} & \underline{0.826} & \underline{0.968} & 0.367 \\
AD & GRPO & 0.700 & 0.936 & 0.219 & \underline{0.002}
& \underline{0.002} & \underline{0.002} & \underline{0.998} & 0.300 & \underline{0.999} & \underline{0.999} & \underline{0.991} & 0.823 & 0.966 & 0.359 \\
AD & L-GRPO & 0.699 & 0.936 & 0.221 & \underline{0.002} & \underline{0.002} & \underline{0.002} & \underline{0.998} & 0.301 & \underline{0.999} & \underline{0.999} & 0.990 & 0.823 & 0.966 & 0.361 \\
AD & CPO & 0.694 & 0.932 & 0.211 & 0.002 & 0.002 & 0.002 & 0.998 & 0.306 & \underline{0.999} & \underline{0.999} & 0.991 & 0.819 & 0.964 & 0.348 \\
AD & GFPO-F & 0.701 & 0.937 & 0.221 & 0.003 & 0.003 & 0.003 & 0.997 & 0.299 & \underline{0.999} & \underline{0.999} & \underline{0.991} & 0.824 & 0.967 & 0.361 \\
AD & GFPO-FR & \textbf{0.706} & \textbf{0.941} & \underline{0.226} & 0.003 & 0.003 & 0.003 & 0.997 & \textbf{0.294} & \underline{0.999} & \underline{0.999} & \underline{0.991} & \textbf{0.827} & \textbf{0.969} & \underline{0.368} \\
\bottomrule
\end{tabular}
}
\end{table}
\clearpage

\begin{table}[p]
\centering
\caption{Anomaly detection benchmark for CMS Run 283408. All metrics are averaged over per chunk. Each RL policy is continually fine-tuned on streaming time chunks rather than frozen at deployment (test-time training). Bold and \underline{underline} mark the best and second-best per column. Our method GFPO-FR place top-two for 7 out of 14 metrics on $H_{T}$ trigger and 6 out of 14 metrics on AD triggers. While baselines like PID, DSPOT, and PPO achieve more top-two metrics for the $H_{T}$ trigger, these metrics purely reflect signal efficiency without background rate constraints. As detailed in Table~\ref{tab:single_trigger_summary_compact_realdata_online}, our methods GFPO-FR and GFPO-F maintain higher InBand rates (95\% and 98\%, respectively) compared to PPO (97.491 \%), DSPOT (17.6\%) and PID (43.2\%).}
\label{tab:trigger_summary_realdata_test_time_training}
\scriptsize
\setlength{\tabcolsep}{1pt}
\renewcommand{\arraystretch}{0.88}
\resizebox{\linewidth}{!}
{\begin{tabular}{lcccccccccrrrrrr}
\toprule
Trigger & Method & TPR(Combined)/Recall $\uparrow$ & TPR(\ttbar) $\uparrow$ & TPR(\haaFourB) $\uparrow$ & FPR$\downarrow$ & FPR (\ttbar) $\downarrow$ & FPR (\haaFourB) $\downarrow$ & TNR$\uparrow$ & FNR$\downarrow$ & Prec.(Combined) $\uparrow$ & Prec.(\ttbar) $\uparrow$ & Prec.(\haaFourB) $\uparrow$ & F1(Combined) $\uparrow$ & F1(\ttbar) $\uparrow$ & F1(\haaFourB) $\uparrow$ \\
\midrule
$H_{T}$  & Constant & 0.568 & 0.913 & 0.223 & \textbf{0.001} &\textbf{0.001} & \textbf{0.001} & \textbf{0.999} & 0.432 & \textbf{0.999} & \textbf{0.999} & \textbf{0.994} & 0.724 & 0.954 & 0.364 \\
$H_{T}$  & PID & 0.654 & 0.975 & 0.333 & \underline{0.002} & \underline{0.002} & \underline{0.002} & \underline{0.998} & 0.346 & \underline{0.998} & 0.997 & \underline{0.993} & 0.790 & 0.986 & 0.499 \\
$H_{T}$ & DSPOT & \textbf{0.667} & 0.940 & 0.282 & \underline{0.002} & \underline{0.002} & \underline{0.002} & \underline{0.998} & \textbf{0.333} & \underline{0.998} & \underline{0.998} & 0.992 & \textbf{0.800} & \textbf{0.989} & \textbf{0.510} \\
$H_{T}$  & ADT & 0.649 & \textbf{0.978} & 0.333 & 0.003 & 0.003 & 0.003 & 0.997 & 0.351 & \underline{0.998} & 0.997 & 0.992 & \underline{0.791} & \underline{0.988} & \underline{0.499} \\
$H_{T}$  & DQN & 0.653 & 0.974 & 0.332 & \underline{0.002} & \underline{0.002} & \underline{0.002}  & \underline{0.998} & 0.347 & \underline{0.998} & 0.997 & \underline{0.993} & 0.789 & 0.985 & 0.498 \\
$H_{T}$  & PPO & 0.653 & 0.977 & \underline{0.333} & \underline{0.002} & \underline{0.002} & \underline{0.002} & \underline{0.998} & 0.347 & \underline{0.998} & 0.997 & \underline{0.993} & \underline{0.791} & 0.987 & 0.498 \\
$H_{T}$  & GRPO & 0.653 & 0.975 & 0.331 & \underline{0.002} & \underline{0.002} & \underline{0.002} & \underline{0.998} & 0.347 & \underline{0.998} & \underline{0.998} & \underline{0.993} & 0.789 & 0.986 & 0.496 \\
$H_{T}$ & L-GRPO & 0.653 & \underline{0.976} & 0.332 & \underline{0.002}
& \underline{0.002} & \underline{0.002} & \underline{0.998} & 0.347 & \underline{0.998} & 0.997 & 0.992 & 0.789 & 0.986 & 0.487 \\
$H_{T}$ & CPO & 0.651 & 0.974 & 0.329 & \underline{0.002} & \underline{0.002} & \underline{0.002} & \underline{0.998} & 0.349 & \underline{0.998} & 0.997 & 0.992 & 0.787 & 0.985 & 0.487 \\
$H_{T}$  & GFPO-F & 0.654 & \underline{0.976} & 0.332 & \underline{0.002} & \underline{0.002} & \underline{0.002} & \underline{0.998} & 0.346 & 0.997 & 0.997 & \underline{0.993} & 0.790 & 0.986 & 0.498  \\
$H_{T}$  & GFPO-FR & \underline{0.656} & \textbf{0.978} & \textbf{0.334} & 0.003 & 0.003 & 0.003 & 0.997 & \underline{0.344} & \underline{0.998} & 0.997 & \underline{0.993} & \underline{0.791} & 0.987 & \underline{0.499} \\
\midrule
AD & Constant & 0.423 & 0.625 & 0.221 &\textbf{0.001} & \textbf{0.001} & \textbf{0.001} & \textbf{0.999} & 0.577 & \textbf{0.999} & \textbf{0.998} & \underline{0.995} & 0.594 & 0.769 & 0.362 \\
AD & PID & 0.573 & 0.753 & 0.392 & \underline{0.002} & \underline{0.002} & \underline{0.002} & \underline{0.998} & 0.427 & \underline{0.998} & \underline{0.997} & 0.994 & 0.728 & 0.858 & 0.563 \\
AD & DSPOT & \underline{0.688} & \textbf{0.810} & \textbf{0.567} & \underline{0.002} & \underline{0.002} & \underline{0.002} & \underline{0.998} & \underline{0.312} & \underline{0.998} & \underline{0.997} & \textbf{0.996} & \textbf{0.815} & \textbf{0.893} & \textbf{0.722} \\
AD & ADT & 0.574 & 0.754 & 0.394 & \underline{0.002} & \underline{0.002} & \underline{0.002} & \underline{0.998} & 0.426 & \underline{0.998} & \underline{0.997} & 0.994 & 0.729 & 0.858 & 0.565 \\
AD & DQN & 0.582 & \underline{0.759} & 0.404 & \underline{0.002} & \underline{0.002} & \underline{0.002} & \underline{0.998} & 0.418 & 0.997 & \underline{0.997} & 0.994 & 0.735 & \underline{0.862} & 0.575 \\
AD & PPO & 0.565 & 0.748 & 0.383 & \underline{0.002} & \underline{0.002} & \underline{0.002} & \underline{0.998} & 0.435 & \underline{0.998} & \underline{0.997} & 0.994 & 0.722 & 0.854 & 0.553 \\
AD & GRPO & 0.578 & 0.755 & 0.401 & \underline{0.002} & \underline{0.002} & \underline{0.002} & \underline{0.998} & 0.422 & \underline{0.998} & \underline{0.997} & 0.994 & 0.732 & 0.859 & 0.571 \\
AD & L-GRPO & \underline{0.596} & 0.757 & \underline{0.435} & \underline{0.002} & \underline{0.002} & \underline{0.002} & \underline{0.998} & \underline{0.404} & \underline{0.998} & \underline{0.997} & 0.993 & \underline{0.746} & 0.860 & \underline{0.591} \\
AD & CPO & 0.594 & 0.756 & 0.432 & \underline{0.002} & \underline{0.002} & \underline{0.002} & \underline{0.998} & 0.406 & \underline{0.998} & \underline{0.997} & 0.993 & 0.740 & 0.859 & 0.586 \\
AD & GFPO-F & 0.580 & 0.757 & 0.402 & \underline{0.002} & \underline{0.002} & \underline{0.002} & \underline{0.998} & 0.420 & \underline{0.998} & \underline{0.997} & 0.994 & 0.733 & 0.861 & 0.572 \\
AD & GFPO-FR & 0.578 & 0.755 & 0.401 & \underline{0.002} & \underline{0.002} & \underline{0.002} & \underline{0.998} & 0.422 & \underline{0.998} & \underline{0.997} & 0.994 & 0.732 & 0.859 & 0.571 \\
\bottomrule
\end{tabular}
}
\end{table}
\clearpage        

\begin{table}[p]
\centering
\caption{Anomaly detection benchmark on CMS Run 283408 (zero-shot transfer). All metrics are averaged over per chunk. MC-trained policies from Table~\ref{tab:single_trigger_summary_compact} are \emph{frozen} and \emph{deployed} on real collision data without fine-tuning. Bold and \underline{underline} mark the best and second-best per column. While baseline like DSPOT (9 out of 14 metrics) achieve more top-two metrics for the $H_{T}$ trigger than GFPO-F (5 out of 14 metrics), again, these metrics purely reflect signal efficiency without background rate constraints. As detailed in Table~\ref{tab:single_trigger_summary_compact_realdata}, our methods GFPO-FR and GFPO-F maintain \emph{substantially} higher InBand rates for $H_{T}$ trigger (95 \% and 99 \% , respectively) compared to PPO (83.8 \%), DSPOT (43.2\%) and PID (43.2\%).}
\label{tab:trigger_summary_realdata}
\scriptsize
\setlength{\tabcolsep}{1pt}
\renewcommand{\arraystretch}{0.88}
\resizebox{\linewidth}{!}
{\begin{tabular}{lcccccccccrrrrrr}
\toprule
Trigger & Method & TPR(Combined)/Recall $\uparrow$ & TPR(\ttbar) $\uparrow$ & TPR(\haaFourB) $\uparrow$ & FPR$\downarrow$ & FPR (\ttbar) $\downarrow$ & FPR (\haaFourB) $\downarrow$ & TNR$\uparrow$ & FNR$\downarrow$ & Prec.(Combined) $\uparrow$ & Prec.(\ttbar) $\uparrow$ & Prec.(\haaFourB) $\uparrow$ & F1(Combined) $\uparrow$ & F1(\ttbar) $\uparrow$ & F1(\haaFourB) $\uparrow$ \\
\midrule
$H_{T}$  & Constant & 0.568 & 0.913 & 0.223 & \textbf{0.001} &\textbf{0.001} & \textbf{0.001} & \textbf{0.999} & 0.432 & \textbf{0.999} & \textbf{0.999} & \textbf{0.994} & 0.724 & 0.954 & 0.364 \\
$H_{T}$  & PID & 0.654 & 0.975 & \underline{0.333} & \underline{0.002} & \underline{0.002} & \underline{0.002} & \underline{0.998} & 0.346 & \underline{0.998} & 0.997 & \underline{0.993} & 0.790 & 0.986 & \underline{0.499} \\
$H_{T}$  & DSPOT & \underline{0.667} & 0.940 & 0.282 & \underline{0.002} & \underline{0.002} & \underline{0.002} & \underline{0.998} & \textbf{0.333} & \underline{0.998} & \underline{0.998} & 0.992 & \textbf{0.800} & \textbf{0.989} & \textbf{0.510} \\
$H_{T}$  & ADT & 0.649 & 0.971 & 0.328 & \underline{0.002} & \underline{0.002} & \underline{0.002} & \underline{0.998} & 0.351 & \underline{0.998} & \underline{0.998} & \underline{0.993} & 0.787 & 0.984 & 0.493 \\
$H_{T}$ & DQN & 0.653 & 0.975 & 0.331 & \underline{0.002} & \underline{0.002} & \underline{0.002}  & \underline{0.998} & 0.347 & \underline{0.998} & 0.997 & \underline{0.993} & 0.790 & 0.986 & 0.496 \\
$H_{T}$ & PPO & 0.653 & 0.975 & 0.331 & \underline{0.002} & \underline{0.002} & \underline{0.002} & \underline{0.998} & 0.347 & \underline{0.998} & 0.997 & \underline{0.993} & 0.790 & 0.986 & 0.496 \\
$H_{T}$ & GRPO & 0.653 & 0.975 & 0.331 & \underline{0.002} & \underline{0.002} & \underline{0.002} & \underline{0.998} & 0.347 & \underline{0.998} & 0.997 & \underline{0.993} & 0.789 & 0.986 & 0.498 \\
$H_{T}$ & L-GRPO & 0.654 & \underline{0.977} & 0.332 & \underline{0.002} & \underline{0.002} & \underline{0.002} & \underline{0.998} & 0.346 & \underline{0.998} & 0.997 & 0.992 & 0.790 & \underline{0.987} & 0.490 \\
$H_{T}$ & CPO & 0.651 & 0.974 & 0.329 & 0.002 & 0.002 & 0.002 & 0.9975 & 0.349 & 0.998 & 0.997 & 0.992 & 0.787 & 0.985 & 0.487 \\
$H_{T}$  & GFPO-F & 0.654 & \underline{0.977} & 0.332 & \underline{0.002} & \underline{0.002} & \underline{0.002} & \underline{0.998} & 0.346 & \underline{0.998} & 0.997 & \underline{0.993} & \underline{0.791} & \underline{0.987} & \underline{0.499}  \\
$H_{T}$  & GFPO-FR & 0.656 & \textbf{0.978} & \textbf{0.334} & 0.003 & 0.003 & 0.003 & 0.997 & \underline{0.344} & \underline{0.998} & 0.997 & \underline{0.993} &
\underline{0.791} & \underline{0.987} & \underline{0.499} \\
\midrule
AD & Constant & 0.423 & 0.625 & 0.221 &\textbf{0.001} & \textbf{0.001} & \textbf{0.001} & \textbf{0.999} & 0.577 & \textbf{0.999} & \textbf{0.998} & \underline{0.995} & 0.594 & 0.769 & 0.362 \\
AD & PID & 0.573 & 0.753 & 0.392 & \underline{0.002} & \underline{0.002} & \underline{0.002} & \underline{0.998} & 0.427 & \underline{0.998} & \underline{0.997} & 0.994 & 0.728 & 0.858 & 0.563 \\
AD & DSPOT & \underline{0.688} & \underline{0.810} & \textbf{0.567} & \underline{0.002} & \underline{0.002} & \underline{0.002} & \underline{0.998} & \underline{0.312} & \underline{0.998} & \underline{0.997} & \textbf{0.996} & \textbf{0.815} & \textbf{0.893} & \textbf{0.722} \\
AD & ADT & 0.579 & 0.758 & 0.399 & 0.003 & 0.003 & 0.003 & 0.997 & 0.421 & \underline{0.998} & \underline{0.997} & 0.994 & 0.733 & 0.861 & 0.570 \\
AD & DQN & 0.570 & 0.751 & 0.390 & \underline{0.002} & \underline{0.002} & \underline{0.002} & \underline{0.998} & 0.430 & \underline{0.998} & \underline{0.997} &0.994 & 0.726 & 0.857 & 0.560 \\
AD & PPO & 0.584 & 0.762 & \underline{0.406} & 0.003 & 0.003 & 0.003 & 0.997 & 0.416 & \underline{0.998} & \underline{0.997} & 0.994 & 0.737 & 0.864 & 0.577 \\
AD & GRPO & 0.574 & 0.756 & 0.391 & \underline{0.002} & \underline{0.002} & \underline{0.002} & \underline{0.998} & 0.426 & \underline{0.998} & \underline{0.997} & 0.994 & 0.728 & 0.860 & 0.561 \\
AD & L-GRPO & 0.595 & 0.757 & 0.433 & \underline{0.002} & \underline{0.002} & \underline{0.002} & \underline{0.998} & 0.405 & \underline{0.998} & \underline{0.997} & 0.993 & \underline{0.745} & 0.859 & \underline{0.589} \\
AD & CPO & 0.594 & 0.756 & 0.432 & \underline{0.002} & \underline{0.002} & \underline{0.002} & \underline{0.998} & 0.406 & 0.998 & 0.997 & 0.993 & 0.740 & 0.859 & 0.586 \\
AD & GFPO-F & 0.580 & 0.758 & 0.402 & \underline{0.002} & \underline{0.002} & \underline{0.002} & \underline{0.998} & 0.420 & \underline{0.998} & \underline{0.997} & 0.994 & 0.734 & 0.861 & 0.573 \\
AD & GFPO-FR & 0.581 & 0.759 & 0.403 & \underline{0.002} & \underline{0.002} & \underline{0.002} & \underline{0.998} & 0.419 & \underline{0.998} & \underline{0.997} & 0.994 & 0.734 & \underline{0.862} & 0.573 \\
\bottomrule
\end{tabular}
}
\end{table}
\end{landscape}

\newpage

%% file: Appendix/Implementation.tex
\newpage
\section{Additional Implementation Details}
\subsection{Baseline Adaptation}
\label{appendix:baseline_description}

\paragraph{Stream chunking.} State features are computed over 5K-event windows; threshold updates $\Delta$ are applied per 50K-event chunk, following the PID setting of~\citet{emami2026selfdrivingtriggerlhcadaptive}. This yields 195 chunks for MC and 99 chunks for CMS Run 283408 (Table~\ref{tab:dataset_summary}).

\subsubsection{ADT} 

We adapt the ADT controller of~\citet{yang2024adt} to streaming trigger control; four changes are required.

\paragraph{Objective.} Vanilla ADT classifies fixed time-series windows as normal/abnormal (10 or 12 classes, dataset-dependent). We instead regulate the per-event accept rate to a target $r_B^{*}$ within tolerance $\tau$, with decisions at event granularity rather than window granularity.


\paragraph{State.} ADT's state vector includes the running $\{\rho^{TP}_t, \rho^{TN}_t, \rho^{FP}_t, \rho^{FN}_t\}$ over the last $k$ windows, quantities that require ground-truth labels at decision time. In LHC trigger control no such labels exist: signal/background identity is recovered only after the full event stream is processed. We replace ADT's classification-feedback state with the rate-regulation state of \sctt{BuildSequentialState} (Algorithm~\ref{alg:seq_state_and_controller}).

\paragraph{Episode/step mapping.} We identify each ``episode'' with one chunk and each RL step with one micro-step (a stride over events within the chunk). This preserves ADT's end-of-episode update rule while operating on a long event stream rather than fixed time-series windows.

\paragraph{Action space.} ADT uses binary $\delta \in \{0, 1\}$. Triggers are continuous-valued ($H_T$ in GeV, $AD_{\mathrm{cut}}$ in score units), so we discretize $\delta$ to match DQN and GRPO. We retain ADT's update cadence ($l=10$ micro-steps;~\citet{yang2024adt} default), reusing $a_{t-1}$ on intermediate steps. Other baselines (DQN, GRPO, GFPO-F, GFPO-FR) update every 5K-event window.

\paragraph{Safety shield.} For all RL methods we add a safety shield that overrides the policy when its action would require \emph{too much} threshold update. Vanilla ADT does not include such a shield; we add it because rate-budget violations trigger DAQ deadtime at the LHC~\citep{racz2000trigger, bertelsen2016operation}.



\subsubsection{DSPOT}
\label{subsection:DSPOT}

\paragraph{Algorithm.} DSPOT~\citep{siffer2017anomaly} maintains a sliding reference window of the most recent $n$ observations and models the excess distribution above an intermediate threshold $t$ as a Generalized Pareto Distribution (GPD). It updates the GPD parameters $(\hat{r}, \hat{\sigma})$ via maximum likelihood at each step and returns
\begin{equation}
    z_q = t + \frac{\hat{\sigma}}{\hat{\gamma}} \left[ \left(\frac{q \cdot n}{N_t}\right)^{-\hat{\gamma}} - 1 \right],
\end{equation}
where $N_t$ is the number of observations above $t$ in the reference window and $q$ is the target false-positive risk.

\paragraph{Adaptation.} At each chunk boundary we pass background scores from the most recent window to DSPOT and use the returned $z_q$ as the next-chunk trigger threshold $\threshold_t$, with $q = r_B^* = 0.25\%$. Reference window size matches our chunk size (50K events MC, 20K CMS Run 283408); the intermediate threshold $t$ is initialized at the 98th percentile of the calibration window~\citep{siffer2017anomaly}. No reward, policy network, or gradient update is involved.

\paragraph{No safety shield.} DSPOT is a one-sided rate-compliance algorithm: it controls the upper tail of the score distribution to bound the false-positive rate, but has no mechanism to keep the achieved rate within a two-sided tolerance band $[r^-, r^+]$, and no objective ranking among rate-compliant thresholds. We deliberately do not impose a safety shield on its output. Out-of-band chunks therefore appear directly in the InBand fraction (Table~\ref{tab:single_trigger_summary_compact}: 39.8\% $H_T$, 45.2\% AD). This is the honest evaluation: DSPOT was designed for unsupervised anomaly detection in unlabelled streams~\citep{siffer2017anomaly}, where signal efficiency is undefined. Our RL methods, by contrast, enforce the operational rate constraints arising from detector readout bandwidth~\citep{aaboud2017performance} and computing resources~\citep{atlas2020operation} via a safety shield, and then \emph{rank} feasible thresholds by signal efficiency.

\paragraph{$H_T$ trigger.} Pileup increases $H_T$ systematically across the full score distribution, since soft radiation from pileup vertices contributes to all jet $p_T$ measurements. DSPOT's tail estimator is therefore less well-matched to $H_T$ than to AD, where pileup manifests predominantly as a shift in the reconstruction-loss tail. We include DSPOT on both triggers for completeness; the asymmetry is itself informative about the qualitative difference in non-stationarity between the two trigger paths.

\paragraph{MC-only setting.} In Table~\ref{tab:single_trigger_summary_compact}, calibration and deployment share the same MC stream. DSPOT achieves InBand fractions of 45.2\% (AD) and 39.8\% ($H_T$), with InBand signal efficiencies of 93.7\% and 98.1\% on \ttbar\, close to the PID ceiling on both triggers. The InBand-fraction gap to PID (66.1\% AD, 59.1\% $H_T$) reflects DSPOT's 50-chunk GPD calibration burn-in, during which the threshold is held fixed; PID adapts from chunk one.

\paragraph{Sim-to-real transfer.} For MC-trained, CMS-deployed runs (Table~\ref{tab:single_trigger_summary_compact_realdata}), the GPD is fitted on MC and applied on CMS. The InBand fraction is largely preserved (41.9\% AD, 43.2\% $H_T$): DSPOT's 5-chunk sliding window re-fits the tail online and absorbs much of the sim-to-real shift. A single catastrophic chunk, however, pushes both MAE values past their P95 (AD: 0.106 vs.\ 0.087; $H_T$: 0.163 vs.\ 0.072), abrupt CMS rate excursions fall outside the support of the MC-calibrated GPD. PID's integral feedback degrades gracefully under the same shift; DSPOT's tail-fit estimator does not.

\paragraph{Test-time training.} For test-time training (Table~\ref{tab:single_trigger_summary_compact_realdata_online}) DSPOT follows the two-phase protocol of~\citet{siffer2017anomaly}: a calibration window accumulates background scores to fit the GPD; the sliding window then re-fits online each chunk (Figure~\ref{fig:exp3_dspot_evidence}). We use 50 CMS chunks at 20K events each~\citep{emami2026selfdrivingtriggerlhcadaptive}, which exhausts $\approx\!68\%$ of the available CMS data and leaves 24 chunks for deployment. During calibration DSPOT holds the MC-trained cut fixed, so it observes background rates of $0.113\%$ (AD) and $0.150\%$ ($H_T$), well below the $0.25\%$ target (Figure~\ref{fig:exp3_dspot_evidence}). The GPD is therefore fitted to the tail of this low-rate regime. When DSPOT activates at chunk 50, the CMS score distribution has shifted upward: deployment means are $0.233\%$ (AD) and $0.227\%$ ($H_T$). The abrupt threshold drop at chunk 50 corrects for this mismatch, but 24 chunks are too few to recover. Over all chunks DSPOT reaches 13.5\% and 17.6\% InBand vs.\ PID's 40.5\% and 43.2\% (Table~\ref{tab:single_trigger_summary_compact_realdata}), without calibration burn-in.

\paragraph{Why DSPOT lacks signal efficiency.} DSPOT fits a GPD to the background tail to enforce a target false-positive rate, with no signal knowledge at any stage. Among all rate-compliant thresholds, it has no basis for preferring the one that maximizes \ttbar\ or \haaFourB\ retention --- it returns whatever the GPD estimate dictates. This is structural, not a tuning choice: the algorithm has \emph{no} objective beyond rate control. Our two variants (GFPO-F and GFPO-FR) fill this gap by explicitly ranking feasible thresholds by signal efficiency at each step, recovering retention DSPOT forgoes while satisfying the same rate constraint.

\noindent \textbf{Note on $H_{T}$ trigger.} For the $H_{T}$ trigger, pileup increases $H_{T}$ systematically across the full score distribution rather than primarily in the tail, since additional soft radiation from pileup vertices contributes to all jet $p_{T}$ measurements. DSPOT's EVT-based tail estimator is therefore less well-matched to the $H_{T}$ setting than to the AD setting, where pileup drift manifests predominantly as a shift in the reconstruction loss distirbution. We include DSPOT on both triggers for completeness; the performance asymmetry between the two setttings is itself informative about the qualitatively different nature of non-stationarity in each trigger path.

\begin{figure*}
    \centering
\includegraphics[width=\linewidth]{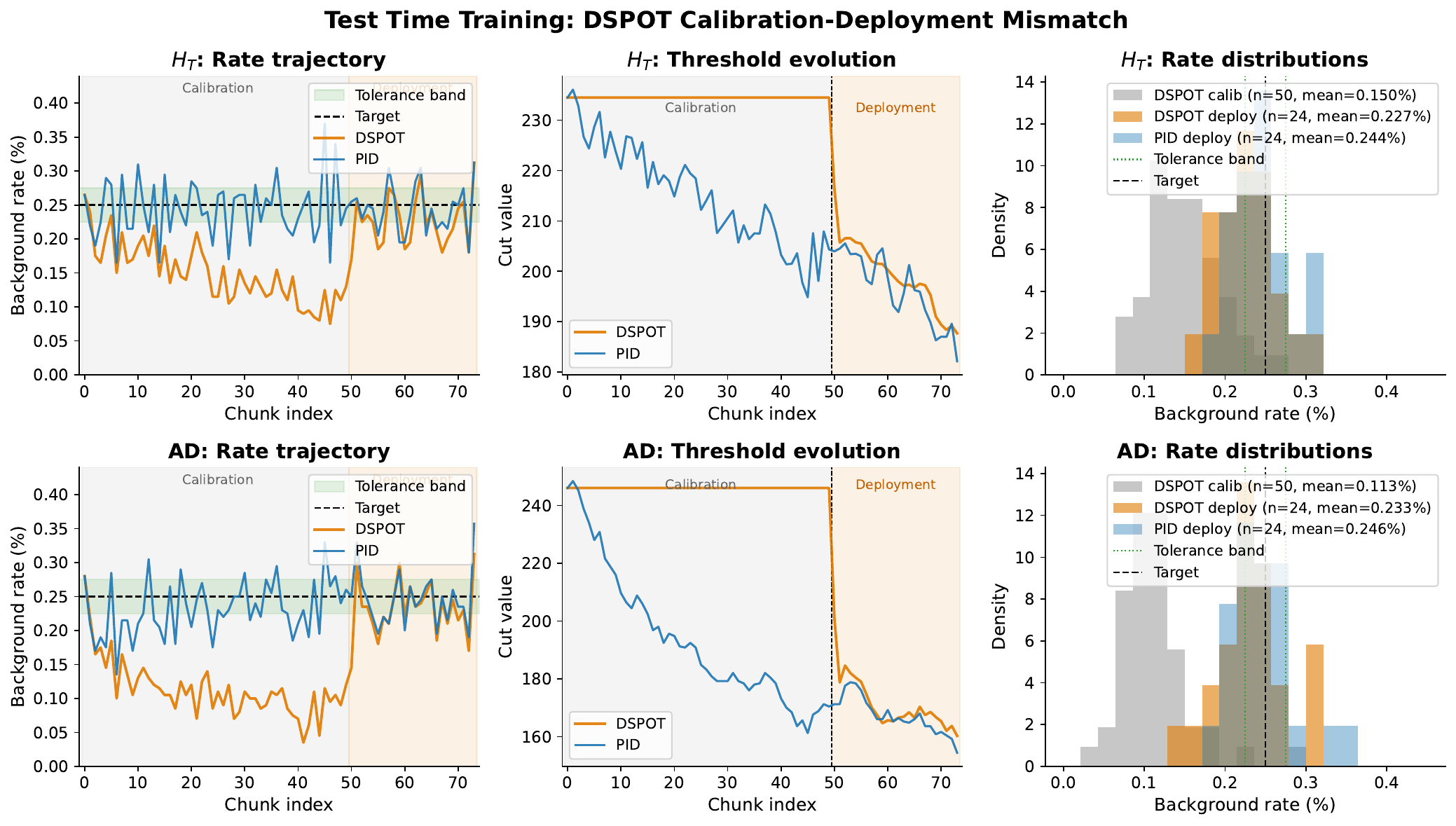}
    \caption{                                                    
    DSPOT calibration-deployment mismatch: the GPD is fitted on 50 calibration chunks where the background rate is $0.113\%$ (AD) and $0.150\%$ ($H_{T}$) under the fixed initial cut, well below the $0.25\%$ target, causing a distributional shift of $+0.12\%$ and $+0.08\%$ at deployment that limits the inband fraction to $13.5\%$ and $17.6\%$ over the reamining 24 chunks, while PID tracks the target continuously without burn-in phase.}            
    \label{fig:exp3_dspot_evidence}                   
\end{figure*}

\subsubsection{L-GRPO}
\label{appendix:L_GRPO}
We extend GRPO by replacing the fixed rate tracking penalty weight $\lambda_{1}$ with a dual variable $\lambda_{t}$ that auto-tunes online. At each chunk boundary, after observing the background rate $r_{t}$, the dual variable is updated via a projected gradient ascent step on the Lagrangian:
\begin{align}
    \lambda_{t+1} = [\lambda_{t} + \alpha_{\text{step}} (|r_{t} - r_{B}^{*}| - \tau)]^{+}
\end{align}
where $[\cdot]^{+}$ denotes projection onto $\mathbb{R}_{\geq 0}$ to maintain dual feasibility, $\alpha_{\text{step}} > 0$ is the dual step size, and $\tau$ is the tolerance width. Intuitively, $\lambda_{t}$ increases when the background rate violates the tolerance band and decreases when the constraint is satisfied, automatically tightening or relaxing the rate penalty over time. 

The reward in Eq~\ref{equation:reward_design} at step $t$ shall be written as:

\begin{equation}
\label{equation:reward_design_l_grpo}
\begin{split}
R_t^{\cdot}
&= \lambda_t
\underbrace{
  \begin{cases}
    1 - \displaystyle\left(\frac{|r_{t+1} - r^\star_{B}|}{\tau}\right)^{\!2}
      & \text{if } |r_{t+1} - r^\star_{B}| \leq \tau \\[6pt]
    -\displaystyle\left(\frac{|r_{t+1} - r^\star_{B}|}{\tau} - 1\right)
      & \text{if } |r_{t+1} - r^\star_{B}| > \tau
  \end{cases}
}_{\text{rate tracking}} \\
&\quad + (1 - \lambda_t)
\underbrace{
  \Big(\alpha\, \epsilon_{t+1}^{\ttbarraw,\cdot}
  + (1-\alpha)\, \epsilon_{t+1}^{\haaFourB,\cdot}\Big)
}_{\text{signal efficiency}}
\\ &\quad
- \lambda_2
\underbrace{
  \left(\frac{|\Delta \threshold_{t-1}^{\cdot}|}
             {\Delta \threshold_{\max}^{\cdot}}\right)
}_{\text{move penalty}}
\end{split}
\end{equation}

The policy update follows the standard GRPO objective in Eq.~\ref{eq:loss_grpo} with rewards rewards recomputed under the current $\lambda_{t}$ at each stepp.

\textbf{Implementation details.} We initialize $\lambda_{0} = \lambda_{1} = 0.25$, matching the fixed baseline for GRPO, GFPO-F and GFPO-FR, so that Lagrangian-GRPO starts from the same operating point and any subsequent improvement reflects the dual adaptation rather than initialization. We set the dual step size $\alpha_{\text{step}} = 0.01$ and clip $\lambda_{t} \in [0.05, 0.95]$ to prevent degeneracy: a near-zero $\lambda_{t}$ would eliminate the rate constraint entirely, while $\lambda_{t} \approx 1$ would suppress the signal efficiency objective. The dual variable is updated once per chunk, operating on the same timescale as the threshold update, and is not updated during shielded steps. All other hyperparameters $G, K, \lambda_{2}, \alpha$ in the reward mixing) follow the GRPO settings in Appendix~\ref{appendix:hyperparameter_ablation}.

\textbf{Relationship to PID-Lagrangian methods.} The update above corresponds to the integral term of a PID-Lagrangian controller~\citep{stooke2020responsive}, where the dual step size $\alpha_{\text{step}}$ plays the role of the integral gain. We omit the proportional and derivative terms for simplicity, as the GRPO policy already observes the instantaneous rate error and drift through the state representation (Appendix~\ref{sec:sequential_network_architectures}), making separate proportional and derivative correction in the dual redundant. 

\paragraph{Structural Failure Analysis of L-GRPO.}
  We provide empirical support for the structural argument in Section~\ref{sec:from_dqn_to_grpo} that L-GRPO's dual variable $\lambda_t$ cannot resolve the zero-feasibility failure mode. All diagnostic runs use the AD trigger on CMS real collison data (Table~\ref{tab:single_trigger_summary_compact_realdata}), which represents the harder case (54.1\% in-band). The $H_T$ trigger shows qualitatively identical behavior on par with GRPO baseline and is omitted for brevity.   

\paragraph{Zero-feasible fraction is structurally unchanged by $\lambda_t$ (Figure~\ref{fig:lgrpo_zero_feas}).}  Both GRPO and L-GRPO yield $\approx 20\%$ zero-feasible steps on MC and $\approx 61\%$ on CMS, with per-chunk trajectories that are nearly indistinguishable (Figure~\ref{fig:lgrpo_zero_feas}b).
  
  Since $\lambda_t$ modulates reward weights only, it has no effect on   the proposal distribution of $G$ and cannot reduce the fraction of steps with zero feasible candidates.                                                      
  \paragraph{$\lambda_t$ adapts during violations but the in-band rate
  does not recover (Figure~\ref{fig:lgrpo_lam_traj}).}
  On CMS, $\lambda_t$ remains active near $0.25$ and responds to                                                                  
  violation windows, yet the in-band rate shows no systematic recovery                                                            
  and oscillates around $54\%$ throughout.                                                                                        
  When the sampled group contains no feasible action, group-relative                                                              
  normalisation assigns positive advantage to the least-infeasible                                                                
  out-of-band candidate regardless of $\lambda_t$; increasing
  $\lambda_t$ amplifies this pathological update rather than correcting it.             
  \paragraph{L-GRPO's rate trajectory is systematically below the                                                                 
  tolerance band on CMS (Figure~\ref{fig:lgrpo_bg_traj}).}                                                                        
  Despite sharing the same MC-pretrained weights as GRPO at deployment,
  L-GRPO drifts persistently below the lower band edge while GRPO and                                                             
  both GFPO variants remain near the target rate.                                                                                 
  The gap isolates the reward formulation as the source of failure under                                                          
  distribution shift, and motivates the feasibility-first candidate                                                               
  selection of GFPO (Section~\ref{sec:gfpo}).

\begin{figure*}
    \centering
\includegraphics[width=\linewidth]{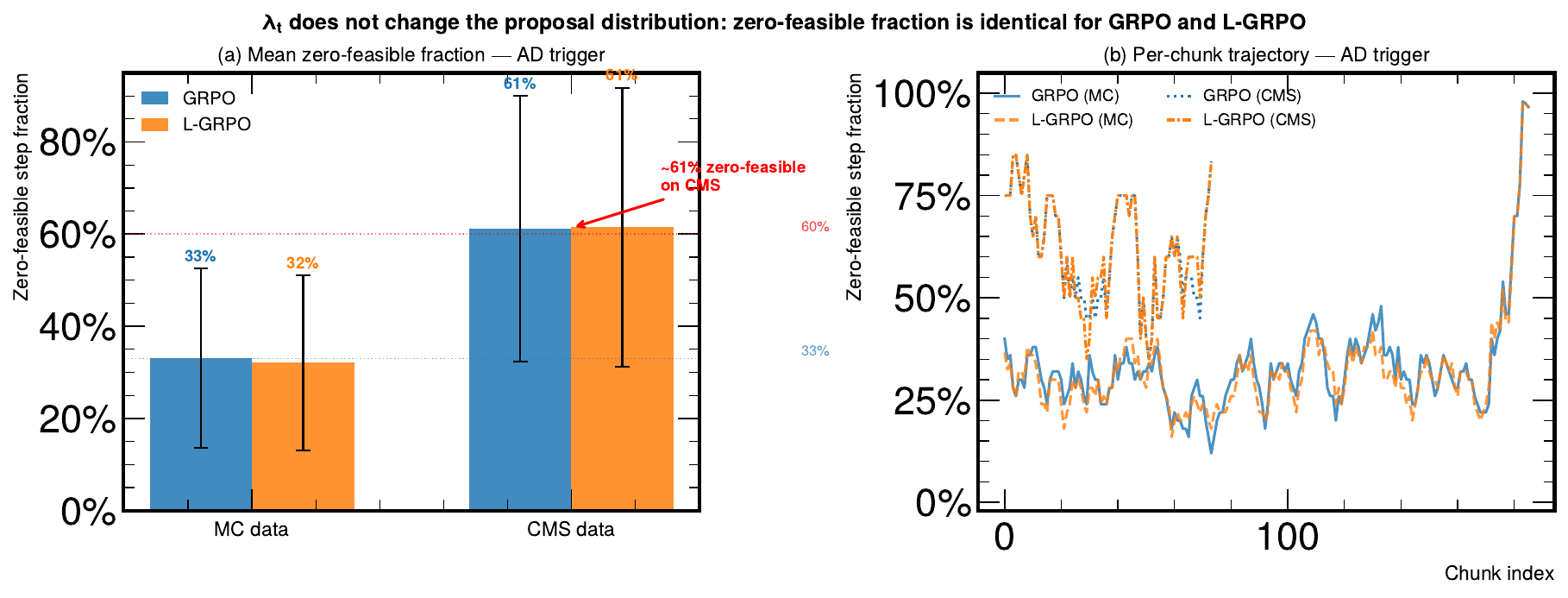}
    \caption{\textbf{Zero-feasible step fraction for GRPO and L-GRPO (AD trigger).} (a) Mean zero-feasible fraction across all chunks. On MC data, approximately 20\% of micro-steps yield no rate-feasible candidate from the sampled group
   of $G$ actions for both GRPO and L-GRPO. On CMS real data this rises to $\sim$61\% for both methods. (b) Per-chunk trajectory
  over time. The two curves are nearly indistinguishable in both datasets, confirming that the dual variable $\lambda_t$ modulates
   reward weights only and has no effect on the proposal distribution or the fraction of feasible candidates.
    }            
    \label{fig:lgrpo_zero_feas}                   
\end{figure*}

\begin{figure*}
    \centering
\includegraphics[width=\linewidth]{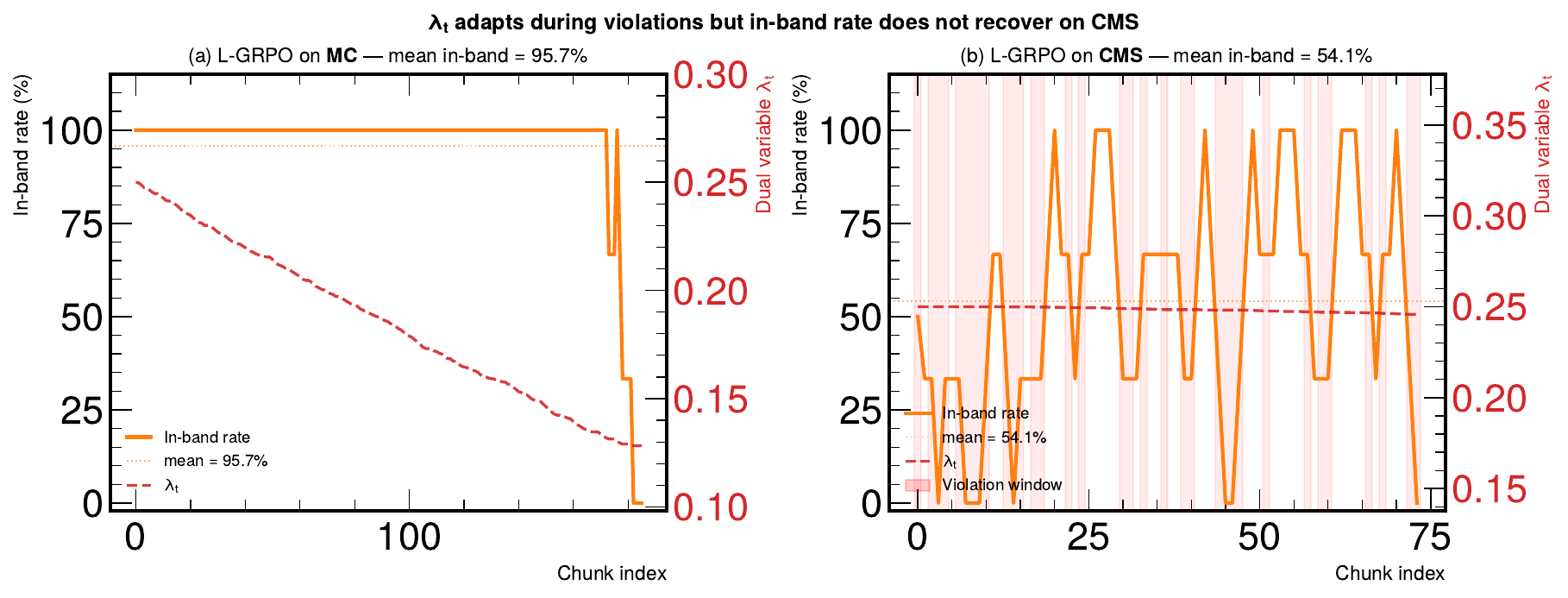}
    \caption{                                                    
    \textbf{Dual variable $\lambda_t$ and in-band rate over chunks for L-GRPO (AD trigger).} (a) On MC data, $\lambda_t$ decreases
  monotonically from its initial value of 0.25 as the policy stays in-band ($\sim$96\% of chunks). (b) On CMS real collision data,           
  $\lambda_t$ fluctuates but remains roughly constant ($\approx 0.25$) while the in-band rate oscillates around 54\%. Shaded
  regions mark chunks where the in-band rate is below 50\%. Despite $\lambda_t$ responding to these violation windows, the in-band 
  rate does not systematically recover, confirming that increasing $\lambda_t$ amplifies the pathological group update rather than
   correcting the underlying sampling failure.
    }            
    \label{fig:lgrpo_lam_traj}                   
\end{figure*}

\begin{figure*}
    \centering
\includegraphics[width=\linewidth]{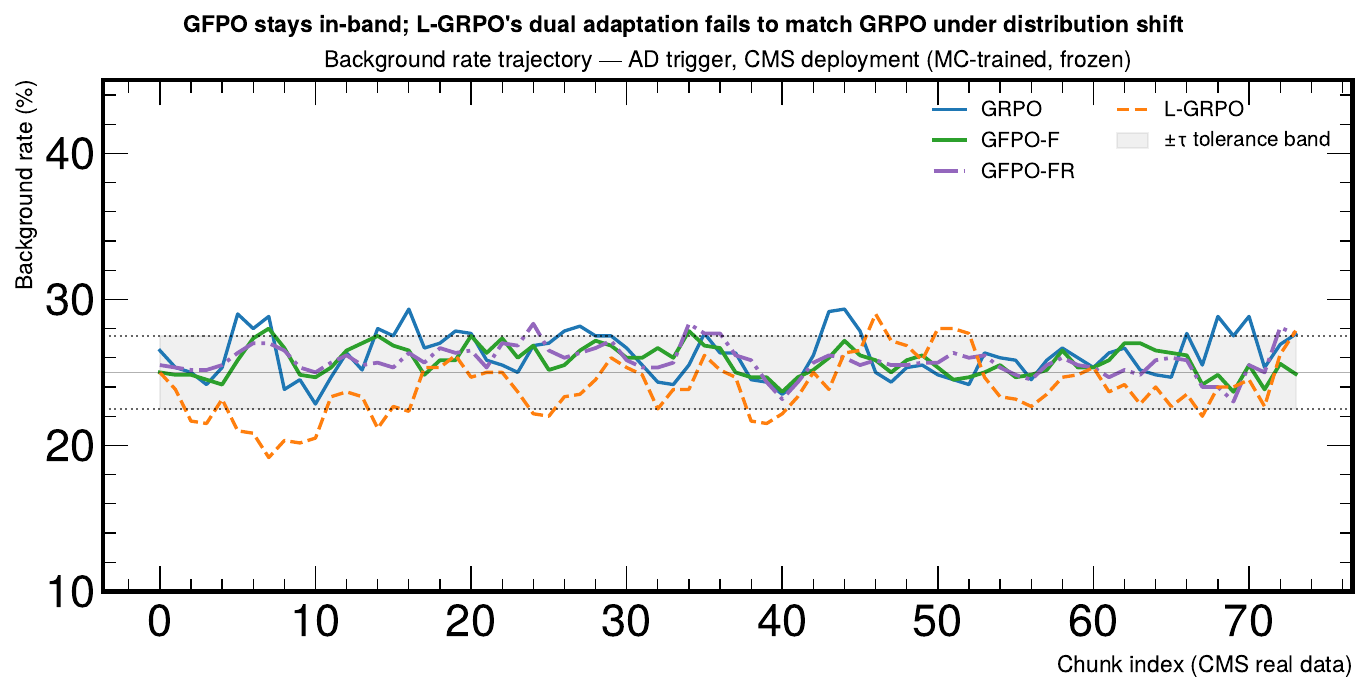}
    \caption{\textbf{Background rate trajectory on CMS real data (AD trigger, 
MC-trained frozen policy).} Grey band: $\pm\tau$ tolerance around the 
target $r_B^{*}$. GFPO-F and GFPO-FR track within the band for most of 
the 73-chunk deployment window. GRPO and L-GRPO sit near or below the 
lower band edge throughout, despite L-GRPO's dual variable adapting 
concurrently (Figure~\ref{fig:lgrpo_lam_traj}). The performance gap is 
therefore \emph{not} a matter of insufficient penalty strength but of the 
GRPO candidate-sampling mechanism failing to produce feasible actions 
under distribution shift.
    }            
    \label{fig:lgrpo_bg_traj}                   
\end{figure*}

\subsubsection{CPO}
\label{appendix:cpo}
\paragraph{Implementation details.} We set $\lambda_1 = 0.25$ (in CPO this is a fixed reward weight, distinct from the trust-region dual $\nu^\star$ that the QP solves for at each update).
CPO inherits the signal and move-penalty pieces of Eq.~\ref{equation:reward_design} but moves the rate-tracking term from the reward into the cost constraint. Concretely, for a
candidate $\Delta$-cut $\delta_k$ at time $t$ we evaluate reward (with capital letters)
\begin{equation}
\label{eq:cpo-reward}
R_k \;=\; (1-\lambda_1)\,\underbrace{\bigl(\alpha\,\epsilon^{t\bar t}_{t+1} + (1-\alpha)\,\epsilon^{h\to 4b}_{t+1}\bigr)}_{\text{signal efficiency}} \;-\;
\lambda_2\,\underbrace{\frac{|\Delta c_{t-1}|}{\Delta c_{\max}}}_{\text{move penalty}}
\end{equation}
and cost (with capital letters)
\begin{equation}
\label{eq:cpo-cost}
C_k \;=\;
\begin{cases}
\bigl(|r_{t+1} - r_B^\star|/\tau\bigr)^2 & \text{if } |r_{t+1} - r_B^\star|\le\tau,\\[2pt]
|r_{t+1} - r_B^\star|/\tau              & \text{if } |r_{t+1} - r_B^\star|>\tau,
\end{cases}
\end{equation}
where $r_{t+1}$ is the realized background rate after the candidate cut is
applied. The cost is the sign-flipped, vertically shifted version of Eq.~\ref{equation:reward_design}'s
rate-tracking term: it equals $0$ at the target, ramps quadratically to $1$ at
the band edge, and grows linearly with the relative violation beyond, remaining
non-negative and continuous at $e_t = 1$. The constraint enforced by CPO is on
the expected per-step cost,
\begin{equation}
\label{eq:constraint}
J_C(\pi_\theta) \;=\; \mathbb{E}_{(s,a)\sim\pi_\theta}\!\bigl[C(s,a)\bigr] \;\le\; d,
\end{equation}
with $d$ a hyperparameter (the per-step rate-violation budget). All other terms
($\alpha$, $\lambda_1$, $\lambda_2$, $\tau$, $r_B^\star$, $\Delta c_{\max}$) keep
the values used elsewhere in the paper. We set $d = 1$, the cost evaluated at the tolerance-band edge, so that any operating point with average rate violation within $\tau$ is constraint-feasible and CPO is not driven toward the recovery branch by its budget alone.

\paragraph{CPO update.}
  We adapt CPO~\citep{achiam2017constrained} to the bandit form of our setup. At
  each micro-step we sample $G=16$ candidate $\Delta$-cuts $a_k\sim\pi_\theta(\cdot|s)$,
  evaluate $(r_k, c_k)$ under the simulator (Eqs.~\ref{eq:cpo-reward}--\ref{eq:cpo-cost}),
  and form \emph{group-relative} advantages
  $\hat A^R_k = R_k - \bar R$,
  $\hat A^C_k = C_k - \bar C$,
  replacing the separate value-function critics used in the original CPO; like the
  GRPO/GFPO baselines, the bandit reduction removes the unreliable value-estimation
  step in the rare-signal regime. Every $50$ micro-steps (with at least $128$
  samples buffered) we form the linearized trust-region program
  \begin{equation}
  \max_{\theta'}\, g^\top(\theta'{-}\theta)\quad\text{s.t.}\quad
  c + b^\top(\theta'{-}\theta) \le 0,\quad
  \tfrac{1}{2}(\theta'{-}\theta)^\top H(\theta'{-}\theta) \le \delta_{\mathrm{KL}},
  \end{equation}
  where $g$ and $b$ are the gradients of the importance-weighted reward and cost
  surrogates at $\theta$, $c = J_C(\pi_\theta) - d$, and $H = \nabla_\theta^2
  \mathrm{KL}(\pi_{\mathrm{old}}\,\|\,\pi_\theta)$ is the Fisher matrix. We
  evaluate Fisher--vector products through mean-KL Hessian--vector products,
  solve $H^{-1}g$ and $H^{-1}b$ by conjugate gradient ($10$ iterations, damping
  $0.1$), and obtain the dual variables $(\lambda^\star, \nu^\star)$ in closed
  form following~\citet[Sec.~6.2]{achiam2017constrained}. A backtracking line
  search (initial step $1.0$, decay $0.8$, up to $10$ backtracks) accepts the
  update if the measured KL stays within $1.5\,\delta_{\mathrm{KL}}$, the reward
  surrogate does not regress, and the projected cost stays under budget. When
  $c > 0$ and $2\delta_{\mathrm{KL}} - c^2/s < 0$ the program is infeasible and
  we apply the recovery step (Eq.~14 of~\citealp{achiam2017constrained}),
  $\theta' - \theta = -\sqrt{2\delta_{\mathrm{KL}}/s}\,H^{-1}b$, which descends
  purely along the cost direction; the line search is then required only to
  reduce the cost surrogate. We train on $3$ seeds over the full
  MC trajectory using a single configuration ($\delta_{\mathrm{KL}} = 0.03$, $d = 1$, $G = 16$, batch threshold $128$) and
  do not perform a hyperparameter sweep; the values chosen are conservative
  defaults from~\citet{achiam2017constrained} for $\delta_{\mathrm{KL}}$ and the
  CG/line-search budgets, with $d = 1$ matching the cost at the band edge.

\subsection{GRPO Failure Mode}
\label{appendix:GRPO_Failure_Mode}
This appendix supplements the $H_T$ analysis in the main text (Figures~\ref{fig:grpo_failure_mode_ht},~\ref{fig:gfpo_intermediate_ht}) with the parallel analysis on the AD trigger (Figures~\ref{fig:grpo_failure_mode_ad} and \ref{fig:gfpo_intermediate_ad}).

\paragraph{Failure does not vanish with group size.} Figure~\ref{fig:grpo_failure_mode} plots the fraction of training steps with zero feasible rollouts as a function of $G \in \{4, 8, \ldots, 256\}$. The AD curve is flat at 34--36\% across all $G$; $H_T$ plateaus above 24\% even at $G=256$. If zero-feasibility were a sampling-variance artifact, scaling $G$ would drive the fraction to zero. It does not: zero feasibility is a property of the policy distribution under the constraint, not of the estimator's sample size.

\begin{figure*}
    \centering
    \begin{subfigure}{0.9\textwidth}
\includegraphics[width=\linewidth]{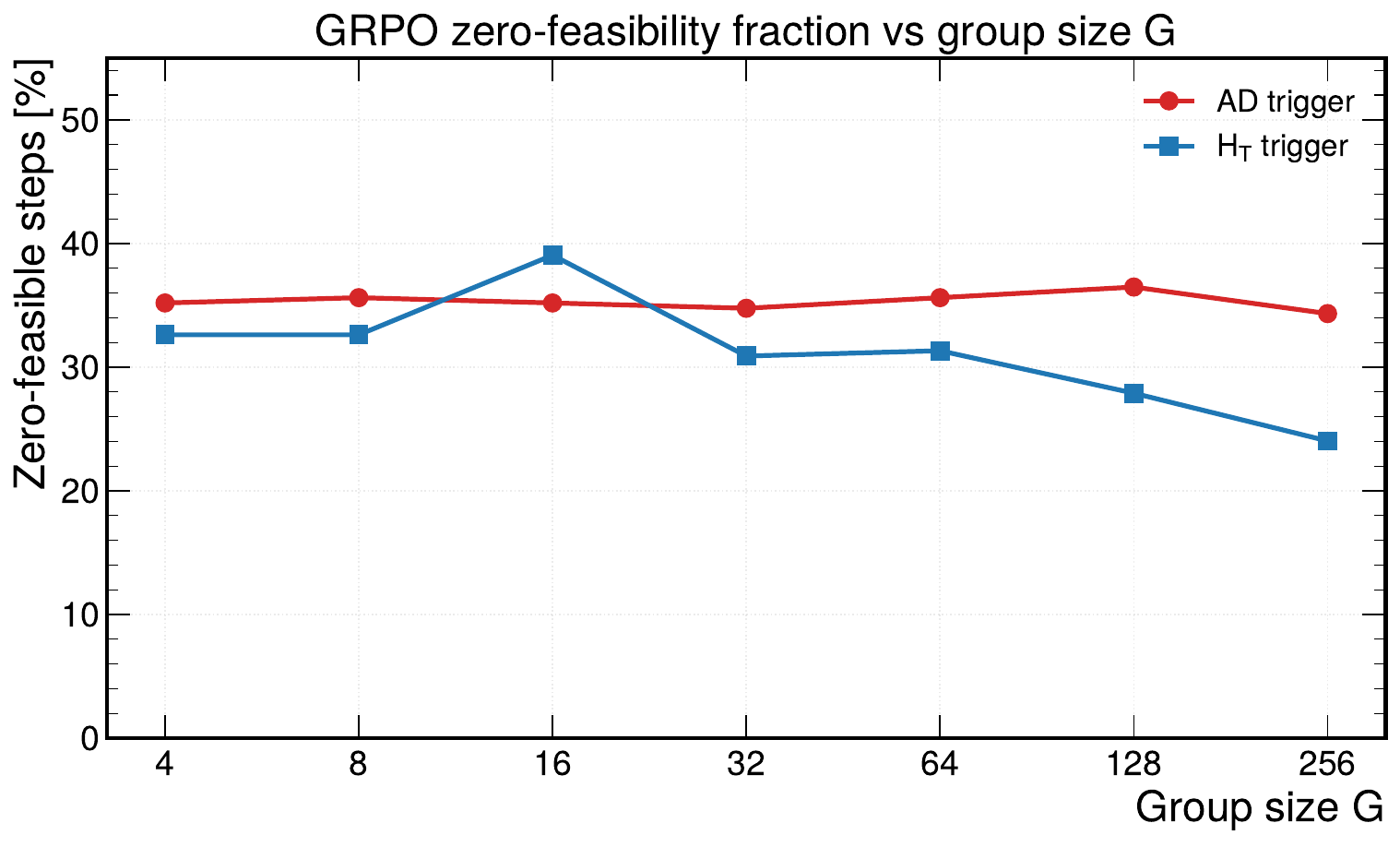}
    \end{subfigure}
    \caption{\textbf{Zero-feasibility fraction of GRPO vs.\ group size $G$.} Fraction of training steps on which every rollout violates the rate budget, for AD and $H_T$ triggers. Even at $G=256$, 24--34\% of steps yield no feasible sample, so the group-relative baseline collapses to an uninformative signal on a constant fraction of updates. Scaling $G$ alone cannot recover constraint satisfaction, motivating the constraint-aware variants of GRPO in Section~\ref{sec:gfpo}.}
\label{fig:grpo_failure_mode}
\end{figure*}

\begin{figure*}
    \centering
    \begin{subfigure}{0.9\textwidth}
\includegraphics[width=\linewidth]{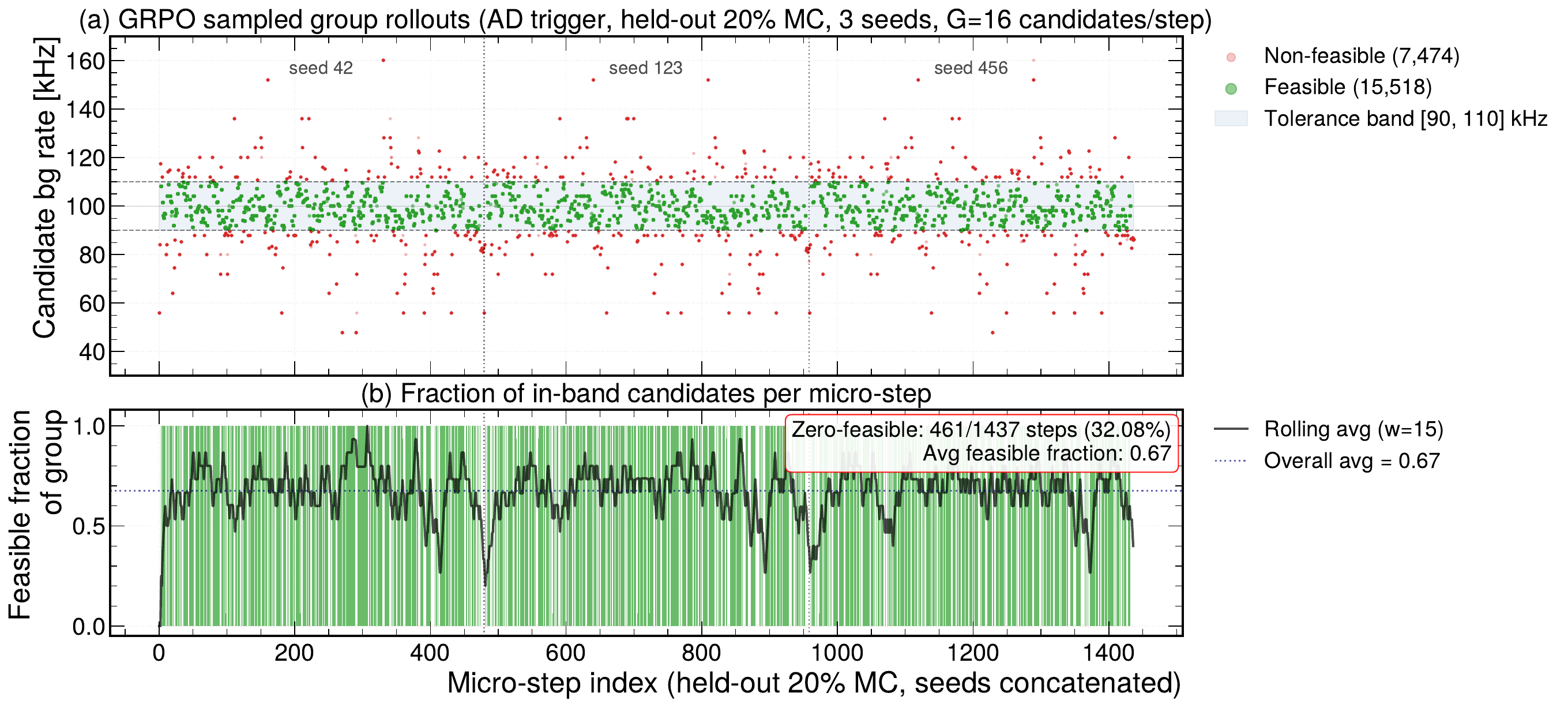}
    \end{subfigure}
    \caption{\textbf{GRPO group-feasibility failure (AD trigger, MC, $G{=}16$); appendix complement to Fig.~\ref{fig:grpo_failure_mode_ht}.} (a) Background rates of all $G{=}16$ candidate rollouts per micro-step; \textcolor{darkgreen}{green}/\textcolor{red}{red} denotes feasible/infeasible w.r.t.\ the $[90, 110]$~kHz band. (b) Per-step feasible fraction $f_t = n_{\text{feas}}/G$ (run mean $\langle f_t\rangle = 0.67$). On 32.08\% of steps $f_t = 0$: every candidate lies outside the band, the within-group advantage has zero variance, and the policy receives no feasibility signal. A 67\% in-band mean coexisting with a $\sim\!32.08\%$ zero-feasible tail is the regime where group-relative updates fail intermittently, motivating the feasibility-filtered GFPO update.}
    \label{fig:grpo_failure_mode_ad}
\end{figure*}

\begin{figure*}
    \centering
    \begin{subfigure}{0.5\textwidth}
\includegraphics[width=\linewidth]{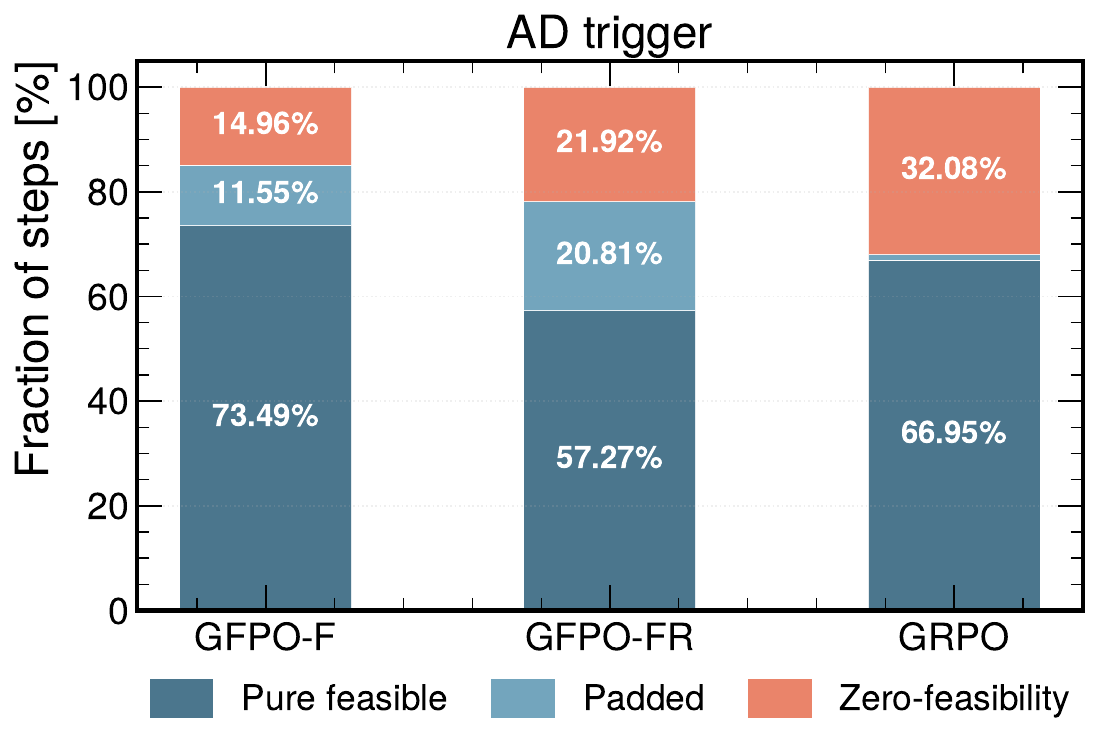}
    \end{subfigure}
    \caption{\textbf{Step composition for the AD trigger; Complement to Figure~\ref{fig:gfpo_intermediate_ht}.} Categories follow Figure~\ref{fig:gfpo_intermediate_ht}:
  \textit{Pure feasible} ($|\mathcal{F}_t|\ge K$), \textit{Padded} $0<|\mathcal{F}_t|<K$), \textit{Zero-feasibility} ($\mathcal{F}_t = \emptyset$).
  GFPO-F/FR achieve $57\sim73\%$ pure-feasible kept sets and pad on $14 \sim 21\%$ of steps. GRPO is dominated by partial-feasibility ($66.95\%$) and zero-feasibility ($32.08\%$) steps. The latter matches the value in Figure~\ref{fig:grpo_failure_mode_ad} and is the failure mode shielded by GFPO.  
}
    \label{fig:gfpo_intermediate_ad}
\end{figure*}

\subsection{Sequential Network Architectures}
\label{sec:sequential_network_architectures}

\paragraph{Why a sequence encoder.} The trigger-control problem is partially observed and non-stationary: at update step $t$ we observe a finite recent window of background events, while the threshold-to-rate mapping drifts with pileup ($N_{pv}$), detector conditions, and score-distribution shifts. A scalar controller (e.g., PID~\citep{emami2026selfdrivingtriggerlhcadaptive}) sees only the instantaneous rate error. We need a state that exposes the recent evolution of the event-level distribution near the cut --- near-cut occupancy, tail thickness, and the conditional dependence of pass indicators on score and kinematics. Algorithm~\ref{alg:seq_state_and_controller} presents the full pipeline of building sequential state for all RL policy.

\paragraph{State.} At update step $t$, the state is a fixed-length sequence $S_t \in \mathbb{R}^{K \times F}$ built from the most recent $K$ background events, with $F$ features per event (Algorithm~\ref{alg:seq_state_and_controller}). Features fall into three groups and broadcast across the sequence:
\begin{itemize}\itemsep0pt
  \item \textbf{Event-level}: score and kinematics; pileup context.
  \item \textbf{Threshold geometry}: distance-to-cut; pass indicator.
  \item \textbf{Chunk-level control}: rate error, drift, last action, feasibility.
\end{itemize}

\paragraph{Architecture.} All RL agents in this work share the same sequential architecture for both state encoding and policy. The same input format and network apply to the $H_T$ and AD triggers.

\paragraph{Setup.} With $\lambda_1 = 0.25$, $\lambda_2 = 1.00$ fixed 
(Section~\ref{appendix:hyperparameter_ablation}), we ablate the recurrent cell 
in the shared sequence encoder. All five RL agents (DQN, GRPO, GFPO-F, 
GFPO-FR, PPO) use the same architecture: a single-layer RNN followed by a 
linear projection head, trained on MC and deployed on CMS Run 283408.

\paragraph{RNNs compared.} LSTM (4 gates, 7{,}333 
parameters)~\citep{graves2012long}; GRU (3 gates, 5{,}541 
parameters)~\citep{dey2017gate}; vanilla RNN with $\tanh$ (1{,}957 
parameters); vanilla RNN with ReLU (1{,}957 parameters).

\paragraph{Results.} Table~\ref{tab:rnn_trigger_comparison} reports means 
over the five agents. On the $H_T$ trigger, signal efficiencies are nearly 
architecture-invariant: \ttbar\ overall varies by under 0.1\% 
(97.57--97.64\%) and \haaFourB\ overall by under 0.1\% (33.29--33.38\%) 
across all four cells. InBand rate is the discriminating axis: GRU reaches 
0.849, vanilla RNN 0.789.

\begin{table*}[t]
\centering
\caption{State representation ablation on CMS real data (Run~283408). All models trained on 80\% MC with $\lambda_1{=}0.25$,        
  $\lambda_2{=}1.00$. Results averaged across five RL methods (DQN, GRPO, GFPO-F, GFPO-FR, PPO).}
\label{tab:rnn_trigger_comparison}
\begin{tabular}{llccccccc}
\toprule
Trigger & RNN & Params & Gates & InBand & $\epsilon_{\text{inband}}^{\ttbarraw} \uparrow$ & $\epsilon_{\text{inband}}^{\haaFourB} \uparrow$ & $\epsilon_{\text{overall}}^{\ttbarraw} \uparrow$ & $\epsilon_{\text{overall}}^{\haaFourB} \uparrow$ \\
\midrule
$H_{T}$ & LSTM      & 7,333 & 4 & 0.830 & 97.50 & 33.07 & 97.57 & 33.29 \\
$H_{T}$ & RNN       & 1,957 & 1 & 0.789 & \textbf{97.68} & \textbf{33.60} & \textbf{97.64} & \textbf{33.38} \\
$H_{T}$ & RNN--ReLU & 1,957 & 1 & 0.800 & 97.64 & 33.54 & 97.60 & 33.36 \\
$H_{T}$ & GRU       & 5,541 & 3 & \textbf{0.849} & 97.48 & 33.06 & 97.57 & 33.30 \\
\midrule
AD & LSTM      & 7,333 & 4 & 0.570 & 75.89 & 43.01 & 75.30 & 39.80 \\
AD & RNN       & 1,957 & 1 & 0.570 & 75.94 & 43.22 & 75.29 & 39.84 \\
AD & RNN--ReLU & 1,957 & 1 & 0.568 & 75.89 & 43.24 & \textbf{75.34} & \textbf{39.86} \\
AD & GRU       & 5,541 & 3 & \textbf{0.573} & \textbf{76.02} & \textbf{43.32} & 75.31 & 39.81 \\
\bottomrule
\end{tabular}
\end{table*}

\paragraph{Choice.} Trigger rate scales with instantaneous 
luminosity~\citep{aad2022performance, atlas2020operation, 
govorkova2022autoencoders}, so rate stability is operationally critical. 
We therefore prioritize InBand rate and adopt GRU as stated in Section~\ref{sec:gfpo}.


\paragraph{State representation.} We summarize the $K$-event window with 
$F$ features per event (Algorithm~\ref{alg:seq_state_and_controller}). 
Features fall into three groups, listed in Table~\ref{tab:state_features}: 
per-event physics features, chunk-level control state (broadcast across 
all $K$ events), and multi-resolution near-cut occupancy. We justify 
the non-obvious choices below; the same input format is used for both 
the $H_T$ and AD triggers, though constructed separately. Total dimensionality is $F = 22$ per timestep ($5 + 14 + 3$).         

\begin{table*}[t]
\centering
\caption{State features. $\mu, \sigma$: window mean and std of the trigger 
observable. $r_t$: chunk-$t$ background rate. $r_B^{*}$: target rate. 
$\threshold$: current threshold.}
\label{tab:state_features}
\small
\begin{tabular}{l l l}
\toprule
Group & Feature & Definition \\
\midrule
\multirow{4}{*}{Per-event physics}
 & Normalized observable & $(x_i - \mu)/\sigma$ \\
 & Pass flag & $\mathbf{1}[x_i \ge \threshold]$ \\
 & Signed distance to cut & $(x_i - \threshold)/\sigma$ \\
 & Normalized $N_{pv,i}$ & $(N_{pv,i} - \overline{N_{pv}})/\mathrm{std}(N_{pv})$ \\
 & Window position & $\mathrm{linspace}(0,1,K)$ \\  
\midrule
\multirow{10}{*}{Chunk-level (broadcast)}
 & Relative rate error & $(r_t - r_B^{*})/r_B^{*}$ \\
 & Absolute rate error & $|r_t - r_B^{*}|/r_B^{*}$ \\
 & Rate drift & $(r_t - r_{t-1})/r_B^{*}$ \\
 & InBand flag & $\mathbf{1}[|r_t - r_B^{*}| \le \tau]$ \\
 & Normalized threshold & $(\threshold_t - \mathrm{mid})/\mathrm{span}$ \\
 & Last action & $\Delta\threshold_{t-1}/\Delta\threshold_{\max}$ \\
 & Pileup mean & $\overline{N_{pv}}$ over window \\
 & Pileup std & $\mathrm{std}(N_{pv})$ over window \\
 & EMA rate error ($\lambda{=}0.95$) & $\mathrm{EMA}_\lambda[(r_t - r_B^{*})/r_B^{*}]$ \\
 & Sensitivity probe & $\partial r/\partial \threshold$, central-difference, normalized by $r_B^{*}$ \\
 & Look-ahead rate $p_1$        & $r(\threshold + \Delta)$  \\                                    
   & Look-ahead rate $p_2$        & $r(\threshold + 2\Delta)$  \\                                   
   & Survival ratio $r_1$         & $r(\threshold + \Delta)/r(\threshold)$  \\
   & Survival ratio $r_2$         & $r(\threshold + 2\Delta)/r(\threshold + \Delta)$  \\           
\midrule
\multirow{3}{*}{Near-cut indicators}
 & Width $w_1$ & $\mathbf{1}[|x_i - \threshold| \le w_1]$ \\
 & Width $w_2$ & $\mathbf{1}[|x_i - \threshold| \le w_2]$ \\
 & Width $w_3$ & $\mathbf{1}[|x_i - \threshold| \le w_3]$ \\
\bottomrule
\end{tabular}
\end{table*}

Widths are $(5, 10, 20)$~GeV for $H_T$ and $(0.25, 0.5, 1.0)$ for the AD score. 
$\mathrm{mid}$ and $\mathrm{span}$ are computed once at the start of the 
first chunk: $\mathrm{mid} = 0.5(P_{95} + P_{99.99})$ and $\mathrm{span} = 
P_{99.99} - P_{95}$ over the calibration window, analogous to how a static 
LHC trigger menu is set.

\paragraph{Normalized threshold (Figure~\ref{fig:eda_normalized_cut}).} 
Pileup drift makes the percentile rank of a fixed raw threshold non-stationary 
(top row): the same $\threshold$ produces different rates over time. 
Normalization by the calibration $\mathrm{mid}$ and $\mathrm{span}$ stabilizes 
the threshold representation; even after normalization, the ideal operating 
point drifts (bottom row), so the agent must observe $\threshold_{\text{norm}}$ 
to know whether its threshold sits high or low relative to the score 
distribution. Without it, the agent would observe only the rate error, purely reactive, not anticipatory.

\paragraph{Window position.} The per-event index 
$\tau_i = (i-1)/(K-1) \in \{0, \tfrac{1}{K-1}, \dots, 1\}$ provides an 
explicit recency signal, freeing the GRU from having to recover position 
information from its hidden-state recurrence alone. This follows the 
positional-encoding inductive bias of \citet{vaswani2017attention} (and 
its time-series extensions, e.g., \citet{kazemi2019time2vec, 
lim2021temporal}); we broadcast the same ramp identically every chunk.

  
\paragraph{EMA rate error (Figure~\ref{fig:ema_over_time}).} The 
instantaneous error is window-to-window noisy; the EMA at $\lambda = 0.95$ 
smoothly tracks the sustained drift out of the tolerance band, analogous 
to the integral term in PID. This gives the agent memory of accumulated 
deviations and distinguishes transient fluctuations from regime shifts.

\paragraph{Sensitivity probe (Figure~\ref{fig:sensitivity_probe_over_time}).}
A central-difference estimate of $\partial r/\partial \threshold$ gives 
the agent local gradient information about the trigger curve: a steep 
slope means small threshold changes produce large rate effects (fine 
adjustments needed); a shallow slope means the agent can move further 
safely.

\paragraph{Pileup features (Figure~\ref{fig:npv_over_time}).} $N_{pv}$ 
distributions drift over the run; including window mean and std lets 
the agent condition on the current pileup regime.

\paragraph{Near-cut indicators (Figure~\ref{fig:near_cut}).} Multi-resolution 
counts of how many events sit close to the cut predict how sensitive the 
rate is to small threshold adjustments.

\paragraph{Look-ahead rate and survival ratios (Figure~\ref{fig:eda_tail_shape}).}   The look-ahead rates $p_1 = r(\threshold + \Delta)$ and $p_2 = r(\threshold + 2\Delta)$ 
evaluate the background rate at the next two grid points along the discrete 
action axis ($\Delta = 1$~GeV for $H_T$, $\Delta = 0.5$ for AD); the survival 
ratios $r_1 = p_1/p_0$ and $r_2 = p_2/p_1$ report the rate decay between 
consecutive shifts, where $p_0 = r(\threshold) \approx r_B^{*}$ at the natural 
operating point $\threshold = P_{99.75}$. Where the central-difference 
sensitivity probe captures the local slope at $\threshold$, $(p_1, p_2, r_1, r_2)$ 
describes the \emph{shape} of the cut neighborhood reachable in one or two 
actions: a thin tail produces small $r_1, r_2$ (one upward step removes most 
of the rate); a fat tail produces large $r_1, r_2$ (several steps needed). 
The four scalars drift across the run on both triggers, motivating their 
inclusion as state inputs.

\begin{figure*}
    \centering
    \includegraphics[width=\linewidth]{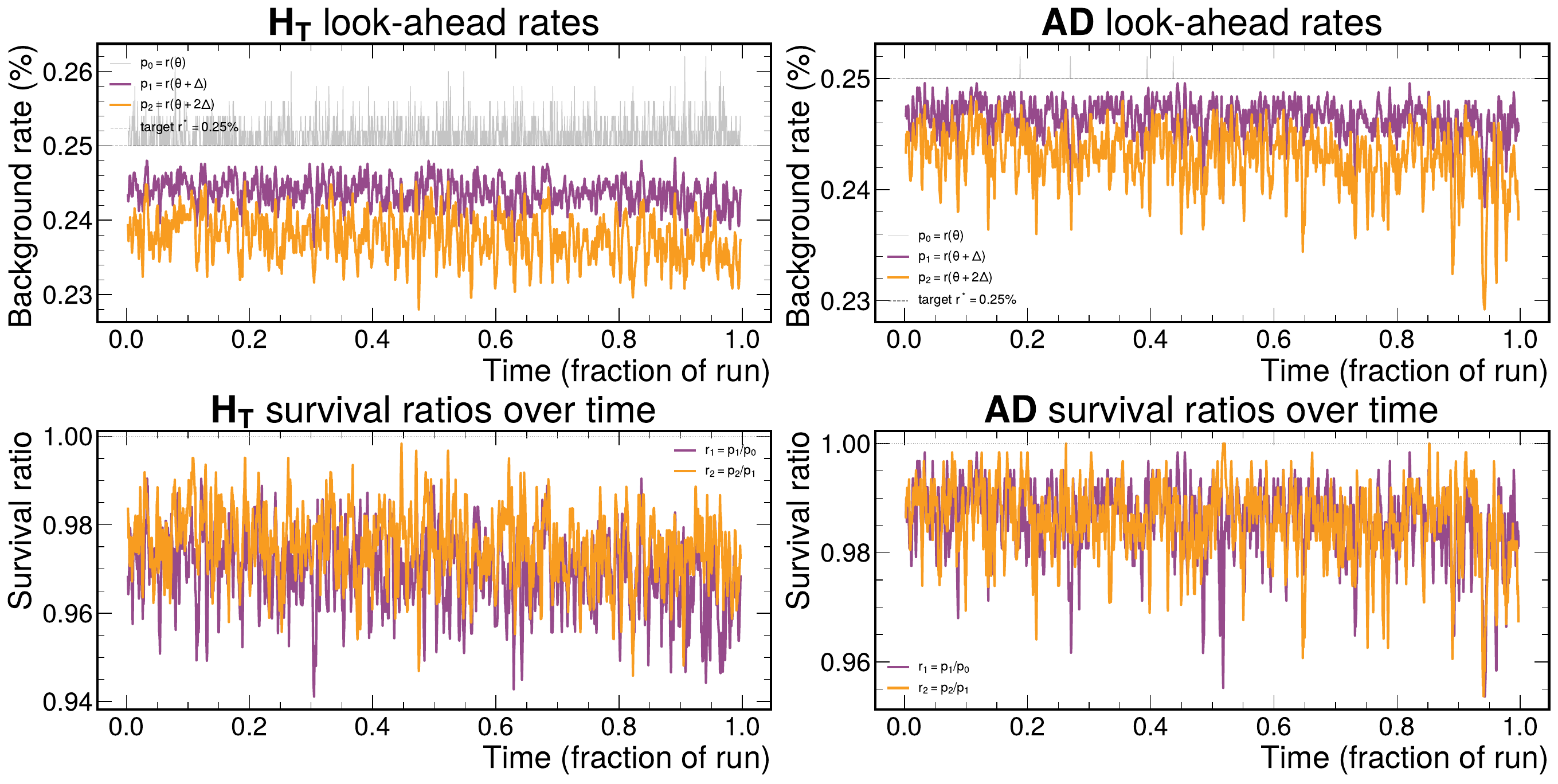}
    \caption{Look-ahead rates ($p_1, p_2$) and survival ratios ($r_1, r_2$)
  over time, evaluated at each window's natural operating point                                     
  $\theta = P_{99.75}$ so that $p_0 \approx r_B^{*}$ by construction. The                           
  four scalars give the agent a discrete forward-rate map of the tail at                            
  $\theta + \Delta$ and $\theta + 2\Delta$ (with $\Delta = 1$~GeV on $H_T$                          
  and $\Delta = 0.5$ on AD), complementing the local sensitivity probe                              
  $\partial r/\partial \theta$. Their drift over the run motivates including                        
  $(p_1, p_2, r_1, r_2)$ as state inputs.}
    \label{fig:eda_tail_shape}
\end{figure*}

\begin{figure*}
    \centering
    \includegraphics[width=\linewidth]{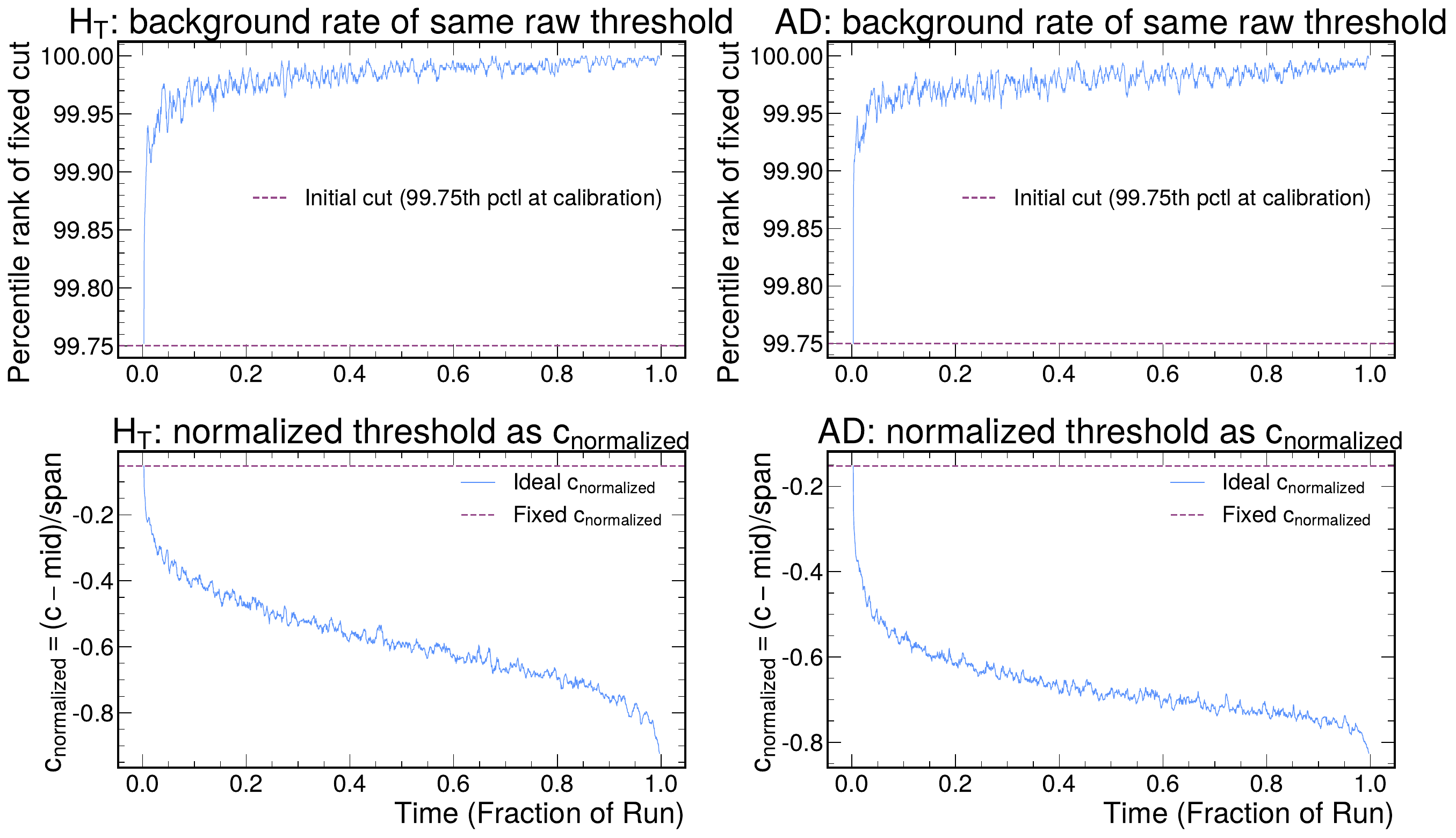}
    \caption{Non-stationarity of trigger thresholds under pileup drift (MC). \textbf{Top row}: The percentile rank of a fixed raw threshold (calibrated to the 99.75th percentile of the first 50k events) drifts substantially over time for both $H_{T}$ (left) and AD (right), meaning the same raw threhsold yields a changing background rate. \textbf{Bottom row}: The normalized threshold $\threshold_{\text{normalized}}$ where mid and span are derived from the P95-P99.99 calibration range of the first 50k events, illustrates that the ideal operating point (blue) deviates from the fixed calibration value (red dashed line). The ideal $\threshold_{\text{normalized}}$ (blue) is calculated as the 99.75th percentile of each 50k events' score distribution. This is the threshold that would give exactly $r^{*} = 0.25\%$. This suggests including $\threshold_{\text{normalized}}$ as a state feature so the agent can determine where $\threshold_{t}$ sits relative to the score distribution.}
    \label{fig:eda_normalized_cut}
\end{figure*}

\begin{figure*}
    \centering
    \includegraphics[width=\linewidth]{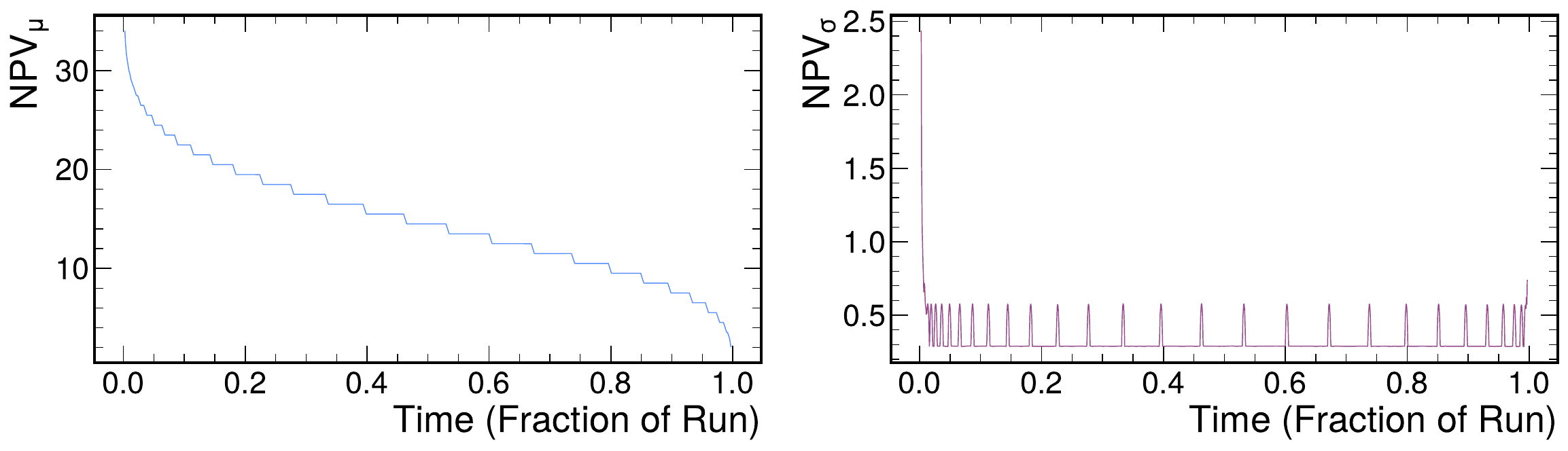}
    \caption{$N_{pv}$ (Number of Primary Vertices) distribution over time (MC).}
    \label{fig:npv_over_time}
\end{figure*}

\begin{figure*}
    \centering
    \includegraphics[width=1.03\linewidth]{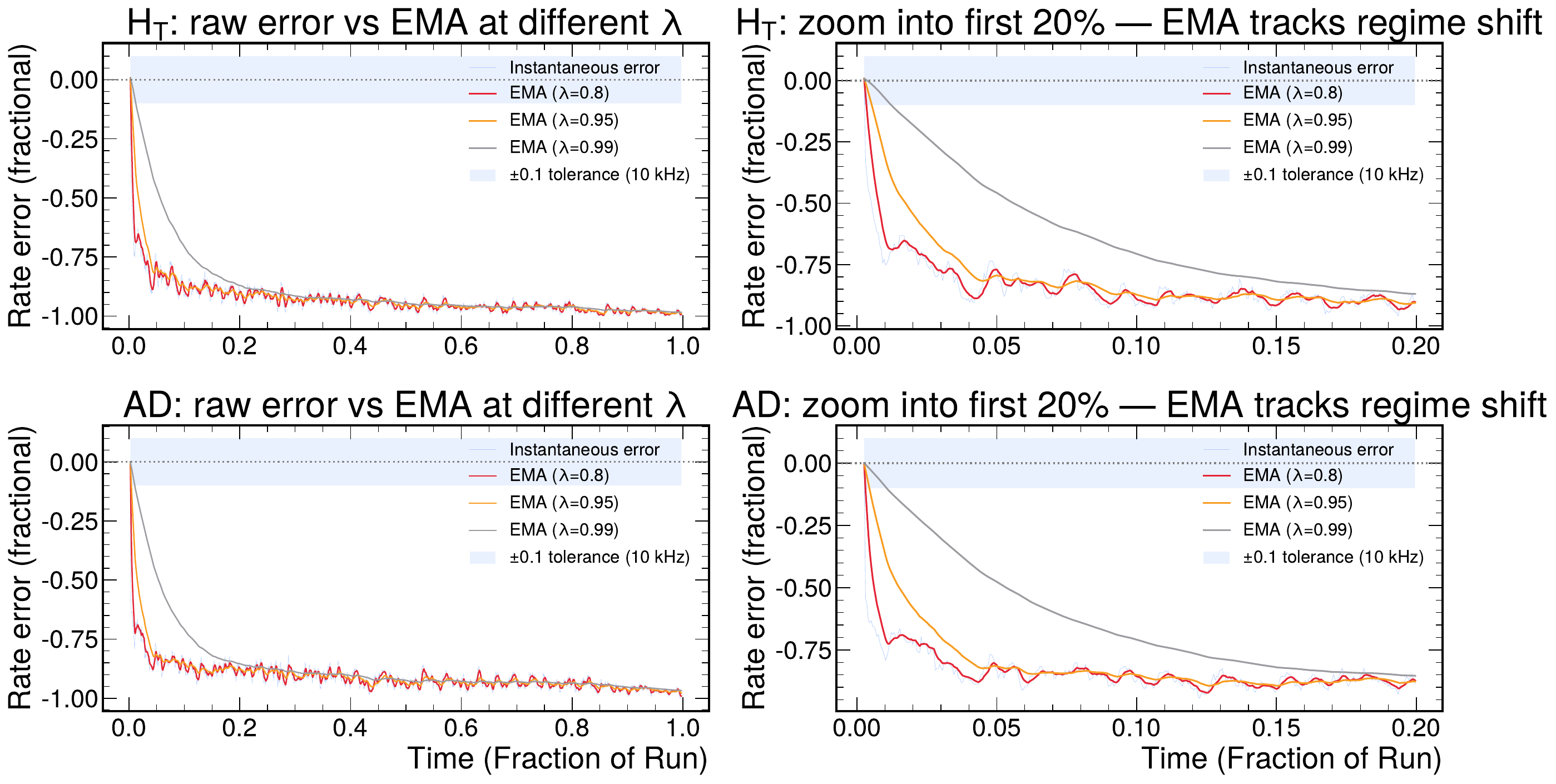}
    \caption{Exponential Moving Average (EMA) of the fractional rate error at different smoothing factors $\lambda$, computed over 50k-event sliding windows with a fixed threshold calibrated at the start of the run (MC). The light 
  blue band indicates the $\pm 10$ kHz tolerance around the $100$ kHz target rate.}
    \label{fig:ema_over_time}
\end{figure*}

\begin{figure*}
    \centering
    \includegraphics[width=1.03\linewidth]{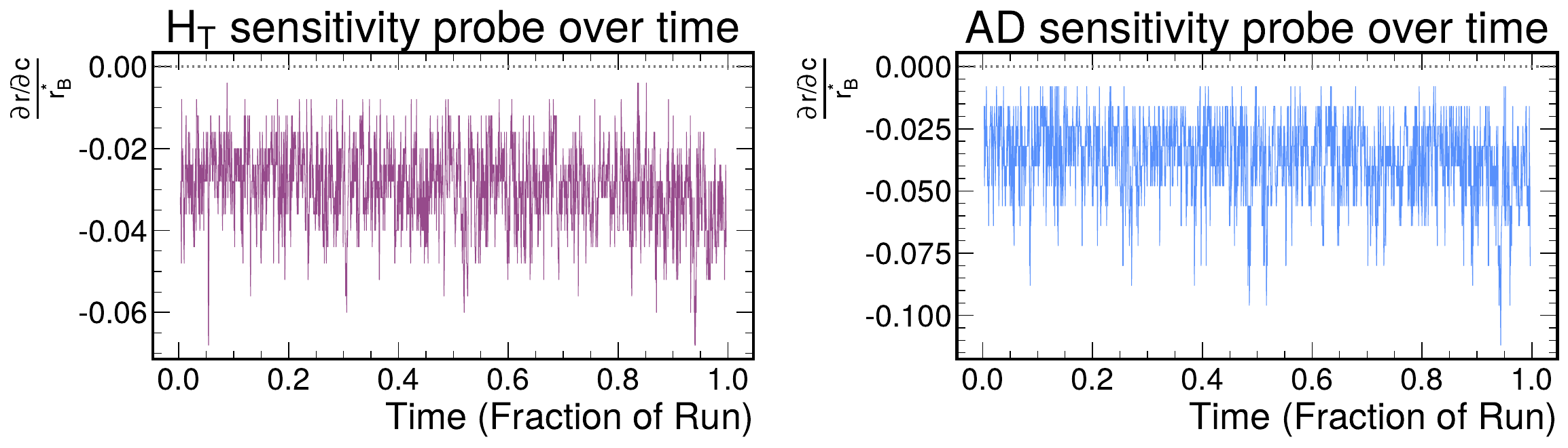}
    \caption{Sensitivity probe over time (MC).}
    \label{fig:sensitivity_probe_over_time}
\end{figure*}


\begin{figure*}
    \centering
    \includegraphics[width=\linewidth]{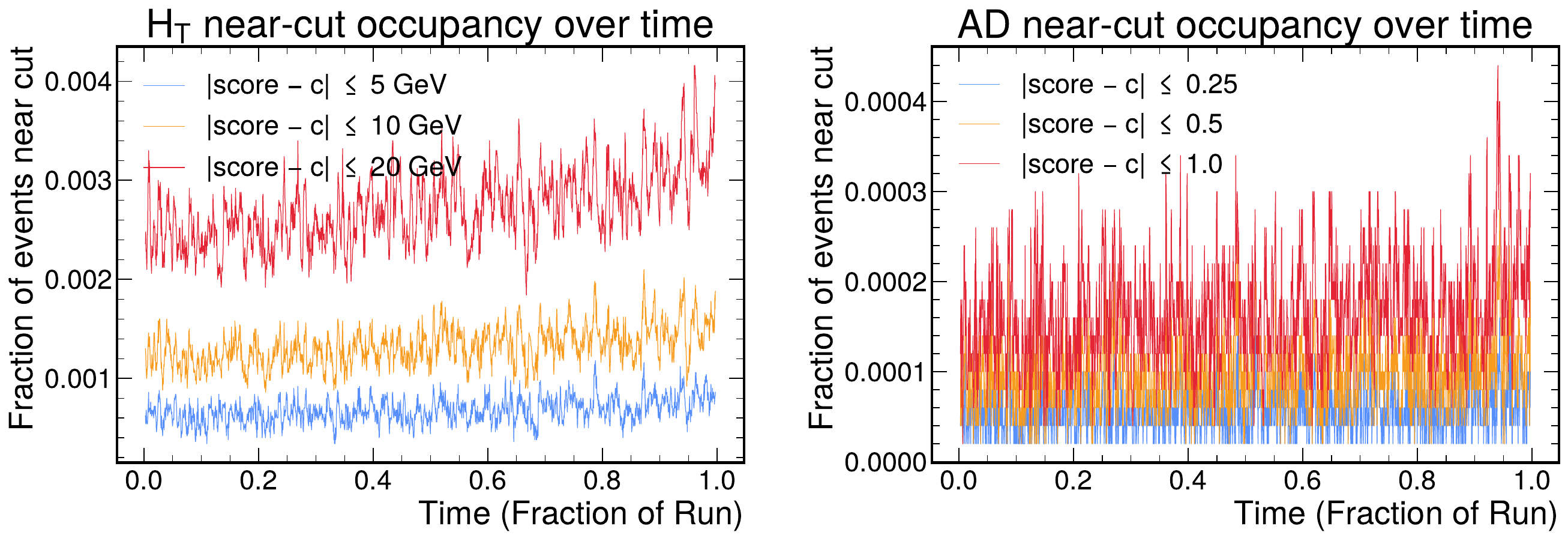}
    \caption{Near-threshold occupancy over time (MC). }
    \label{fig:near_cut}
\end{figure*}

\begin{algorithm*}[t]
\caption{Sequential state construction and recurrent controller for adaptive threshold at LHC}
\label{alg:seq_state_and_controller}
\KwIn{
Background events $\{(x_n,N_{pv,_n}\}_{n=1}^{N_t}$; current cut $c_t$; measured background rates $r_t, r_{t-1}$;
target rate $r_{B}^*$; tolerance $\tau$; window length $K$; calibration constants $(m,s)$ ($m{=}\mathrm{mid}$, $s{=}\mathrm{span}$); near-cut widths $\{w_j\}_{j=1}^J$ with $J{=}3$;
last action $a_{t-1}$, max step $a_{\max}$; tail step $\delta$; integrator $e^{(I)}_t$; sensitivity probe $\widehat{\partial r/\partial c}$.
}
\KwOut{Sequence observation $S_t \in \mathbb{R}^{K\times F}$ ($F{=}22$) and action $\Delta c_t \in \mathcal{A}$.}
\BlankLine
\textbf{Procedure} \textsc{BuildSequentialState}$(\{x\},\{\mathrm{N_{pv}}\}, c_t, r_t, r_{t-1}, r^*, \tau, K)$\;
\Indp
\tcp{1) Recent window extraction (most-recent-biased)}
$(x_{1:K}, \mathrm{N_{pv, 1:K}}) \leftarrow \textsc{DownsampleOrPadRecent}(\{x\},\{\mathrm{N_{pv}}\}, K)$\;
\BlankLine
\tcp{2) Per-event normalization}
\For{$k \leftarrow 1$ \KwTo $K$}{
  $x_k^{\mathrm{norm}} \leftarrow (x_k - m)/s$\;
}
$\mu_{\mathrm{N_{pv}}} \leftarrow \mathrm{Mean}(\mathrm{N_{pv, 1:K}})$;\quad
$\sigma_{\mathrm{N_{pv}}} \leftarrow \mathrm{Std}(\mathrm{N}_{PV, 1:K})$\;
\For{$k \leftarrow 1$ \KwTo $K$}{
  $\mathrm{N}_{PV,k}^{\mathrm{norm}} \leftarrow (\mathrm{N}_{pv,k} - \mu_{\mathrm{N_{pv}}})/(\sigma_{\mathrm{N_{PV}}}+\epsilon)$\;
}
\BlankLine
\tcp{3) Threshold-relative geometry (per event)}
$c^{\mathrm{norm}} \leftarrow (c_t - m)/s$\;
\For{$k \leftarrow 1$ \KwTo $K$}{
  $d_k^{\mathrm{norm}} \leftarrow (x_k - c_t)/s$;\quad
  $p_k \leftarrow \mathbbm{1}[x_k \ge c_t]$\;
}
\BlankLine
\tcp{4) Chunk-level rate and control features (broadcast)}
$e_t \leftarrow (r_t - r_{B}^*)/r_{B}^*$\;
$e^{\mathrm{abs}}_t \leftarrow |r_t - r_{B}^*|/r_{B}^*$\;
$\Delta r_t \leftarrow (r_t - r_{t-1})/r_{B}^*$\;
$b_t \leftarrow \mathbbm{1}[|r_t-r_{B}^*|\le\tau]$\;
\If{last action is available}{
  $\tilde a_{t-1} \leftarrow a_{t-1}/a_{\max}$\;
}
\Else{
  $\tilde a_{t-1} \leftarrow 0$\;
}
\For{$k \leftarrow 1$ \KwTo $K$}{
  $\mathrm{tpos}_k \leftarrow (k-1)/(K-1)$\;
}
\BlankLine
\tcp{5) Near-cut occupancy channels (per event)}
\For{$k \leftarrow 1$ \KwTo $K$}{
  \For{$j \leftarrow 1$ \KwTo $J$}{
    $n_{k,j} \leftarrow \mathbbm{1}[|x_k - c_t|\le w_j]$\;
  }
}
\BlankLine
\tcp{6) Tail-shape features (broadcast)}
$p_0 \leftarrow r(c_t)$;\quad
$p_1 \leftarrow r(c_t + \delta)$;\quad
$p_2 \leftarrow r(c_t + 2\delta)$\;
$\rho_1 \leftarrow p_1/p_0$;\quad
$\rho_2 \leftarrow p_2/p_1$\;
\BlankLine
\tcp{7) Concatenate into sequence tensor (F=22 features per event)}
\For{$k \leftarrow 1$ \KwTo $K$}{
  $S_t[k] \leftarrow \textsc{Concat}\big($
  $x_k^{\mathrm{norm}}, \mathrm{N}_{PV,k}^{\mathrm{norm}}, p_k, d_k^{\mathrm{norm}}, \mathrm{tpos}_k, n_{k,1}, n_{k,2}, n_{k,3},$\;
  \quad $e_t, e^{\mathrm{abs}}_t, \Delta r_t, b_t, c^{\mathrm{norm}}, \tilde a_{t-1}, \mu_{\mathrm{N_{PV}}}, \sigma_{\mathrm{N_{PV}}},$\;
  \quad $e^{(I)}_t, \widehat{\partial r/\partial c}, p_1, p_2, \rho_1, \rho_2 \big)$\;
}
\KwRet{$S_t$}\;
\Indm
\BlankLine
\textbf{Controller (sequence network).}\;
$h_t \leftarrow \mathrm{GRU}(S_t)$ \tcp*{encode the event sequence}
\If(\tcp*[f]{policy-based}){policy-based}{
  $\text{logits}_t \leftarrow \mathrm{MLP}(h_t)$\;
  $\Delta c_t \sim \pi(\cdot\mid \text{logits}_t)$\tcp*[f]{sample or argmax}\;
}
\Else(\tcp*[f]{value-based}){
  $Q_t \leftarrow \mathrm{MLP}(h_t)$\;
  $\Delta c_t \leftarrow \arg\max_{a\in\mathcal{A}} Q_t[a]$\;
}
\KwRet{$\Delta \threshold_t$}\;
\end{algorithm*}

%% file: Appendix/Appendix.tex
\section{Additional Experimental Results}
\subsection{Additional DQN results}
\label{Appendix:dqn_additional}
Here, in this section, we want to demonstrate even with off-policy RL that cannot handle distribution shifts well, simple RL policy DQN shall boost signal efficiency over the trajectory compared to fixed predefined PID loop. 

\paragraph{DQN Background-rate trajectories.} For clarity, we plot out DQN and DQN-F qualitative results for per-timestep trajectories to demonstrate simple DQN can boost in band rate visually.
Figure~\ref{fig:bkg_rate_pid} traces the per-chunk background 
trigger rate over a simulated MC run for the $H_T$ trigger and the 
anomaly-detection (AD) trigger. The Constant controller is set at 
the start of the run and never updates; pileup drift pulls its 
rate steadily downward, the canonical failure mode of static 
thresholds under non-stationary input rates. PID reacts to the 
rate signal but oscillates around the tolerance band, repeatedly 
overshooting and undershooting. DQN tracks the band more 
consistently, and DQN-F (the feasibility-filtered variant) 
further suppresses out-of-band swings and recovers faster from 
pileup transitions. The two panels show qualitatively the same 
pattern, indicating that the controller-level differences are not 
specific to the choice of physics observable.

\begin{figure*}[t]
    \centering
    \begin{subfigure}{0.485\textwidth}
        \includegraphics[width=\linewidth]{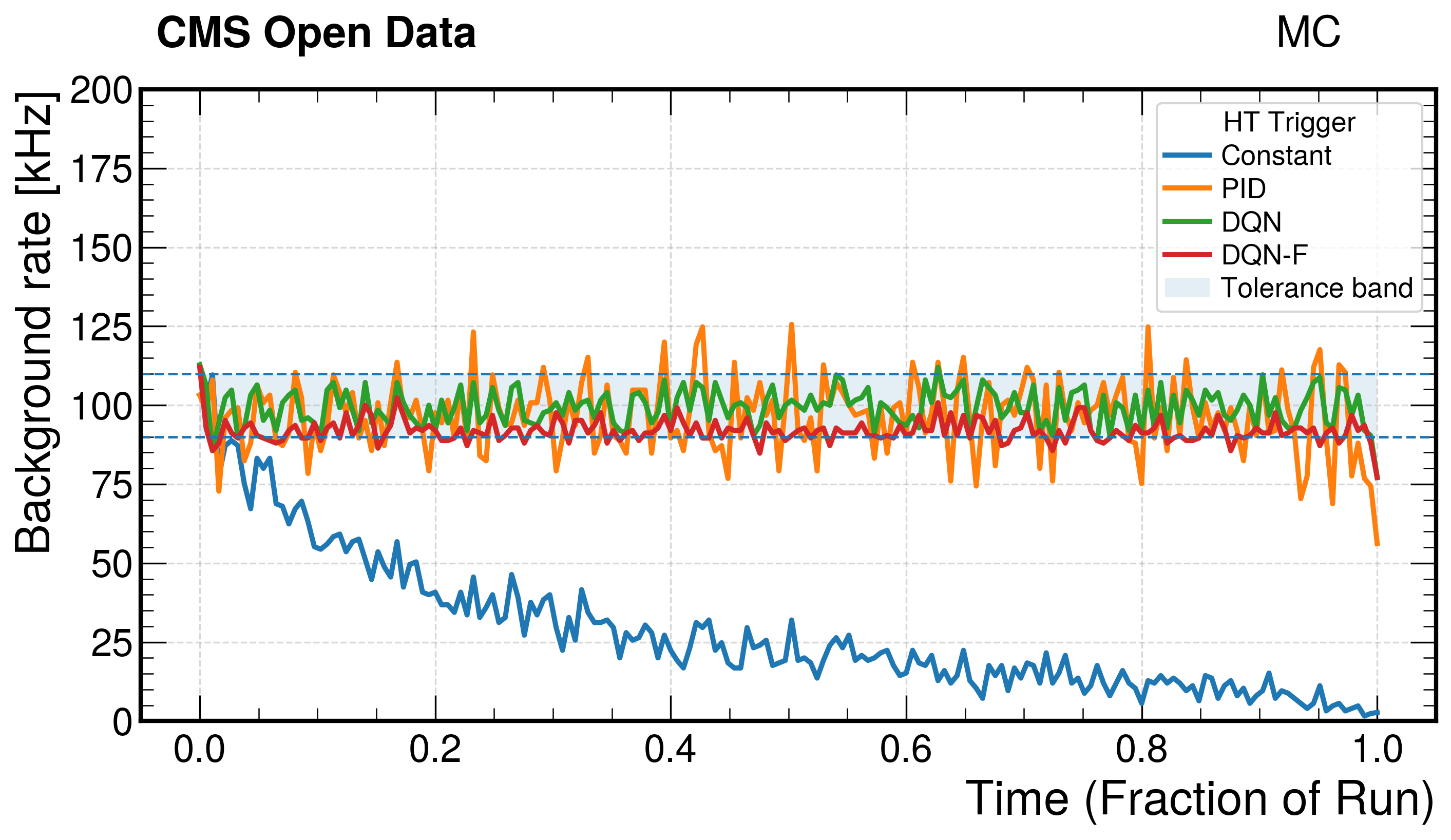}
        \caption{$H_{T}$ trigger}
        \label{fig:bkg_rate_pid_a}
    \end{subfigure}
    \begin{subfigure}{0.485\textwidth}
        \includegraphics[width=\linewidth]{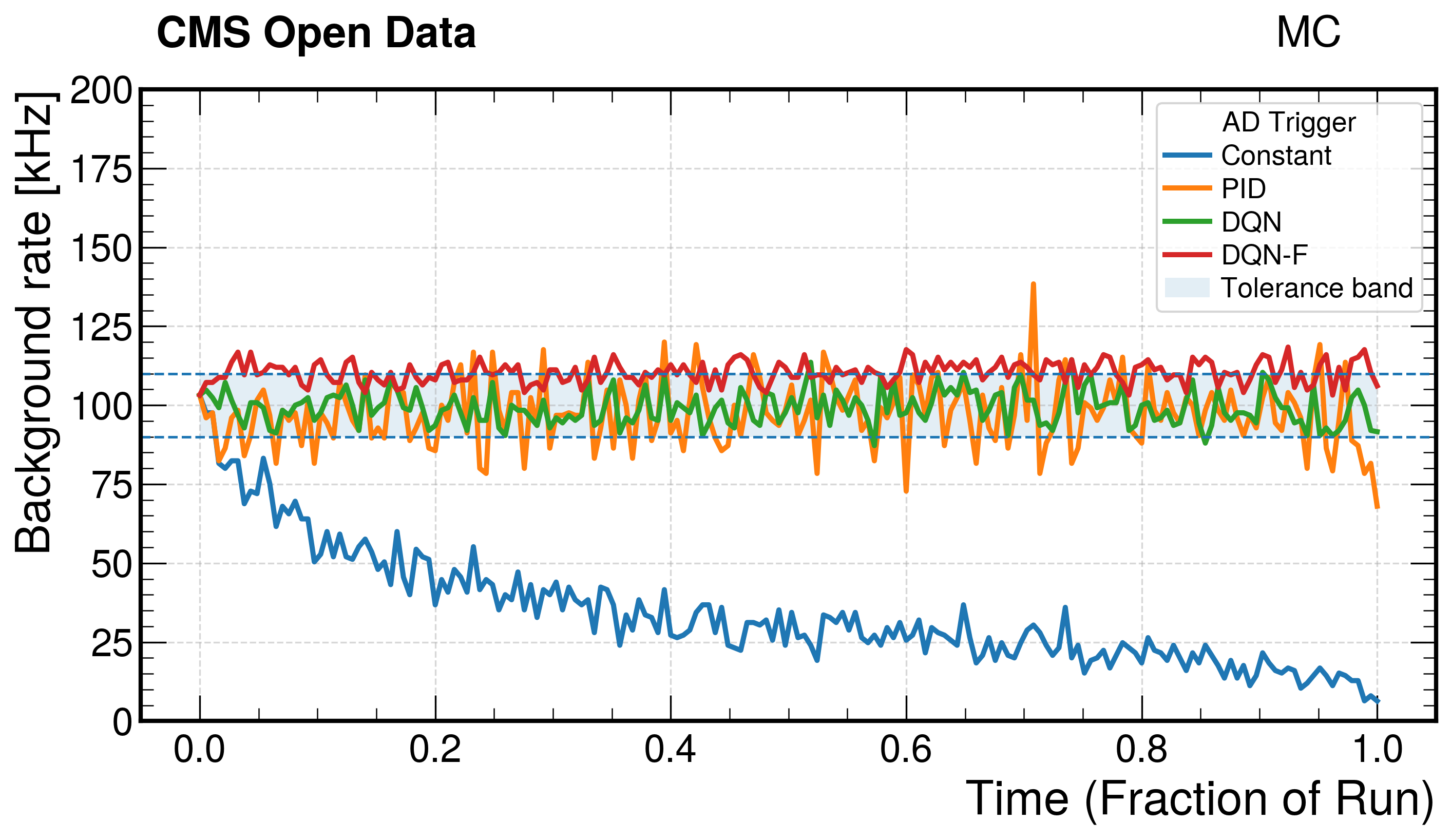}
        \caption{ AD   trigger}
        \label{fig:bkg_rate_pid_b}
    \end{subfigure}
    \caption{
        Background trigger rates under DQN for 
        (a) \HT\ and
        (b)  AD   trigger. 
        DQN stays in-band more frequently, reduces overshoot or undershoot and adapts faster in response to pileup changes.
    }
    \label{fig:bkg_rate_pid}
\end{figure*}

\subsection{GRPO MC training}
\label{appendix:mc_training}
Beyond vanilla DQN, we evaluate the GRPO family (GRPO, GFPO-F, 
GFPO-FR). Figures~\ref{fig:bkg_rate_5_methods_histogram} 
and~\ref{fig:bkg_rate_5_methods_main_paper} show the background-rate 
distribution and per-timestep trajectory for each controller on MC. 
GRPO, GFPO-F, and GFPO-FR concentrate near the $100$~kHz target in 
both views; Constant decays as pileup drifts, and PID oscillates 
with frequent excursions outside the band.

\begin{figure}[htbp]
    \centering
    \begin{subfigure}{0.485\textwidth}
        \includegraphics[width=\linewidth]{figures/singletriggger_five_methods_backup/mc_rate_histogram_HT.pdf}
        \caption{$H_{T}$ trigger}
        \label{fig:bkg_rate_five_histogram_ht}
    \end{subfigure}
    \begin{subfigure}{0.485\textwidth}
        \includegraphics[width=\linewidth]{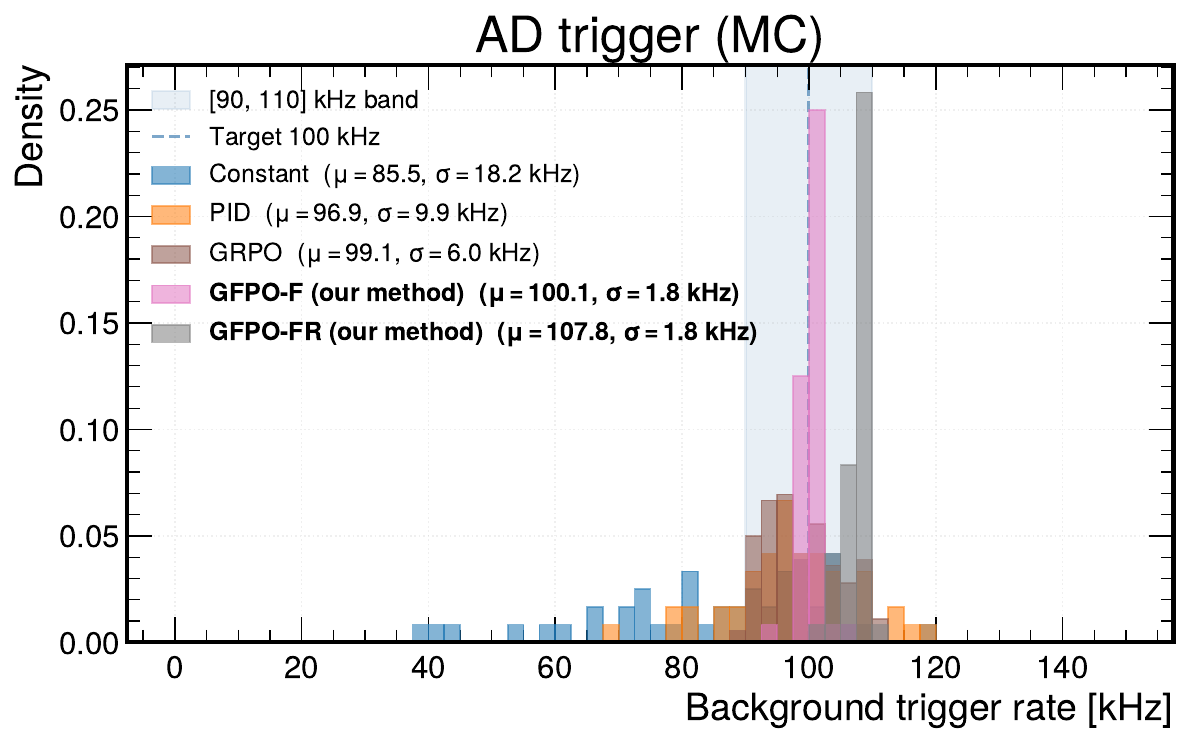}
        \caption{ AD   trigger}
        \label{fig:bkg_rate_five_histogram_ad}
    \end{subfigure}
    \caption{
        Background rates distribution for (a) \HT\ trigger, (b)  AD   trigger. Our methods (GFPO-F and GFPO-FR) have the highest in-band rate of all established baselines. \textbf{GFPO-FR (106.1 kHz) operates nearer the band edge for signal compared to GFPO (100.0 kHz).}
    }
    \label{fig:bkg_rate_5_methods_histogram}
\end{figure}

\begin{figure}[htbp]
    \centering
    \begin{subfigure}{0.485\textwidth}
        \includegraphics[width=\linewidth]{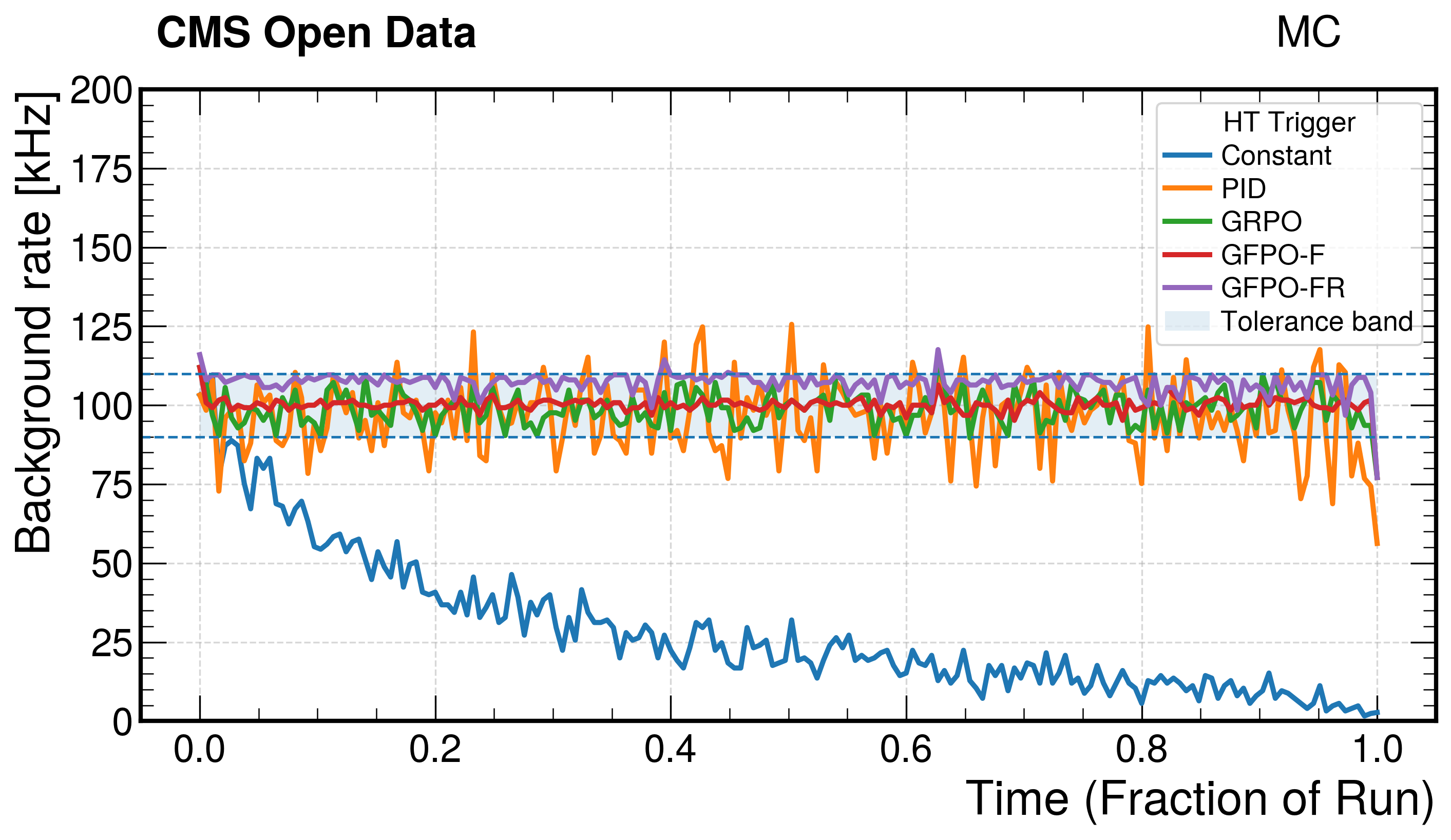}
        \caption{$H_{T}$ trigger}
        \label{fig:bkg_rate_five_methods_ht}
    \end{subfigure}
    \begin{subfigure}{0.485\textwidth}
        \includegraphics[width=\linewidth]{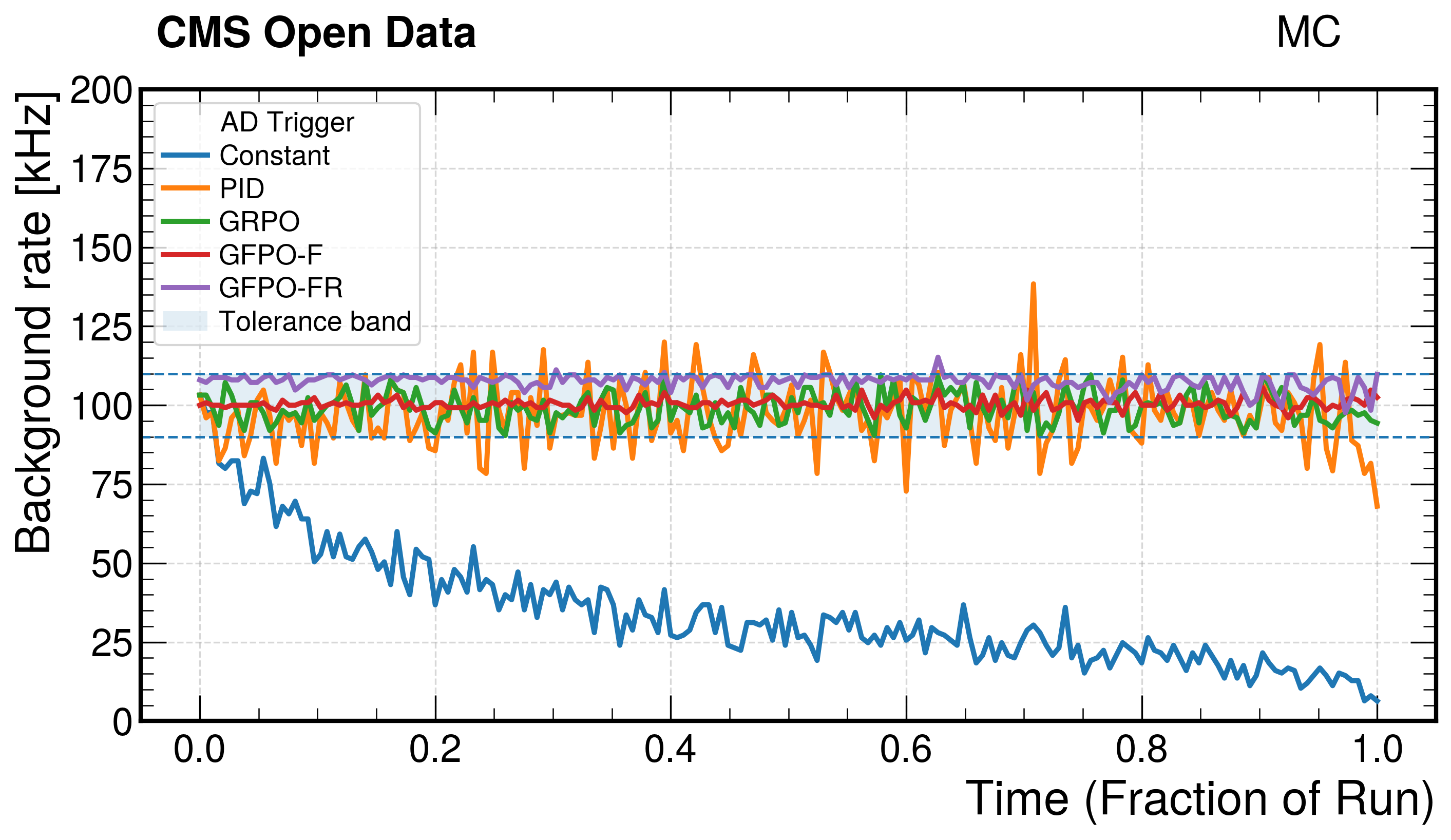}
        \caption{ AD   trigger}
        \label{fig:bkg_rate_five_methods_ad}
    \end{subfigure}
    \caption{
        Background rates under Constant Menu, PID loop, GRPO, GFPO-F and GFPO-FR (MC) for (a) \HT\ trigger, (b)  AD   trigger. RL-based agents stay within the band for most times compared with non-learning based methods. 
    }
    \label{fig:bkg_rate_5_methods_main_paper}
\end{figure}

\paragraph{Signal efficiency over time on MC.} \citet{emami2026selfdrivingtriggerlhcadaptive} showcases that adaptive thresholding recovers signal efficiency as a run progresses, while a
static menu steadily loses it. We move one step further and demonstrate RL methods have better signal efficiency acceptance over PID baselines. Figure~\ref{fig:signal_eff_over_time_5_methods_appendix}
plots per-chunk signal efficiency over the held-out 20\% MC run for the $H_T$ and
AD triggers, now including the two non-learning baselines, Constant and PID,
alongside the adaptive policies (3-seed average, centered window $W{=}5$; at
$W{=}\text{full}$ the curves reproduce the $\epsilon_{\mathrm{ov}}$ entries of
Table~\ref{tab:single_trigger_summary_compact}). The Constant threshold decays on
both triggers and both signals, since it cannot track the drifting held-out
distribution. The adaptive policies instead climb, re-centering the operating
point to hold the background rate on target while recovering signal. The effect is
largest on the rare, harder $h\to4b$ signal: $t\bar{t}$ efficiency is already
near-saturated on $H_T$, whereas the $h\to4b$ panels open a wide margin, with \textsc{GFPO-FR}
reaching roughly 48\% ($H_T$) and 40\% (AD) in the late chunks against roughly
24\% for Constant. The ordering
$\textsc{GFPO-FR}\gtrsim\textsc{GFPO-F}\approx\textsc{GRPO}>\textsc{PID}>\text{Constant}$
holds throughout, and efficiency rises gradually rather than in a single jump,
indicating \emph{sustained} and \emph{smooth} adaptation to the drift rather than a one-time correction.

\begin{figure}[htbp]
    \centering
    \begin{subfigure}{1.0\textwidth}
        \includegraphics[width=\linewidth]{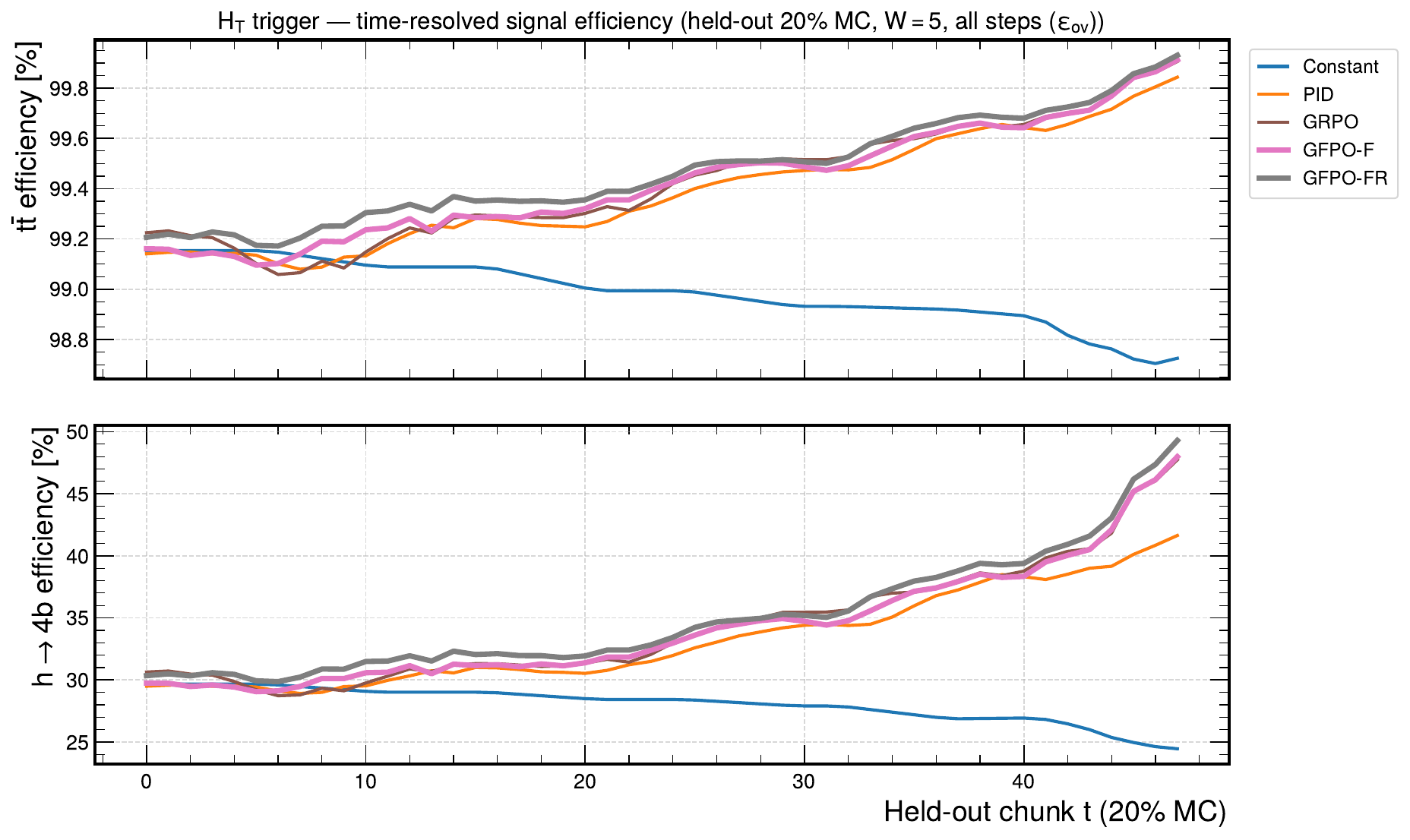}
        \caption{$H_{T}$ trigger}
        \label{fig:signal_eff_over_time_five_methods_ht}
    \end{subfigure}
    \begin{subfigure}{1.0\textwidth}
        \includegraphics[width=\linewidth]{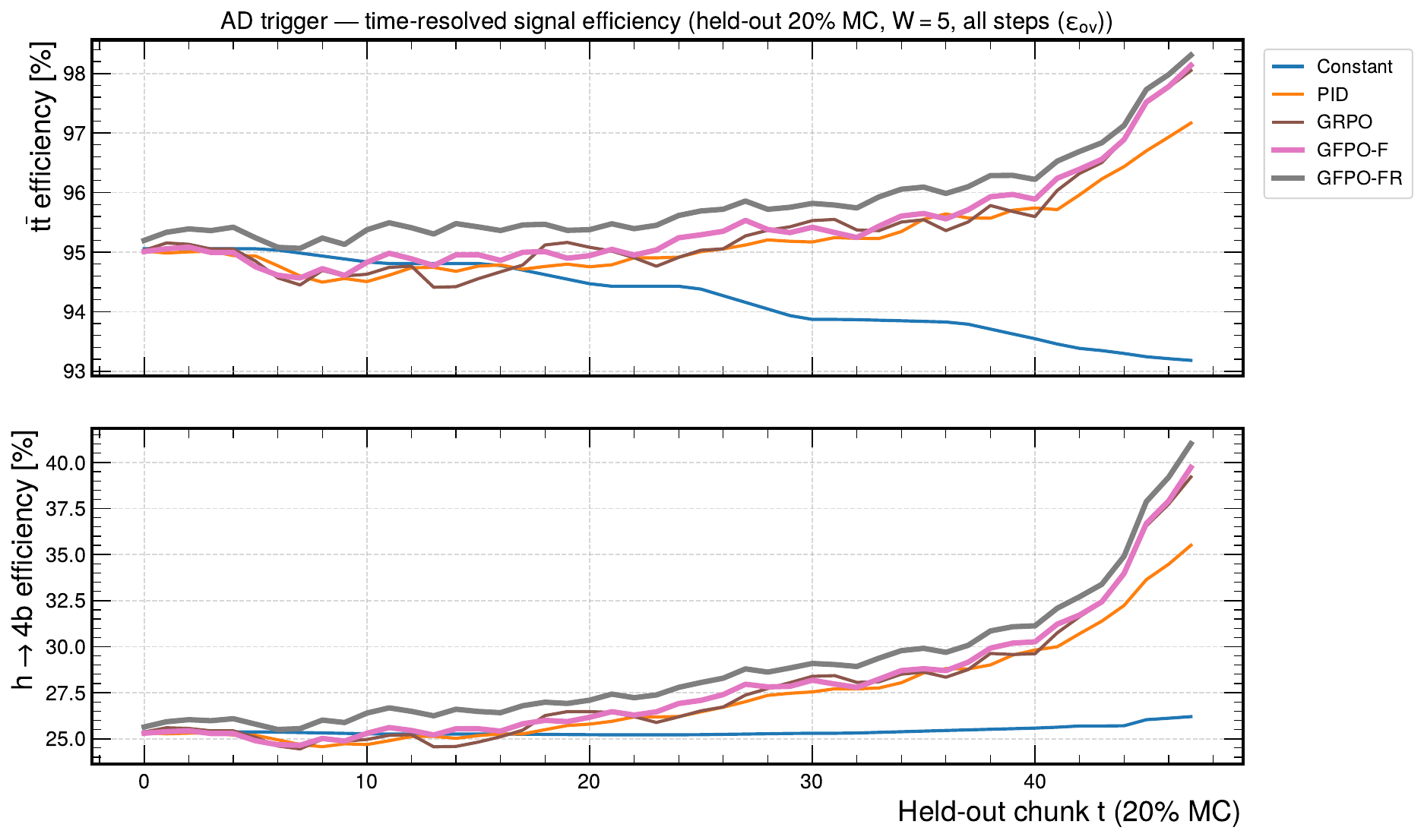}
        \caption{ AD   trigger}
        \label{fig:signal_eff_over_time_five_methods_ad}
    \end{subfigure}
    \caption{
        Signal efficiency over time under Constant Menu, PID loop, GRPO, GFPO-F and GFPO-FR (MC) for (a) \HT\ trigger, (b)  AD   trigger. RL-based agents have higher signal efficiency for both \ttbar and \haaFourB over time than PID and constant menus.
    }
    \label{fig:signal_eff_over_time_5_methods_appendix}
\end{figure}

\subsection{Sim-to-Real Transfer and Test Time Adaptation}
\label{appendix:real_data_adaptation}

Figure~\ref{fig:rate_sim_to_real} shows background rates for all methods under 
test-time adaptation on CMS Run 283408, where models are trained on MC samples 
and deployed on real CMS data. Notably, the results are nearly \textit{identical} to those in 
Figure~\ref{fig:realdata_test_time_training}, suggesting that test-time adaptation 
achieves performance on par with test-time training. Therefore, we answer the question 
raised in the main paper: test-time adaptation, training on MC simulations and 
deploying on real CMS collision data, is a preferable and practical alternative 
to test-time training.

\begin{figure*}[t]
    \centering
    \begin{subfigure}{0.8\textwidth}
        \includegraphics[width=\linewidth]{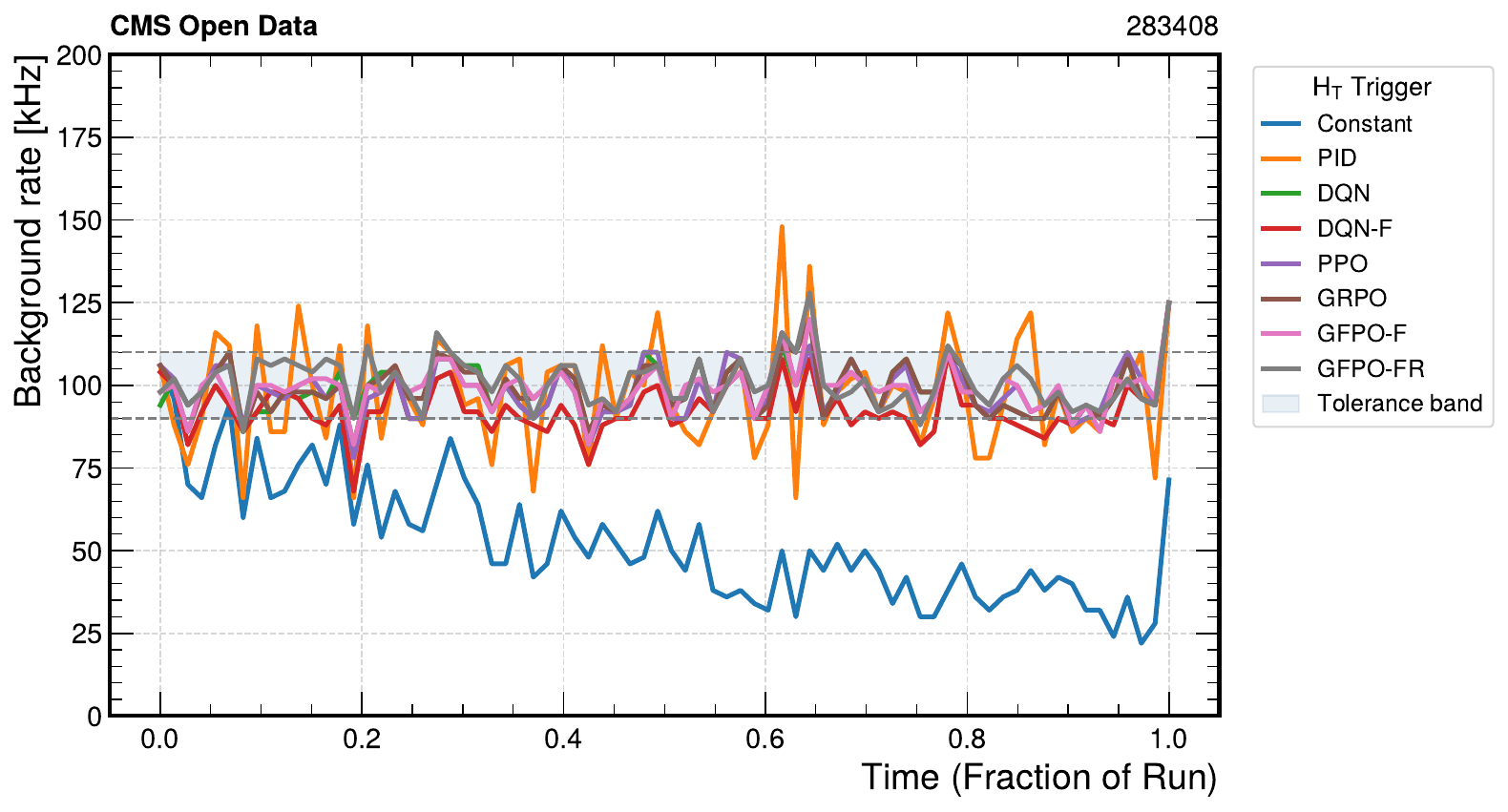}
        \caption{$H_{T}$ trigger}
        \label{fig:rate_ht_sim_to_real}
    \end{subfigure}
    \qquad
    \begin{subfigure}{0.8\textwidth}
        \includegraphics[width=\linewidth]{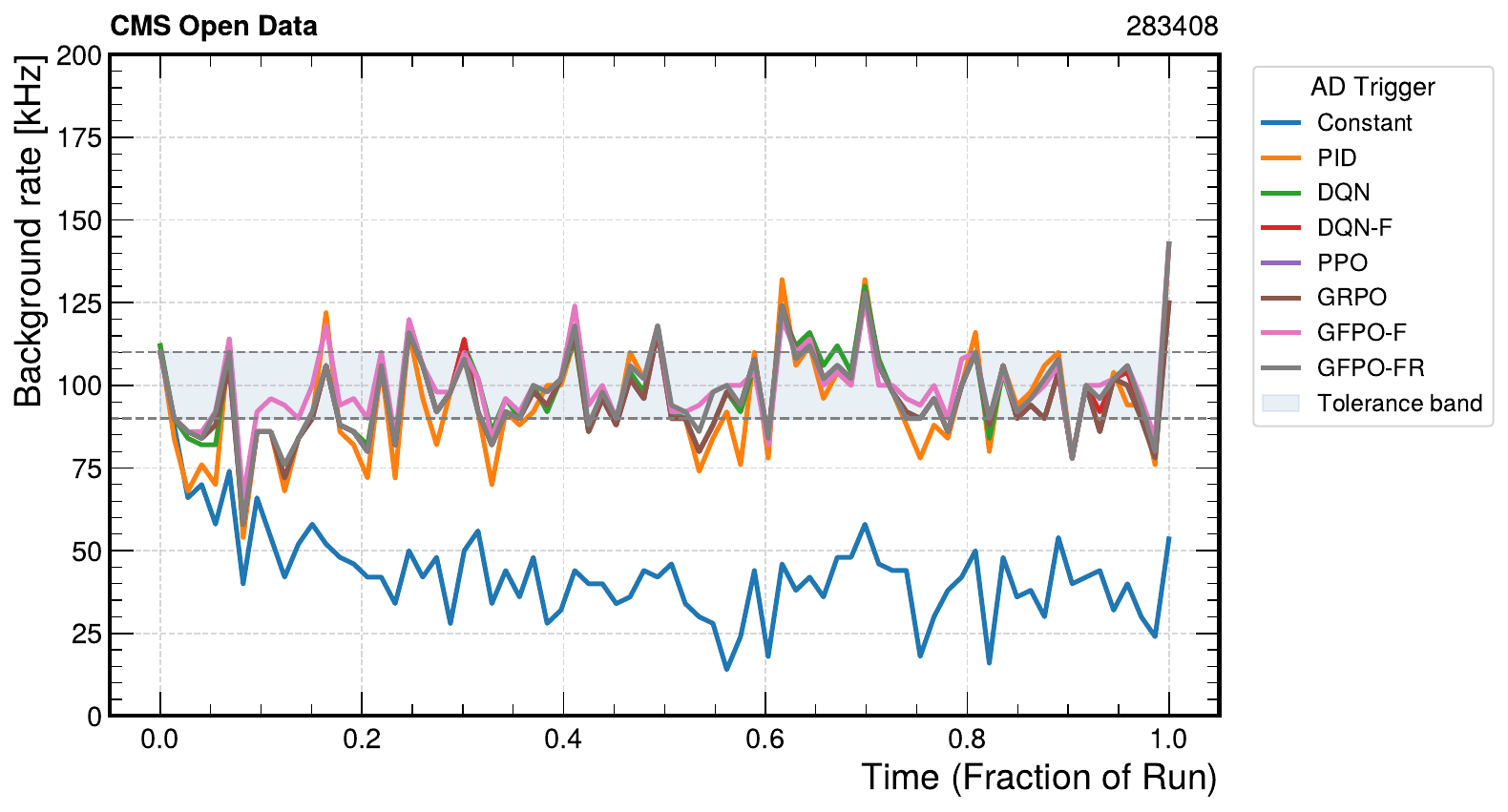}
        \caption{ AD   trigger}
        \label{fig:rate_ad_sim_to_real}
    \end{subfigure}
    \caption{
        \textbf{Background trigger rates under MC$\to$CMS transfer (Table~\ref{tab:single_trigger_summary_compact_realdata}).} 
        Per-step rates on CMS Run~283408 with policies trained on MC and frozen at deployment (no online updates), for the (a) $H_T$ trigger and (b) AD trigger.
    }
    \label{fig:rate_sim_to_real}
\end{figure*}

\subsection{Testing Time Training on Real CMS Data}
\label{appendix:real_data}
We include a study on test-time training over real CMS collision data to evaluate whether adaptive thresholding remains effective beyond the controlled assumptions of simulation. While Monte Carlo studies are useful for development, real detector streams can exhibit mismatches in feature distributions, background composition, detector response, and running conditions that are difficult to model perfectly (cf. Figure~\ref{fig:score_summary} and Figure~\ref{fig:score_summary_realdata})~\citep{gavrikov2025dinamo}. This concern is especially pronounced for anomaly-detection triggers, whose decisions are made through a threshold on a learned score distribution. Since CMS has already demonstrated anomaly-detection algorithms in realistic Level-1 trigger settings on live data~\citep{gandrakota2024real, CMS-DP-2024-059, zipper2024testing}, it is important to test whether an RL agent can adapt thresholds online when exposed to the variability of real data rather than only simulated samples. 

\paragraph{GFPO-F and GFPO-FR remain strongly competitive on real CMS collision data, achieving the best overall stability–efficiency tradeoff for both trigger settings.
}We present online training per time chunk in Table~\ref{tab:single_trigger_summary_compact_realdata_online}. Compared with PID, DQN, PPO, ADT, and GRPO, our proposed GFPO variants are more consistently strong across both control and physics metrics on real CMS data. Classical PID can still be competitive on isolated metrics, but its performance is less balanced overall. In contrast, GFPO-F and GFPO-FR repeatedly appear at or near the top across in-band occupancy and signal efficiency, indicating better robustness and adaptivity to real-time collision data variability. Figure~\ref{fig:realdata_test_time_training} presents the trajectory of background rates for all methods under test time training on CMS Run 283408 (i.e., directly learning an RL policy on real collision data). Remarkably, the frozen MC-trained rollout (Table~\ref{tab:single_trigger_summary_compact}) and the test-time training (Table~\ref{tab:single_trigger_summary_compact_realdata_online} and Figure~\ref{fig:realdata_test_time_training}) achieve nearly \textit{identical} performance on CMS data, InBand rate, and signal efficiency differ by less than 2 percentage points across all methods and triggers. This demonstrates that the online, per-chunk threshold adaptation mechanism (adjusting the cut $\threshold_t$ at each micro-step) is the primary driver of domain transfer, not continued weight updates. The policy network learned on MC generalizes well enough such that the action-level adaptation (selecting $\Delta \threshold$ to stay within the background rate tolerance band) compensates for the MC to CMS distribution shift without requiring gradient-based fine-tuning at deployment time.

\subsection{CPO degeneracy on the CMS dataset}
\label{app:cpo-degeneracy}

The CPO entries in Tables~\ref{tab:single_trigger_summary_compact_realdata}
and~\ref{tab:single_trigger_summary_compact_realdata_online} agree to
floating-point precision, with \emph{zero} per-seed standard deviation. This is
structural rather than a seeding artifact: the three trained checkpoints differ
in $\pi_\theta$, and every other RL baseline on the same protocol shows nonzero seed
spread. The cause is that CPO's executed trajectory is screened from the policy.

\paragraph{Setup.} At each micro-step CPO samples $G{=}16$ candidate actions
$a_k\sim\pi_\theta(\cdot\mid s)$ from a discrete set
$\mathcal{A}=\{-3,-1.5,0,1.5,3\}$ (AD) or $\{-2,-1,0,1,2\}$ ($H_T$) of action space size $|\mathcal{A}|=5$, evaluates each candidate's reward $R_k$ and cost $c_k$ in the
simulator (Eqs.~\ref{eq:cpo-reward}--\ref{eq:cpo-cost}), and executes
\begin{equation}
k^\star=
\begin{cases}
\arg\max_{k\,:\,c_k\le 10^{-9}}\, R_k & \text{if some } c_k\le 10^{-9},\\
\arg\min_{k}\, c_k                   & \text{otherwise},
\end{cases}
\label{eq:cpo-execute-app}
\end{equation}
followed by a safety-shield override $\delta=\pm\Delta_{\max}$ whenever
$|\bar r_{\text{bg}}-r_B^\star|>\tau$. We partition micro-steps by rate state.

\paragraph{The executed action is independent of $\pi_\theta$.}
\emph{Out of band} ($|\bar r_{\text{bg}}-r_B^\star|>\tau$), the shield overrides
the selection with $\delta=\pm\Delta_{\max}$, a function of
$(\bar r_{\text{bg}}, r_B^\star, \tau, \Delta_{\max})$ alone, independent of the
sampled group and of $\pi_\theta$. \emph{In band}, both selection branches of
Eq.~\ref{eq:cpo-execute-app} rank candidates by their values $(R_k, c_k)$, and
$R_k$ and $c_k$ are functions of the candidate cut $\delta_k$ and the data
through the simulator, not of $\pi_\theta$. (Because the in-band cost is $c=e^2$ with
$e=|\bar r_{\text{bg}}-r_B^\star|/\tau$, the tolerance $c_k\le 10^{-9}$ demands
$e_k\approx 0$, which discrete cuts essentially never meet, so the otherwise-branch
$k^\star=\arg\min_k c_k$ is the operative one.) The policy enters only by determining which $\delta$ appear in
the sample. The operative $\delta$ is absent with probability at most
$(4/5)^{16}\approx 0.03$ under a uniform policy, and far less once the trained
policy concentrates its mass, so the in-sample optimizer coincides with the
full-action-set optimizer. The executed $\delta$ therefore depends only on the
data and the current cut, and the current cut is itself a function of past
executed actions, so the rollout is a deterministic function of the data. We
confirm this directly: the per-micro-step executed-$\delta$ sequence is
bit-identical across the three CMS seeds despite the differing $\theta$.

\paragraph{Test-time training does not help.} 
Under test-time training the recovery step~\citep[Eq.~14]{achiam2017constrained}
dominates, because the batch estimate exceeds the budget, $\hat J_C > d$
(equivalently $\hat c > 0$, Eq.~\ref{eq:constraint}), on roughly half of the CMS update chunks (consistent with the $\sim\!0.54$ in-band fraction for AD trigger in
Table~\ref{tab:single_trigger_summary_compact_realdata}). The recovery branch
fires when $C>0$ and the trust region cannot absorb the violation,
$2\delta_{\mathrm{KL}}-c^2/s<0$, and its direction $\theta'-\theta\propto
-H^{-1}b$ moves $\theta$ purely along the cost gradient. But neither the recovery
step nor an ordinary update alters the $(\delta_k, R_k, c_k)$ values that
Eq.~\eqref{eq:cpo-execute-app} consults, and the selection screens $\theta$ as
above. The executed trajectory, and every metric derived from it, is therefore
identical to the frozen rollout.

\begin{figure*}[t]
    \centering
    \begin{subfigure}{0.8\textwidth}
        \includegraphics[width=\linewidth]{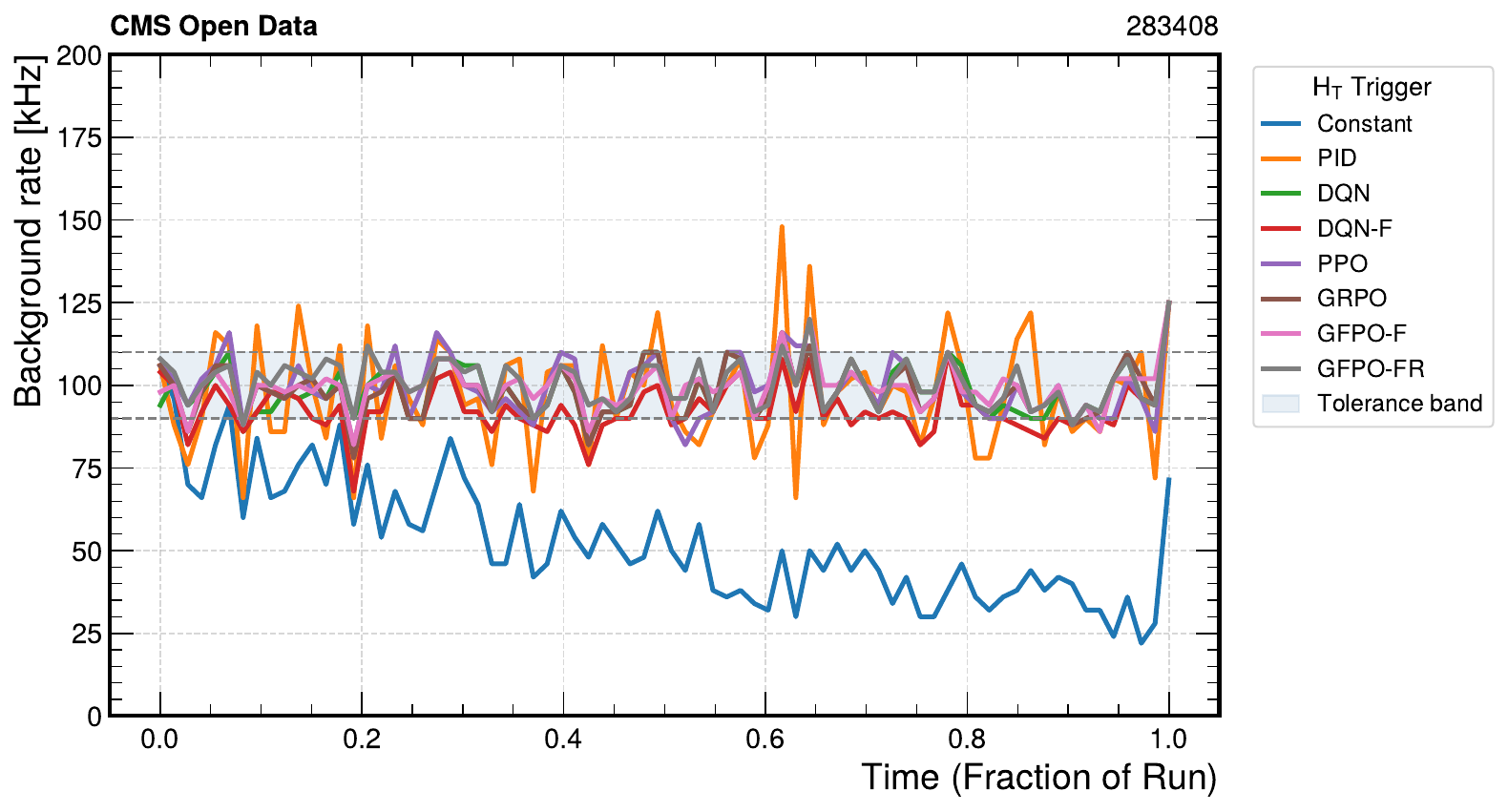}
        \caption{$H_{T}$ trigger}
    \end{subfigure}
    \qquad
    \begin{subfigure}{0.8\textwidth}
        \includegraphics[width=\linewidth]{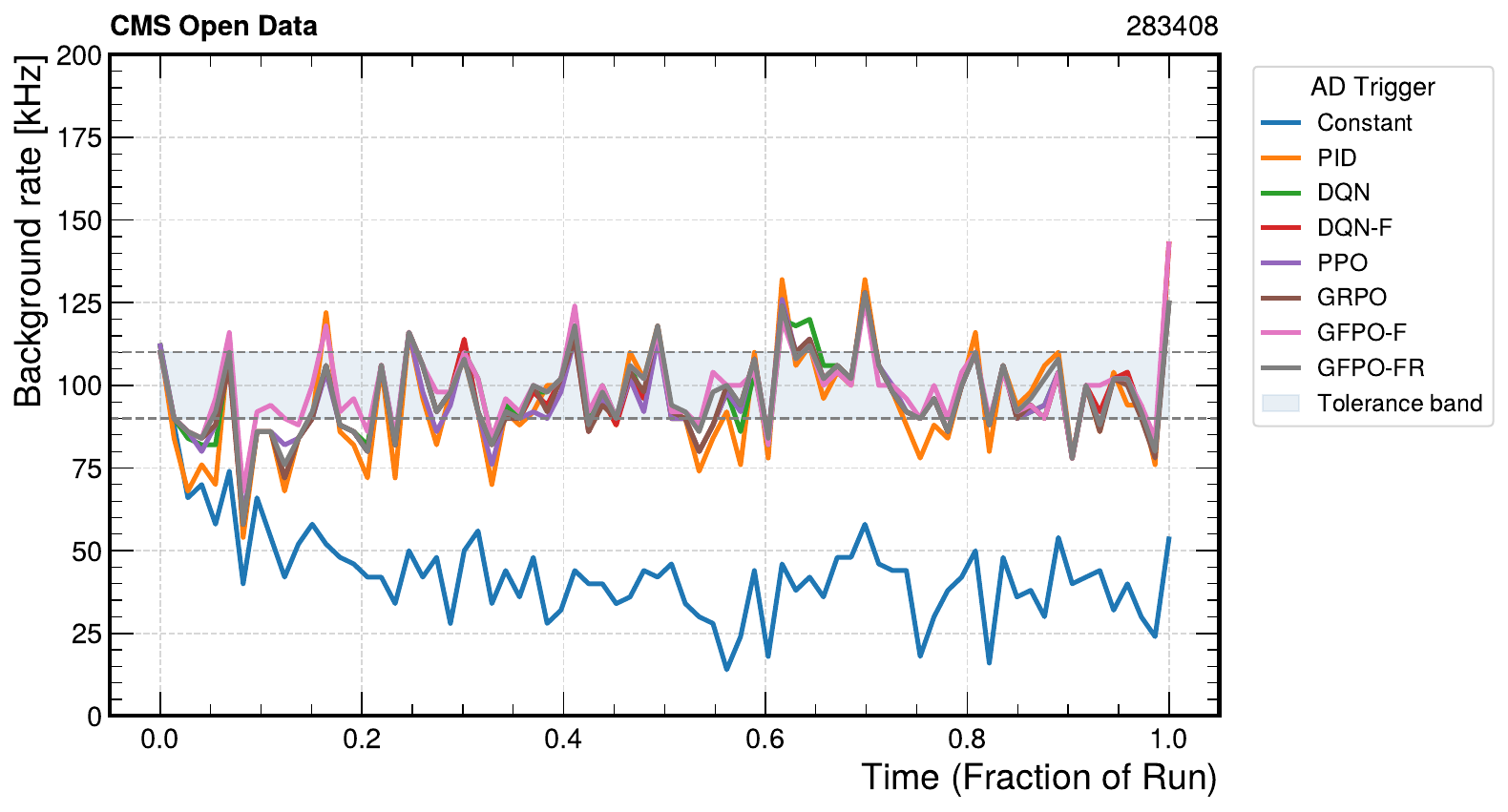}
        \caption{AD trigger}
    \end{subfigure}
    \caption{
        \textbf{Background trigger rates under MC$\to$CMS transfer 
        with test-time training (Table~\ref{tab:single_trigger_summary_compact_realdata_online}).} Per-step rates on CMS Run~283408 with 
        policies trained on MC and updated online during the run on real collision data (one 
        gradient step per chunk), for the (a) $H_T$ trigger and (b) AD trigger.
    }
    \label{fig:realdata_test_time_training}
\end{figure*}

\begin{table*}[t]
\centering
\scriptsize
\setlength{\tabcolsep}{3.3pt} 
\renewcommand{\arraystretch}{1.15}
\sisetup{
  detect-weight=true,
  detect-family=true,
  table-number-alignment=center,
  round-mode=places,
  round-precision=3
}
\caption{\textbf{Single-trigger control on CMS Run 283408 (test-time training).} The setup is the same as Table~\ref{tab:single_trigger_summary_compact} with same CMS dataset, except each RL policy is continually fine-tuned on streaming time chunks rather than frozen at deployment (test-time policy updates). Bold and underline mark the best and second-best per column; standard deviations are shown in parentheses. 
}
\label{tab:single_trigger_summary_compact_realdata_online}
\begin{tabular}{llccccccc
               }
\toprule
Trigger & Method &
{MAE$\downarrow$} & {P95$|e|$ $\downarrow$} &
{InBand $\uparrow$} & 
{$\epsilon_{\text{overall}}^{\ttbarraw} \uparrow$} & {$\epsilon_{\text{overall}}^{\haaFourB} \uparrow$} &
{$\epsilon_{\text{inband}}^{\ttbarraw} \uparrow$} & {$\epsilon_{\text{inband}}^{\haaFourB} \uparrow$} \\
\midrule
\multicolumn{9}{l}{\textbf{$H_{T}$ trigger}}\\
$H_{T}$ & Constant & 0.118(0)  & 0.175(0) & 0.041(0) & 91.325(0) & 22.365(0) & 95.842(0) & 15.198(0) \\
$H_{T}$ & PID \SEpaper      & \textbf{0.003}(0) & 0.085(0) & 0.432(0) & 97.381(0) & \underline{33.347(0)} & 97.50(0) & \textbf{35.242}(0) \\
$H_{T}$ & DSPOT \citep{siffer2017anomaly}  & 0.080(0) & 0.157(0) & 0.176(0) & 93.998(0) & 28.200(0)  & 96.910(0) & 33.278(0) \\
$H_{T}$ & DQN \citep{mnih2015human}     & 0.018(0) & 0.037(0.003) & 0.824(0.01) & 97.510(0.09) & 33.212(0.10) & 97.420(0.09) & 33.40(0.19) \\
$H_{T}$ & DQN-F     & 0.118(0) & 0.175(0) & 0.04(0) & 91.33(0) & 22.36(0) & 95.84(0) & 15.20(0) \\
$H_{T}$ & PPO \citep{schulman2017proximal} & 0.016(0.004) & 0.035(0.012) & 0.824(0.14) & 97.491(0.19) & 33.26(0.21) & 97.413(0) & 32.32(1.01) \\
$H_{T}$ & ADT \citep{yang2024adt}    &
0.018(0.002) & 0.042(0.011) &
0.811(0.04) & 97.460(0.10) & 33.26(0.11) & 97.334(0.02) & 33.250(0.42)
\\
$H_{T}$ & GRPO \citep{shao2024deepseekmath}     &
0.016(0.002) & 0.037(0.008) &
0.824(0.05) & 97.491(0.09) & 33.222(0.10) & 97.358(0.14) & 32.940(0.83)
\\
$H_{T}$ & L-GRPO & 0.016(0.001) & 0.033(0) & 0.94(0.01) & 97.486(0.04) & 33.216(0.08) & 97.341(0.04) & 32.533(0.15) \\
$H_{T}$ & CPO \citep{achiam2017constrained} & 0.014(0) & 0.025(0) & 0.905(0) & 97.498(0) & 33.139(0) & 97.392(0) & 32.901(0) \\
\rowcolor{highlight}
$H_{T}$ & GFPO-F (Our method)  
&
\underline{0.007(0)} & \textbf{0.021(0.001)} &
\textbf{0.98(0.01)} & \underline{97.61(0)} & 33.32(0) & \underline{97.57(0.01)} & 33.16(0.07) \\
\rowcolor{highlightDark}
$H_{T}$ & GFPO-FR (Our method)
&
0.015(0) & \underline{0.022(0)} &
\underline{0.95(0)} & \textbf{97.97(0)} & \textbf{33.49(0)} & \textbf{97.98(0)} & \underline{33.59(0)} \\
\midrule
\multicolumn{9}{l}{\textbf{AD trigger}} \\
AD & Constant & 0.142(0) & 0.195(0) & 0.00(0) & 62.548(0) & 14.399(0) & NA & NA \\
AD & PID \citep{emami2026selfdrivingtriggerlhcadaptive} & 0.035(0) & 0.080(0) & 0.405(0) & 75.053(0) & 39.191(0) & 76.210(0) & \underline{44.573(0)} \\
AD & DSPOT ~\citep{siffer2017anomaly} & 0.104(0) & 0.180(0) & 0.135(0) & 67.843(0) & 27.225(0) & \textbf{80.979(0)} & \textbf{56.686(0)} \\
AD & DQN \citep{mnih2015human} & \underline{0.027(0)} & \underline{0.060(0)} & 0.527(0.01) & 74.971(0.07) & 39.66(0.23) & 75.743(0.21) & 43.621(0.71) \\
AD & DQN-F & 0.142(0) & 0.195(0) & 0(0) & 62.548(0) & 14.399(0) & NA & NA \\
AD & PPO \citep{schulman2017proximal} & \underline{0.027(0)} & 0.061(0) & 0.54(0) & 75.06(0) & 39.70(0) & 75.600(0) & 43.150(0) \\
AD & ADT \citep{yang2024adt}     & \underline{0.027(0.001)} & 0.063(0.002) & 0.52(0.03) & 75.011(0.31) & 39.510(0.20) & 75.614(0.67) & 42.930(0.92) \\
AD & GRPO \citep{shao2024deepseekmath} & \underline{0.027(0.001)} & 0.061(0.004) & 0.541(0.01) & 75.059(0.11) & 39.700(0.06) & 75.601(0.10) & 43.153(0.41) \\
AD & L-GRPO & \underline{0.027(0.001)} & 0.061(0.002) & 0.541(0.01) & 75.059(0.06) & 39.68(0.04) & 75.597(0.09) & 43.520(0.46) \\
AD & CPO \citep{achiam2017constrained} & 0.027(0.000) & 0.061(0) & 0.541(0) & 75.059(0) & 39.703(0) & 75.60(0) & 43.15(0) \\
\rowcolor{highlight}
AD & GFPO-F (Our method) 
& \textbf{0.021(0.001)} & \textbf{0.053(0.003)} & \textbf{0.676(0.02)} & \textbf{75.640(0.19)} & \textbf{40.370(0.10)} 
& 76.183(0.02) & 42.910(0.68) \\
\rowcolor{highlightDark}
AD & GFPO-FR (Our method) 
& \underline{0.027(0)} & \underline{0.060(0)} & \underline{0.595(0)} & \underline{75.580(0)} & \underline{40.240(0)} & \underline{76.397(0)} & 44.163(0) \\
\bottomrule
\end{tabular}
\end{table*}

%% file: Appendix/Autoencoder_ablation_study.tex
\newpage
\section{Ablation study on noisy anomaly scores}
\label{appendix:ablation_anomaly_score}
In HEP anomaly detection, a long line of work has explicitly highlighted the sensitivity of reconstruction-based methods to hyperparameters and score design, while more recent studies show that the geometry and topology of the latent space can materially change anomaly-separation performance~\citep{ngairangbam2025enhancing, gandrakota2024real, zipper2024testing}. Since LHC anomaly triggers operate by thresholding such scores in a real-time setting, understanding how the latent dimension affects score quality is necessary before treating any threshold as reliable.

\paragraph{Latent dimension of the autoencoder training for AD trigger.}
We ablate the autoencoder bottleneck size to quantify how representation capacity and the resulting variability in reconstruction scores affect downstream threshold control. While the main paper uses $d=2$, we retrain the autoencoder on MC with larger latents $d\in\{4,6,8,10,12,14,16\}$. To reflect realistic deployment, we train on one MinBias sample and apply the model to a disjoint MinBias dataset, introducing the distribution shift expected between offline training and online operation.

Table~\ref{tab:single_trigger_summary_latent_dims} shows that enlarging $d$ changes the score distribution, yet the qualitative controller behavior is stable. The PID baseline is consistently less reliable due to \textit{fixed} and \textit{predefined} coefficients, with substantially lower InBand fractions (about $0.30$--$0.46$ with $d \in \{4, 6, 8\}$), suggesting sensitivity to score calibration under latent dimension change. These observations are expected, as the parameter for the AD trigger in \SEpaper\ is calibrated on a single MinBias sample with latent dimension $d = 2$ and lacks \emph{generalizability} to other latent dimensions. In contrast, RL-based controllers (DQN/GRPO/GFPO-F/GFPO-FR) achieve high feasibility across all $d$ (InBand $\approx 0.76$--$1.00$) with low MAE and P95$|e|$, suggesting that closed-loop adaptation is weakly dependent on precise anomaly-score calibration and is \emph{generalizable} across different autoencoder latent dimensions~\citep{gandrakota2024real}.

Among RL variants, the filtered objectives expose the expected feasibility--efficiency trade-off: feasibility-first filtering (GFPO-F) achieves near-perfect background constraint satisfaction, while relaxed filtering (GFPO-FR) shall improve conditional signal efficiency in some settings at the cost of more violations. Overall, these results resonate our main paper's findings with $d = 2$. That is, our RL controllers primarily leverage rate feedback over windows rather than absolute score calibration, and therefore remain robust to moderate changes in autoencoder capacity and score quality, consistent with prior dynamic-thresholding observations \citep{yang_adt_github}.

\begin{table*}[t]
  \centering
  \caption{Single-trigger control (MC) across latent dimensions $d\in\{4,6,8,10,12,14,16\}$ for AD trigger (MC). Rates are in
  percent units with target $r^*=0.25\%$ and tolerance $\pm 0.025\%$. Standard deviations are omitted because they are negligible (near zero) across all entries.
  }
  \label{tab:single_trigger_summary_latent_dims}
  \tiny
  \setlength{\tabcolsep}{3.5pt}
  \renewcommand{\arraystretch}{0.95}
  \resizebox{\linewidth}{!}{%
  \begin{tabular}{c l rrr rrrr}
  \hline
  Latent dim & Method
  & MAE$\downarrow$ & P95$|e|$$\downarrow$ & InBand$\uparrow$
  & $\epsilon_{\text{overall}}^{\ttbarraw}\uparrow$ & $\epsilon_{\text{overall}}^{\haaFourB}\uparrow$
  & $\epsilon_{\text{inband}}^{\ttbarraw}\uparrow$ & $\epsilon_{\text{inband}}^{\haaFourB}\uparrow$ \\
  \hline

  \multirow{8}{*}{$d=4$}
  & Constant  & 0.164 & 0.216 & 0.016 & 88.084 & 13.415 & 94.922 & 19.155 \\
  & PID       & 0.035 & 0.072 & 0.301 & 97.918 & 25.993 & 98.086 & 26.277 \\
  & DQN       & 0.011 & 0.024 & \underline{0.984} & 98.002 & 26.411 & 97.999 & 26.323 \\
  & ADT       & 0.021 & 0.054 & 0.554 & 98.040 & 26.475 & 98.010 & 26.164 \\
  & PPO       & 0.016 & 0.030 & 0.860 & \underline{98.105} & \underline{27.057} & 98.085 & \underline{26.837} \\
  & GRPO      & 0.011 & \underline{0.022} & \textbf{1.000} & 98.012 & 26.435 & 98.039 & 26.384 \\
  \rowcolor{highlight}
  & GFPO-F    & \textbf{0.003} & \textbf{0.008} & \textbf{1.000} & 98.095 & 26.718 & \underline{98.111} & 26.768 \\
  \rowcolor{highlightDark}
  & GFPO-FR   & 0.020 & 0.033 & 0.849 & \textbf{98.260} & \textbf{27.486} & \textbf{98.343} & \textbf{27.721} \\
  \hline

  \multirow{8}{*}{$d=6$}
  & Constant  & 0.130 & 0.202 & 0.032 & 80.586 & 14.815 & 86.915 & 16.584 \\
  & PID       & 0.024 & 0.057 & 0.456 & 91.858 & 22.049 & 91.623 & 21.602 \\
  & DQN       & \underline{0.015} & \underline{0.033} & \underline{0.876} & 92.202 & 22.509 & 92.139 & 22.266 \\
  & ADT       & 0.019 & 0.048 & 0.608 & 92.195 & 22.445 & 92.083 & 22.051 \\
  & PPO       & \underline{0.015} & 0.034 & 0.833 & 92.209 & 22.519 & 92.205 & 22.361 \\
  & GRPO      & \underline{0.015} & \textbf{0.032} & 0.871 & 92.231 & 22.526 & 92.114 & 22.194 \\
  \rowcolor{highlight}
  & GFPO-F    & \textbf{0.010} & \underline{0.033} & \textbf{0.882} & \underline{92.718} & \underline{23.100} & \underline{92.735}
   & \underline{23.082} \\
   \rowcolor{highlightDark}
  & GFPO-FR   & 0.022 & 0.044 & 0.737 & \textbf{93.237} & \textbf{23.709} & \textbf{93.407} & \textbf{23.992} \\
  \hline

  \multirow{8}{*}{$d=8$}
  & Constant  & 0.159 & 0.229 & 0.032 & 92.084 & 15.554 & 95.235 & 19.186 \\
  & PID       & 0.032 & 0.067 & 0.387 & 97.861 & 25.439 & 97.893 & 25.050 \\
  & DQN       & 0.018 & 0.040 & 0.763 & 98.097 & 26.472 & 98.086 & 26.210 \\
  & ADT       & \underline{0.018} & 0.042 & 0.634 & 98.101 & 26.348 & 98.039 & 25.846 \\
  & PPO       & 0.021 & 0.048 & 0.694 & 98.061 & 26.145 & 97.967 & 25.571 \\
  & GRPO      & 0.018 & \textbf{0.038} & \underline{0.780} & 98.108 & 26.509 & 98.078 & 26.191 \\
  \rowcolor{highlight}
  & GFPO-F    & \textbf{0.014} & \underline{0.039} & \textbf{0.839} & \underline{98.278} & \underline{27.287} & \underline{98.221}
   & \underline{26.908} \\
   \rowcolor{highlightDark}
  & GFPO-FR   & 0.022 & 0.044 & 0.667 & \textbf{98.380} & \textbf{27.825} & \textbf{98.307} & \textbf{27.329} \\
  \hline

  \multirow{8}{*}{$d=10$}
  & Constant  & 0.205 & 0.246 & 0.016 & 79.062 & 9.362 & 91.594 & 16.695 \\
  & PID       & 0.022 & 0.053 & 0.651 & 96.369 & 23.794 & \underline{96.900} & \underline{25.151} \\
  & DQN       & 0.015 & 0.030 & \underline{0.909} & 96.542 & 24.167 & 96.679 & 24.478 \\
  & ADT       & 0.016 & 0.042 & 0.677 & 96.566 & 24.301 & 96.649 & 24.453 \\
  & PPO       & 0.020 & 0.038 & 0.737 & 96.398 & 23.619 & 96.404 & 23.574 \\
  & GRPO      & \underline{0.014} & \underline{0.028} & 0.903 & 96.547 & 24.153 & 96.701 & 24.531 \\
  \rowcolor{highlight}
  & GFPO-F    & \textbf{0.007} & \textbf{0.022} & \textbf{0.952} & \underline{96.632} & \underline{24.394} & 96.705 & 24.568 \\
  \rowcolor{highlightDark}
  & GFPO-FR   & 0.019 & 0.030 & 0.887 & \textbf{96.854} & \textbf{25.036} & \textbf{97.031} & \textbf{25.534} \\
  \hline

  \multirow{8}{*}{$d=12$}
  & Constant  & 0.182 & 0.236 & 0.016 & 84.400 & 12.263 & 93.426 & 18.335 \\
  & PID       & 0.016 & 0.040 & 0.758 & 96.848 & 24.629 & 96.826 & 24.311 \\
  & DQN       & \underline{0.013} & \underline{0.024} & \underline{0.952} & 96.931 & 24.891 & 96.991 & 25.048 \\
  & ADT       & \underline{0.013} & 0.032 & 0.812 & 96.906 & 24.737 & 96.947 & 24.909 \\
  & PPO       & 0.020 & 0.031 & 0.715 & 96.757 & 24.287 & 96.671 & 23.829 \\
  & GRPO      & \underline{0.013} & 0.026 & 0.946 & 96.946 & 24.920 & 96.952 & 24.882 \\
  \rowcolor{highlight}
  & GFPO-F    & \textbf{0.005} & \textbf{0.018} & \textbf{0.984} & \underline{97.056} & \underline{25.264} & \underline{97.069} &
  \underline{25.313} \\
  \rowcolor{highlightDark}
  & GFPO-FR   & 0.020 & 0.030 & 0.876 & \textbf{97.299} & \textbf{26.055} & \textbf{97.321} & \textbf{26.157} \\
  \hline

  \multirow{8}{*}{$d=14$}
  & Constant  & 0.159 & 0.230 & 0.043 & 90.830 & 15.606 & 94.393 & 19.180 \\
  & PID       & 0.014 & 0.036 & 0.839 & 97.310 & 26.116 & 97.299 & 25.943 \\
  & DQN       & \underline{0.012} & \underline{0.024} & \underline{0.973} & 97.357 & 26.374 & 97.378 & 26.361 \\
  & ADT       & 0.014 & 0.034 & 0.774 & 97.365 & 26.473 & 97.299 & 26.237 \\
  & PPO       & 0.012 & 0.024 & 0.952 & \underline{97.404} & 26.479 & \underline{97.415} & 26.464 \\
  & GRPO      & \underline{0.012} & \underline{0.024} & \underline{0.973} & 97.365 & 26.384 & 97.350 & 26.282 \\
  \rowcolor{highlight}
  & GFPO-F    & \textbf{0.004} & \textbf{0.015} & \textbf{0.989} & 97.401 & \underline{26.487} & 97.407 & \underline{26.510} \\
  \rowcolor{highlightDark}
  & GFPO-FR   & 0.019 & 0.026 & 0.919 & \textbf{97.596} & \textbf{27.174} & \textbf{97.646} & \textbf{27.373} \\
  \hline

  \multirow{8}{*}{$d=16$}
  & Constant  & 0.165 & 0.233 & 0.032 & 84.263 & 14.358 & 88.231 & 17.853 \\
  & PID       & \underline{0.012} & 0.030 & 0.898 & 94.297 & 25.031 & 94.271 & 24.916 \\
  & DQN       & \underline{0.012} & \underline{0.024} & \underline{0.989} & 94.331 & 25.132 & 94.259 & 24.946 \\
  & ADT       & 0.013 & 0.032 & 0.753 & 94.391 & 25.318 & 94.410 & 25.273 \\
  & PPO       & 0.012 & 0.026 & 0.946 & \underline{94.420} & 25.271 & \underline{94.467} & 25.271 \\
  & GRPO      & \underline{0.012} & \underline{0.024} & 0.978 & 94.341 & 25.152 & 94.289 & 24.984 \\
  \rowcolor{highlight}
  & GFPO-F    & \textbf{0.003} & \textbf{0.010} & \textbf{1.000} & 94.409 & \underline{25.319} & 94.409 & \underline{25.319} \\
  \rowcolor{highlightDark}
  & GFPO-FR   & 0.018 & \underline{0.024} & 0.984 & \textbf{94.697} & \textbf{25.875} & \textbf{94.724} & \textbf{25.928} \\
  \hline
  \end{tabular}
  }
  \end{table*}

%% file: Appendix/anomaly_benchmarks.tex
\section{Anomaly Benchmarks}
\label{appendix:anomaly_benchmarks}
\subsection{UNSW-NB15}
\label{appendix:unsw_nb15}

\begin{figure}[h]
    \centering
\includegraphics[width=0.9\linewidth]{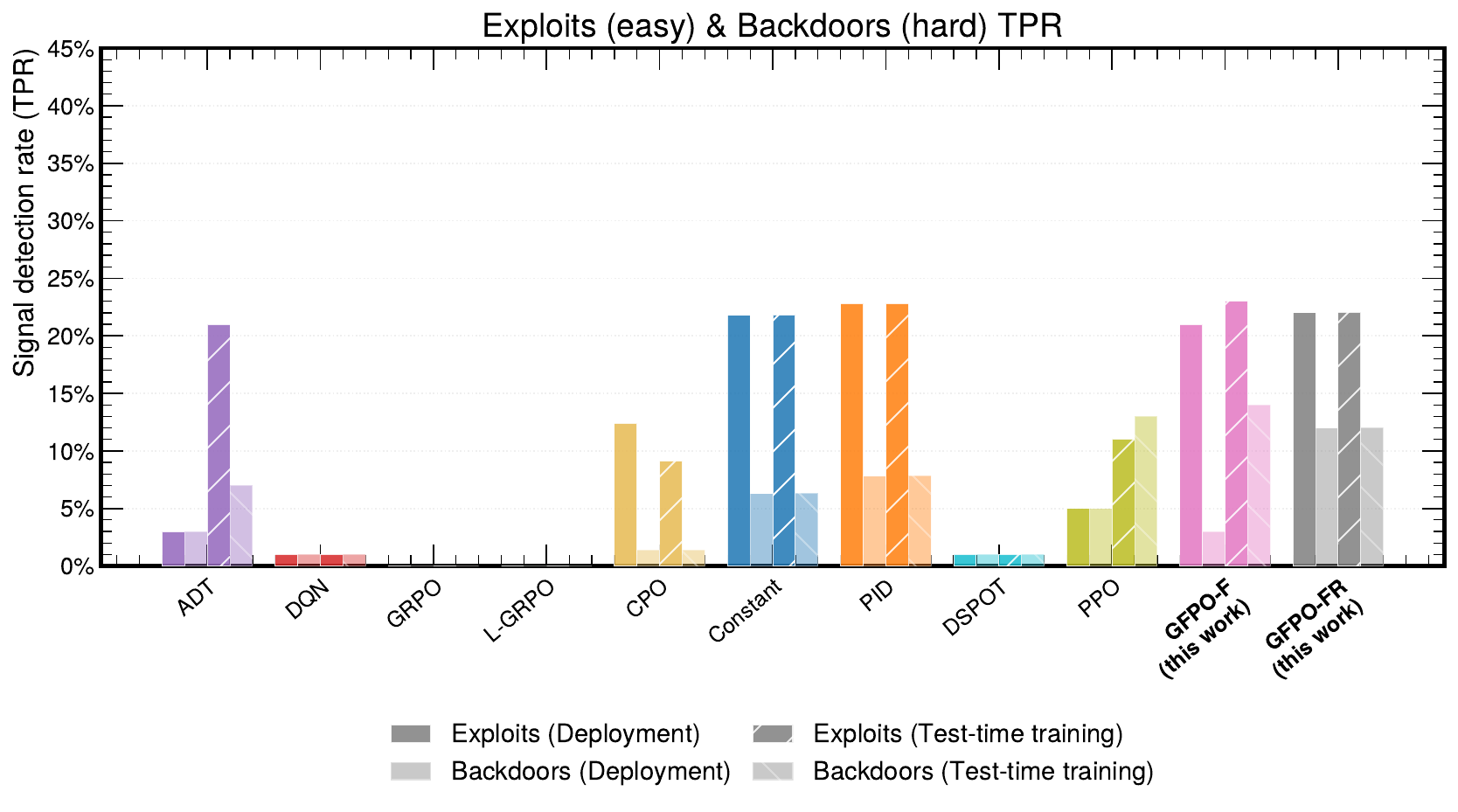}
    \caption{\textbf{UNSW-NB15 TPR by attack difficulty and adaptation regime.} For each controller we report TPR on two attack classes: \textit{Exploits} (frequent, high-signal; easy) and \textit{Backdoors} (rare, low-signal; hard). Bars span $\{\text{Deployment}, \text{Test-time training}\} \times \{\text{Exploits}, \text{Backdoors}\}$: solid = frozen policy, hatched = online adaptation; high opacity = Exploits, low opacity = Backdoors. GFPO-F and GFPO-FR (\textit{ours}) maintain the highest Backdoors TPR while satisfying the FAR budget (Fig.~\ref{fig:anomaly_usnw_nb15}); GRPO and L-GRPO collapse to near-zero TPR via the zero-feasibility failure mode. Test-time or online training widens the GFPO advantage on Backdoors without trading off Exploits.}
\label{fig:anomaly_appendix_tpr}
\end{figure}

\paragraph{Setup.}
UNSW-NB15~\citep{moustafa2015unsw} comprises network traffic from 
two collection campaigns with documented distributional shift, 
giving a natural sim-to-real test of the MC$\to$CMS protocol: we 
train the autoencoder and RL policy on Period~1 and deploy the 
frozen policy on Period~2. The background class is normal traffic. 
We designate \emph{Exploits} as the easy signal (anomalous payloads 
produce reconstruction losses well-separated from normal, analogous 
to $t\bar{t}$~\citep{CMS:2017trigger}) and \emph{Backdoors} as the 
hard signal (low-footprint behaviour overlaps substantially with 
normal scores, analogous to $h \to 4b$~\citep{aaboud2018search}). 
The per-window connection rate replaces $N_{\mathrm{PV}}$ as the 
context variable.
 
\paragraph{Choice of FAR target and tolerance on UNSW-NB15.}
Recall FAR denotes False Alert Rate, the per-chunk fraction of benign records flagged as anomalies.  
We set $r^{*} = 0.5\%$ with tolerance $\tau = 0.05\%$, mirroring the 
$\pm 10\%$ relative tolerance of the LHC trigger setting 
($r_{B}^{*} = 0.25\%$ for LHC, $\tau = 0.025\%$). Holding 
relative stringency fixed lets the in-band quadratic reward and the 
weight $\lambda_1$ transfer across domains without retuning. The absolute level is doubled ($0.5\%$ vs.\ $0.25\%$) for a 
quantization reason. UNSW-NB15 streams in chunks of $C = 1000$ 
records; Period~2 has $\sim 55\%$ attack 
prevalence~\citep{moustafa2015unsw}, leaving $N^- \approx 447$ 
negatives per chunk and quantizing per-chunk FAR on a grid of step 
$1/N^- \approx 0.224\%$, where $\mathrm{FP}$ is the integer count of benign 
records flagged in a chunk.  The LHC band $[0.225\%, 0.275\%]$ contains 
no realizable FP count: $\mathrm{FP}=1$ undershoots 
($\mathrm{FAR} = 0.224\%$) and $\mathrm{FP}=2$ overshoots 
($\mathrm{FAR} = 0.447\%$), so ``in-band'' is unreachable. The 
chosen band $[0.45\%, 0.55\%]$ contains exactly one realizable 
count, $\mathrm{FP}=2$ at $\mathrm{FAR} \approx 0.45\%$ (within 
chunk-level variation in $N^-$), making it the tightest tolerance 
compatible with the integer FP grid. The LHC analogue avoids this 
floor: its per-chunk negative count is several orders of magnitude 
larger, so $0.25\% \pm 0.025\%$ sits well above quantization there.

\paragraph{RL setting.} The agent observes a length-$8$ history of
  $(p_{10}, p_{25}, p_{50}, p_{75}, p_{90}, p_{95}, p_{99},\allowbreak                              
  \text{threshold},\allowbreak \mathrm{FAR},\allowbreak \mathrm{TPR},                               
  \allowbreak |\text{FAR} - r^{*}|,\allowbreak \text{attack prevalence})$                                   
  features ($12$-dim per chunk). Actions are discrete threshold deltas                              
  on a $21$-bin grid covering $\pm 0.5$ in normalised score space. We reuse the LHC rate-control reward of Section~\ref{sec:single_trigger} (Eq.~\ref{equation:reward_design}): an in-band term on $\mathrm{FAR}$ (target $r^{*}=0.5\%$, tolerance $\tau=0.05\%$), a TPR bonus gated on feasibility, and a movement penalty, with $\lambda_1 = 0.25$, $\alpha=0$, $\beta=0.005$. 

  \paragraph{Training setup for deployment.}  
  We train each method for $E = 50$ passes over Period~1 (the $175$-chunk UNSW-NB15 training-set stream, $68\%$ attack prevalence) with Adam (learning rate $3 \times 10^{-4}$, mini-batch size $256$, $2$ inner epochs per update). GRPO and L-GRPO use group size $G = 16$; GFPO-F and GFPO-FR use $G = 64$ with $K = 16$, matching the LHC setups. We then freeze the policy and deploy on Period~2 (the $82$-chunk test-set stream, $55.3\%$ attack prevalence); no gradient updates occur during deployment. We report means $\pm$ standard deviations over $5$ random seeds per method (seeds $0$--$4$). All UNSW-NB15 experiments run on a single Apple~M4 CPU ($10$~cores); the full $5$-seed train$+$frozen-rollout pipeline completes in        approximately $241$~seconds (\textasciitilde$12$~s per method per seed).  

  \paragraph{Training setup for Test-time training.} Here, we report the setting for test-time training results in Figure~\ref{fig:anomaly_appendix_tpr}.
We additionally evaluate a TTT variant in which each agent 
continues updating its policy during Period~2 rollout. At every 
chunk the agent (i) selects a threshold delta greedily, 
(ii) observes $(\mathrm{FAR}, \mathrm{TPR})$ and computes the 
in-band reward, and (iii) performs a single gradient step: DQN 
draws one minibatch from a replay buffer; PPO performs one PPO 
update; GRPO and L-GRPO take one group update with $G = 16$; 
GFPO-F and GFPO-FR take one feasibility-filtered group update 
with $G = 64$, $K = 16$. L-GRPO and GFPO-FR additionally update 
their dual variable $\lambda$ once per chunk on the observed FAR. 
This yields a single in-order pass of $82$ updates per agent, 
versus the $175 \times 50 = 8{,}750$ updates of the Period~1 
initialization; all other hyperparameters (reward, target/tolerance, 
group sizes, learning rate) are unchanged.

  \paragraph{Structural ceiling on in-band rate.}    Of Period~2's $82$ chunks, $43$ contain only attack traffic and have $\mathrm{FAR} \equiv 0$ regardless of threshold (the denominator $\mathrm{FP}+\mathrm{TN}$ is zero, thus reported as $0$); they can never lie within $[r^* - \tau, r^* + \tau] = [0.45\%, 0.55\%]$. The remaining $39$ benign-containing chunks ($35$ all-benign, $4$ mixed) set a hard upper bound of $39/82 \approx 47.6\%$ on the achievable in-band rate. Reported in-band fractions should be read against this ceiling: in Figure~\ref{fig:anomaly_usnw_nb15}, GFPO-FR's $6.1\%$ ($5/82$ chunks) corresponds to $\approx 12.8\%$ of the $47.6\%$ structural maximum.

\paragraph{D-SPOT failure.}
D-SPOT achieves $0\%$ in-band rate despite being designed for 
streaming anomaly detection (Figure~\ref{fig:anomaly_usnw_nb15}). 
The cause is class imbalance: with $\approx 55\%$ attack prevalence, 
the sliding window used to fit the Generalized Pareto distribution 
is dominated by attack-score values, pushing the estimated tail 
threshold into attack-score territory. The threshold is then too 
low to constrain FAR near $r^*$, and the method flags effectively 
all traffic. This is a known failure mode of EVT-based detectors 
on imbalanced streams~\citep{siffer2017anomaly}. GFPO sidesteps 
this failure by closing the loop on FAR rather than estimating 
the score tail: the policy adjusts the threshold via per-chunk 
FAR feedback, feasibility filtering retains only training samples 
already in-band, and all-attack chunks are excluded from 
rate-control updates. The class imbalance that corrupts DSPOT's 
GPD fit never enters GFPO's training signal, which is why GFPO-F 
and GFPO-FR maintain the highest Backdoors TPR 
(Figure~\ref{fig:anomaly_usnw_nb15}) while satisfying the FAR 
budget, exactly the regime where DSPOT collapses.

\paragraph{Analysis on Per-class TPR.}
Figure~\ref{fig:anomaly_appendix_tpr} reports per-method TPR by 
attack class (\emph{Exploits}, easy; \emph{Backdoors}, hard) and 
adaptation regime (frozen Deployment, Test-time training), 
underlying the in-band/MAE summary in 
Section~\ref{sec:anomaly_detection_experiment_results}. Constant 
and PID hit non-trivial Exploits TPR ($\sim$22\%) because their 
fixed thresholds happen to land near the easy-signal mode, but 
they collapse on Backdoors ($6$--$8\%$) and stay outside the FAR 
band entirely. DQN, DSPOT, and ADT (under deployment) reach 
near-zero TPR on both classes, the over-conservative collapse 
diagnosed above. GRPO and L-GRPO collapse to near-zero TPR via 
the zero-feasibility failure mode of Section~\ref{sec:gfpo}: when 
no sampled action satisfies the FAR constraint, the group-relative 
advantage degenerates and the policy converges to a single 
over-conservative threshold. PPO recovers some TPR but only by 
abandoning the FAR budget, with mean absolute FAR error two 
orders of magnitude worse than ours.

GFPO-F and GFPO-FR (\textit{ours}) match the best Exploits TPR 
($\approx 22\%$) and dominate on Backdoors ($\approx 12\%$), 
roughly $2\times$ the next-best learned baseline, while satisfying 
the FAR budget. Test-time training widens the Backdoors gap 
without trading off Exploits, indicating that per-chunk FAR 
feedback provides a usable adaptation signal in exactly the 
rare-class regime where static thresholds fail.

\newpage
\subsection{NAB}
\label{app:nab}

\begin{figure}
    \centering
\includegraphics[width=.9\linewidth]{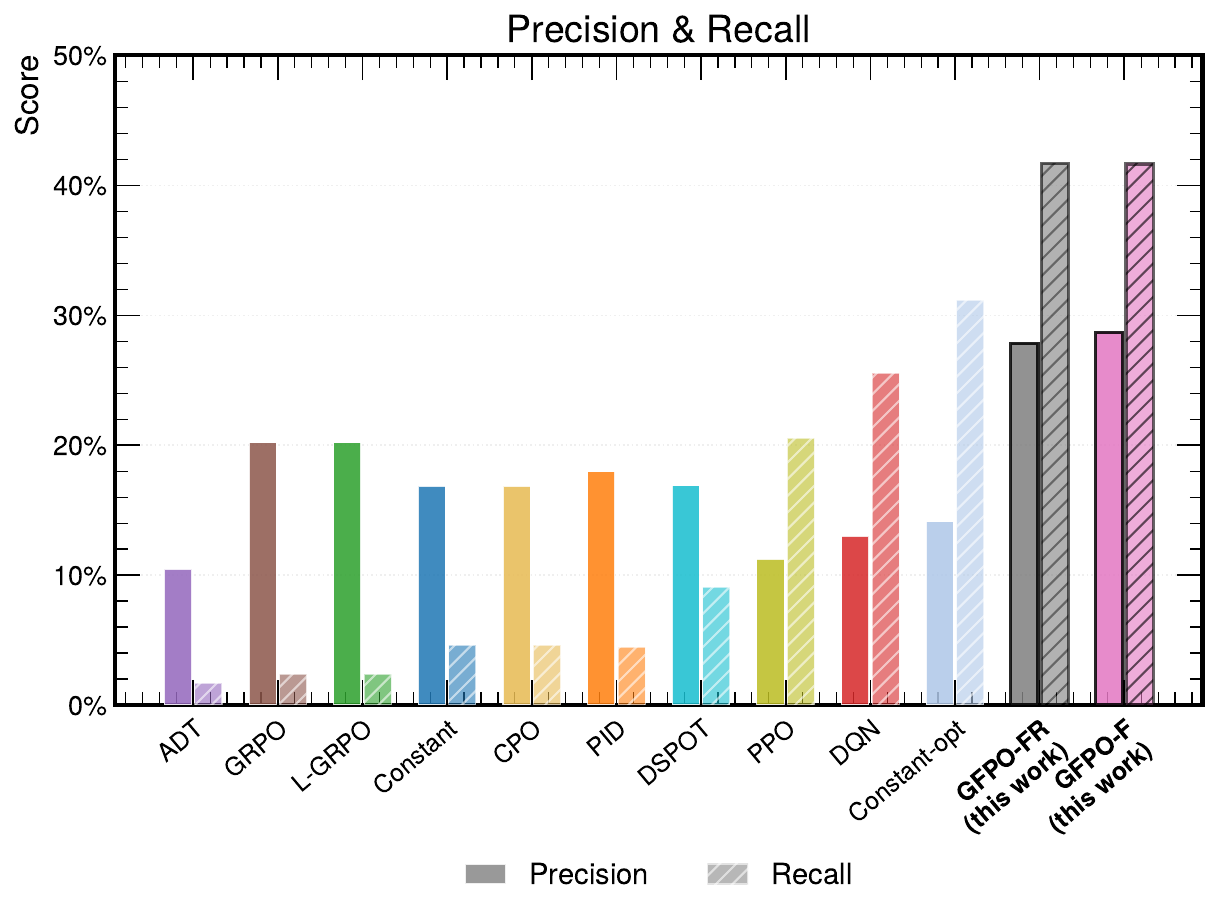}
    \caption{\textbf{NAB precision and recall per method (mean over 24 test files).} Solid bars: precision; hatched bars: recall. Methods sorted left-to-right by ascending F1. All controllers operate on the same per-chunk anomaly scores under the same FAR budget. GFPO-F and GFPO-FR (\textit{ours}) dominate the Pareto frontier: precision $\approx 28\%$ and recall $\approx 42\%$. Classical controllers (Constant, PID, DSPOT) 
reach $11$--$20\%$ precision but recall below $10\%$, too conservative 
to fire on more than a small fraction of true anomalies. PPO, DQN, 
and the oracle Constant-opt push recall above $20\%$ at a precision 
cost. GRPO and L-GRPO collapse to the high-precision/near-zero-recall 
corner ($\sim\!20\%$ / $2.4\%$): under the sparse anomaly reward, the 
group-relative advantage carries little signal between high-threshold 
modes, and the policy converges to a near-silent threshold.}
\label{fig:anomaly_appendix_pr_nab}
\end{figure}

We evaluate on the Numenta Anomaly Benchmark (NAB)~\citep{lavin2015evaluating}, restricted
to its two real-world streaming subsets, \textit{realKnownCause} and
\textit{realAWSCloudwatch}. Together these comprise $24$ univariate time series
of server, machine, and cloud-monitoring telemetry with hand-labelled anomaly
windows, totalling $135{,}900$ timesteps after we discard series shorter than
$200$ steps (insufficient for a meaningful train/test split). The goal of this agent is to optimize a pure detection-quality reward 
$r_t = \mathrm{TPR} - \alpha \cdot \mathrm{FPR} 
- \beta\,|\threshold_t|/\Delta_{\max}$. 

\paragraph{Preprocessing.}
Each series is processed chronologically. Raw values are converted 
to anomaly scores via a sliding-window robust z-score, 
$s_t = |x_t - \tilde{\mu}_{W_t}| / 
(\widetilde{\mathrm{MAD}}_{W_t} + \epsilon)$, with a causal window 
of 100 timesteps, then normalized to $[0,1]$ via empirical CDF rank 
on the training portion. We split each series 70/30 chronologically 
into training and test segments, then partition into non-overlapping 
chunks of 100 timesteps, yielding 954 training and 405 test chunks. 
MAD denotes the median absolute deviation, a robust scale estimator $\mathrm{MAD}(W) = \mathrm{median}_x |x - \mathrm{median}(W)|$; we substitute it for the standard deviation in the z-score so that anomalies in the recent window do not inflate the normality envelope used to score them.  

\paragraph{RL setting.} The agent observes a length-8 history of $(p_{25}, p_{50}, p_{75},
p_{90}, p_{95}, p_{99},\allowbreak \text{threshold},\allowbreak
\text{flagging rate},\allowbreak \text{TPR}, \text{FPR})$ features. Actions are discrete threshold deltas on a 21-bin grid covering $\pm 0.3$ in score space. Unlike the LHC setting, NAB has no rate constraint: the reward is $\mathrm{TPR} - \alpha \mathrm{FPR} - \beta |\Delta|/\Delta_{\max}$, which is a pure detection quality with a movement penalty.

\paragraph{Training setup}
We train each method for $E = 50$ passes over the training partition 
with Adam (learning rate $3 \times 10^{-4}$, batch size $64$). GRPO 
and L-GRPO use group size $G = 16$; GFPO-F and GFPO-FR use $G = 64$ 
with $K = 16$, matching the LHC setup in Figure~\ref{fig:gfpo_intermediate_ht}. We report means over $5$ random seeds 
per method. All NAB experiments run on a single Apple M4 CPU 
(10 cores), with average wall-clock time of approximately $10$ 
minutes per method per series.

\paragraph{Environment.}
The RL environment we designed uses reward 
$r_t = \mathrm{TPR} - \alpha \cdot \mathrm{FPR} 
- \beta\,|\delta_t|/\Delta$, where $\alpha = 0.10$ penalizes false 
positives, $\beta = 0.005$ discourages unnecessary threshold 
movement, and $\Delta = 0.3$ is the action range. The action space 
consists of 21 discrete threshold deltas uniformly spaced in 
$[-0.3, +0.3]$. The 10-dimensional state encodes score-distribution 
percentiles ($p_{25}, p_{50}, p_{75}, p_{90}, p_{95}, p_{99}$), 
the current threshold, the instantaneous flagging rate, and 
exponentially averaged recent TPR and FPR over the last 8 chunks. 
All agents are initialized with the threshold at the 97th 
training-score percentile ($\approx$3\% flagging rate).

\begin{table}[t]                 
  \centering 
  \caption{NAB benchmark results (405 test chunks, 24 real-world series). Constant-opt is an oracle static threshold.}                
  \label{tab:nab}    
  \begin{tabular}{lccc}                           
  \toprule                                         
  Method & Precision $\uparrow$ & Recall $\uparrow$ & F1 $\uparrow$ \\                       
  \midrule                                          
  Constant          & 0.169 & 0.046 & 0.058 \\
  Constant-opt      & 0.141 & \textbf{0.312} & 0.184 \\                                    
  PID               & 0.180 & 0.045 & 0.059 \\                                     
  D-SPOT            & 0.169 & 0.090 & 0.082 \\                    DQN & 0.205 & 0.024 & 0.038 \\                
    ADT   
    & 0.103 & 0.014 & 0.024 \\
  GRPO          & 0.198 & 0.031 & 0.042 \\           
    L-GRPO        & 0.198 & 0.031 & 0.042 \\
  CPO               & 0.169 & 0.046 & 0.058 \\
  PPO &  0.145 & 0.117 & 0.097 \\
  \rowcolor{highlight}
  
  GFPO-F &
  \textbf{0.320} & \underline{0.198} & \underline{0.215} \\
\rowcolor{highlightDark}
  GFPO-FR    & 
  \underline{0.316} & \underline{0.198} & \textbf{0.216} \\
  \bottomrule                 
  \end{tabular}  
  \end{table}

\paragraph{Precision \& Recall}
Figure~\ref{fig:anomaly_appendix_pr_nab} reports the per-method 
precision and recall/TPR underlying the F1 summary in 
Section~\ref{sec:anomaly_detection_experiment_results} 
(Figure~\ref{fig:anomaly_nab_f1}). Classical controllers (Constant, 
PID, DSPOT) reach 11--20\% precision but TPR/recall below 10\%, defaulting 
to conservative thresholds that miss most true anomalies. PPO and the oracle-tuned Constant-opt push TPR/recall 
above 10\% but drop to 14\% precision, lowering thresholds to 
fire on more events at the cost of false positives. GRPO and L-GRPO 
collapse to the high-precision/near-zero-TPR corner 
($\sim$20\%/3.1\%): under the sparse anomaly reward, the 
group-relative advantage carries little signal between high-threshold 
modes and the policy settles on a near-silent one.

Table~\ref{tab:nab} shows GFPO-F and GFPO-FR (\textit{ours}) reach precision $\approx 32\%$ 
and TPR $\approx 19\%$, surpassing every baseline on both axes and 
yielding the highest F1 in Figure~\ref{fig:anomaly_nab_f1}. As 
NAB drops the rate constraint, this gain isolates our sequence-based 
state and adaptive thresholding from feasibility-filtered policy 
optimization, confirming that GFPO improves detection quality even 
when the FAR-handling mechanism is dropped.